\documentclass
{article}
\usepackage[utf8]{inputenc}

\usepackage{etoolbox}

\newcommand{\arxiv}[1]{\iftoggle{iclr}{}{#1}}
\newcommand{\iclr}[1]{\iftoggle{iclr}{#1}{}}
\newcommand{\arxivstyle}[1]{\iftoggle{iclr}{}{#1}}
\newcommand{\iclrstyle}[1]{\iftoggle{iclr}{#1}{}}

\newtoggle{iclr}
\global\toggletrue{iclr}
\global\togglefalse{iclr}

\arxivstyle{\usepackage{txie}}
\iclrstyle{\usepackage{txie}}
\usepackage{inconsolata}
\usepackage{xspace}
\usepackage{breakcites}
\usepackage{mathrsfs}
\usepackage{transparent}
\DeclarePairedDelimiter{\abs}{\lvert}{\rvert} %
\DeclarePairedDelimiter{\brk}{[}{]}
\DeclarePairedDelimiter{\crl}{\{}{\}}
\DeclarePairedDelimiter{\prn}{(}{)}
\DeclarePairedDelimiter{\nrm}{\|}{\|}
\DeclarePairedDelimiter{\tri}{\langle}{\rangle}

\DeclarePairedDelimiter{\floor}{\lfloor}{\rfloor}

\let\Pr\undefined

\DeclareMathOperator{\En}{\mathbb{E}}

\DeclareMathOperator{\Pr}{Pr}

\newcommand{\wt}[1]{\widetilde{#1}}

\def\ddefloop#1{\ifx\ddefloop#1\else\ddef{#1}\expandafter\ddefloop\fi}
\def\ddef#1{\expandafter\def\csname bb#1\endcsname{\ensuremath{\mathbb{#1}}}}
\ddefloop ABCDEFGHIJKLMNOPQRSTUVWXYZ\ddefloop
\def\ddefloop#1{\ifx\ddefloop#1\else\ddef{#1}\expandafter\ddefloop\fi}
\def\ddef#1{\expandafter\def\csname b#1\endcsname{\ensuremath{\mathbf{#1}}}}
\ddefloop ABCDEFGHIJKLMNOPQRSTUVWXYZ\ddefloop
\def\ddef#1{\expandafter\def\csname sf#1\endcsname{\ensuremath{\mathsf{#1}}}}
\ddefloop ABCDEFGHIJKLMNOPQRSTUVWXYZ\ddefloop
\def\ddef#1{\expandafter\def\csname c#1\endcsname{\ensuremath{\mathcal{#1}}}}
\ddefloop ABCDEFGHIJKLMNOPQRSTUVWXYZ\ddefloop
\def\ddef#1{\expandafter\def\csname h#1\endcsname{\ensuremath{\widehat{#1}}}}
\ddefloop ABCDEFGHIJKLMNOPQRSTUVWXYZ\ddefloop
\def\ddef#1{\expandafter\def\csname hc#1\endcsname{\ensuremath{\widehat{\mathcal{#1}}}}}
\ddefloop ABCDEFGHIJKLMNOPQRSTUVWXYZ\ddefloop
\def\ddef#1{\expandafter\def\csname t#1\endcsname{\ensuremath{\widetilde{#1}}}}
\ddefloop ABCDEFGHIJKLMNOPQRSTUVWXYZ\ddefloop
\def\ddef#1{\expandafter\def\csname tc#1\endcsname{\ensuremath{\widetilde{\mathcal{#1}}}}}
\ddefloop ABCDEFGHIJKLMNOPQRSTUVWXYZ\ddefloop
\def\ddefloop#1{\ifx\ddefloop#1\else\ddef{#1}\expandafter\ddefloop\fi}
\def\ddef#1{\expandafter\def\csname scr#1\endcsname{\ensuremath{\mathscr{#1}}}}
\ddefloop ABCDEFGHIJKLMNOPQRSTUVWXYZ\ddefloop

\newcommand{\ind}{\mathbbm{1}}    %

\newcommand{\veps}{\varepsilon}

\newcommand{\ldef}{\vcentcolon=}
\newcommand{\rdef}{=\vcentcolon}

\arxivstyle{\usepackage[letterpaper, left=1in, right=1in, top=1in, bottom=1in]{geometry}}

\PassOptionsToPackage{hypertexnames=false}{hyperref}  %
\arxivstyle{\usepackage{parskip}}

\usepackage{microtype}
\usepackage{hhline}

\makeatletter
\newcommand{\neutralize}[1]{\expandafter\let\csname c@#1\endcsname\count@}
\makeatother

\usepackage{algorithm}

\usepackage{natbib}
\arxivstyle{\bibliographystyle{plainnat}}
\arxivstyle{\bibpunct{(}{)}{;}{a}{,}{,}}

\usepackage{amsthm}
\usepackage{mathtools}
\usepackage{amsmath}
\usepackage{bbm}
\usepackage{amsfonts}
\usepackage{amssymb}

\usepackage{xpatch}

\makeatletter
  \renewenvironment{proof}[1][Proof]%
  {%
   \par\noindent{\bfseries\upshape {#1.}\ }%
  }%
  {\qed\newline}
  \makeatother

\xpatchcmd{\proof}{\itshape}{\normalfont\proofnameformat}{}{}
\newcommand{\proofnameformat}{\bfseries}

\usepackage[nameinlink,capitalize]{cleveref}

\newcommand{\pref}[1]{\cref{#1}}
\newcommand{\pfref}[1]{Proof of \pref{#1}}

\Crefformat{equation}{#2Eq.\;(#1)#3}

\Crefformat{figure}{#2Figure #1#3}
\Crefformat{assumption}{#2Assumption #1#3}
\Crefname{assumption}{Assumption}{Assumptions}

\usepackage{crossreftools}
\newcommand{\creftitle}[1]{\crtcref{#1}}

\usepackage{xparse}

\ExplSyntaxOn
\DeclareDocumentCommand{\XDeclarePairedDelimiter}{mm}
 {
  \__egreg_delimiter_clear_keys: %
  \keys_set:nn { egreg/delimiters } { #2 }
  \use:x %
   {
    \exp_not:n {\NewDocumentCommand{#1}{sO{}m} }
     {
      \exp_not:n { \IfBooleanTF{##1} }
       {
        \exp_not:N \egreg_paired_delimiter_expand:nnnn
         { \exp_not:V \l_egreg_delimiter_left_tl }
         { \exp_not:V \l_egreg_delimiter_right_tl }
         { \exp_not:n { ##3 } }
         { \exp_not:V \l_egreg_delimiter_subscript_tl }
       }
       {
        \exp_not:N \egreg_paired_delimiter_fixed:nnnnn 
         { \exp_not:n { ##2 } }
         { \exp_not:V \l_egreg_delimiter_left_tl }
         { \exp_not:V \l_egreg_delimiter_right_tl }
         { \exp_not:n { ##3 } }
         { \exp_not:V \l_egreg_delimiter_subscript_tl }
       }
     }
   }
 }

\keys_define:nn { egreg/delimiters }
 {
  left      .tl_set:N = \l_egreg_delimiter_left_tl,
  right     .tl_set:N = \l_egreg_delimiter_right_tl,
  subscript .tl_set:N = \l_egreg_delimiter_subscript_tl,
 }

\cs_new_protected:Npn \__egreg_delimiter_clear_keys:
 {
  \keys_set:nn { egreg/delimiters } { left=.,right=.,subscript={} }
 }

\cs_new_protected:Npn \egreg_paired_delimiter_expand:nnnn #1 #2 #3 #4
 {%
  \mathopen{}
  \mathclose\c_group_begin_token
   \left#1
   #3
   \group_insert_after:N \c_group_end_token
   \right#2
   \tl_if_empty:nF {#4} { \c_math_subscript_token {#4} }
 }
\cs_new_protected:Npn \egreg_paired_delimiter_fixed:nnnnn #1 #2 #3 #4 #5
 {
  \mathopen{#1#2}#4\mathclose{#1#3}
  \tl_if_empty:nF {#5} { \c_math_subscript_token {#5} }
 }
\ExplSyntaxOff

\XDeclarePairedDelimiter{\supnorm}{
  left=\lVert,
  right=\rVert,
  subscript=\infty
  }

\newcommand{\dtil}{\wt{d}}%
\newcommand{\mustar}{\mu^{\star}}

\newcommand{\betext}{Bellman-Eluder dimension\xspace}
\newcommand{\betextsq}{Squared Bellman-Eluder dimension\xspace}
\newcommand{\betextshort}{Bellman-Eluder dim.\xspace}

\newcommand{\bitext}{Bilinear rank\xspace}

\newcommand{\exbmdp}{Ex-BMDP\xspace}

\newcommand{\Conc}{C_{\mathrm{conc}}}

\newcommand{\Enpi}[1][\pi]{\En^{#1}}

\newcommand{\Prpi}[1][\pi]{\bbP^{#1}}

\renewcommand{\emptyset}{\varnothing}

\newcommand{\CompText}{Decision-Estimation Coefficient\xspace}

\newcommand{\M}[1]{^{{\scriptscriptstyle M}}}  %

\newcommand{\sups}[1]{^{{\scriptscriptstyle#1}}}

\newcommand{\fstar}{f^{\star}}
\newcommand{\pistar}{\pi^{\star}}

\newcommand{\Mstar}{M^{\star}}

\newcommand{\algcommentlight}[1]{\textcolor{blue!70!black}{\transparent{0.5}\small{\texttt{\textbf{//\hspace{2pt}#1}}}}}

\newcommand{\trn}{\top}

\newcommand{\psdleq}{\preceq}

\newcommand{\approxleq}{\lesssim}

\renewcommand{\ind}[1]{^{{\scriptscriptstyle(#1)}}}

\newcommand{\bigoh}{O}
\newcommand{\bigoht}{\wt{O}}
\newcommand{\bigom}{\Omega}
\newcommand{\bigomt}{\wt{\Omega}}
\newcommand{\bigthetat}{\wt{\Theta}}

\newcommand{\indic}{\mathbb{I}}

\renewcommand{\Pr}{\bbP}

\newcommand{\poly}{\mathrm{poly}}
\newcommand{\polylog}{\mathrm{polylog}}

\newcommand{\Ber}{\mathrm{Ber}}

\newcommand{\Qpi}{Q^{\pi}}
\newcommand{\Vstar}{V^{\star}}
\newcommand{\Qstar}{Q^{\star}}

\newcommand{\supp}{\mathrm{supp}}

  \newcommand{\mathand}{\quad\text{and}\quad}

\def\multiset#1#2{\ensuremath{\left(\kern-.3em\left(\genfrac{}{}{0pt}{}{#1}{#2}\right)\kern-.3em\right)}}

\newcommand{\iid}{i.i.d.\xspace}

\renewcommand{\emptyset}{\varnothing}

\newcommand{\loose}{\iclrstyle{\looseness=-1}}

\newcommand{\Cconc}{C_{\mathsf{conc}}}
\newcommand{\Ccov}{C_{\mathsf{cov}}}

\newcommand{\Cconcgen}{\Cgen_{\mathsf{conc}}}
\newcommand{\Ccovgen}{\Cgen_{\mathsf{cov}}}

\newcommand{\coverability}{coverability\xspace}
\newcommand{\Coverability}{Coverability\xspace}
\newcommand{\compmeasure}{sequential extrapolation coefficient\xspace}
\newcommand{\CompMeasure}{Sequential Extrapolation Coefficient\xspace}
\newcommand{\SNC}{\mathsf{SEC}\xspace}
\newcommand{\bicompmeasure}{\textsf{gen}-$\mathsf{SEC}$\xspace}
\newcommand{\BiCompMeasure}{\textsf{Gen}-$\mathsf{SEC}$\xspace}

\newcommand{\sinit}{x_1}

\renewcommand{\indic}{\1}

\usepackage{pifont}

\iclrstyle{
\usepackage{iclr2023_conference}
\usepackage{times}
\setcitestyle{authoryear,round,citesep={;},aysep={,},yysep={;}}\bibliographystyle{plainnat}
\setlist{leftmargin=*,topsep=0pt}
\setlength{\parskip}{.5pc minus 3pt}
\setlength{\topsep}{3pt plus 1pt minus 1pt}
\setlength{\partopsep}{1pt plus 0.5pt minus 0.5pt}
\setlength{\itemsep}{2pt plus 1pt minus 1pt}
\setlength{\parsep}{2pt plus 1pt minus 1pt}
\setlength{\parindent}{0pt}
}

\let\oldparagraph\paragraph
\renewcommand{\paragraph}[1]{\oldparagraph{#1.}}

\usepackage[T1]{fontenc}
\usepackage[scaled=.91]{helvet}

\iclrstyle{
\usepackage[all=normal,mathspacing=tight,wordspacing=tight]{savetrees}
}

\arxiv{\captionsetup{font={small}}}
\iclr{\captionsetup{font={footnotesize}}}

\addtocontents{toc}{\protect\setcounter{tocdepth}{0}}

\title{The Role of Coverage in Online Reinforcement Learning}

\arxivstyle{

\author{
Tengyang Xie\thanks{Equal contribution}
\\
\normalsize
\href{mailto:tx10@illinois.edu}{\texttt{tx10@illinois.edu}}
\and
Dylan J. Foster\footnotemark[1]
\\
\normalsize
\href{mailto:dylanfoster@microsoft.com}{\texttt{dylanfoster@microsoft.com}}
\and
Yu Bai
\\
\normalsize
\href{mailto:yu.bai@salesforce.com}{\texttt{yu.bai@salesforce.com}}
\and
Nan Jiang
\\
\normalsize
\href{mailto:nanjiang@illinois.edu}{\texttt{nanjiang@illinois.edu}}
\and
Sham M. Kakade
\\
\normalsize
\href{mailto:sham@seas.harvard.edu}{\texttt{sham@seas.harvard.edu}}
}
}

\arxivstyle{\date{}}

\begin{document}

\maketitle

\begin{abstract}

\emph{Coverage conditions}---which assert that the data logging
distribution adequately covers the state space---play a fundamental
role in determining the sample complexity of offline reinforcement
learning.
While such conditions might seem irrelevant to online reinforcement
learning at first glance, we establish a new connection by
showing---somewhat surprisingly---that the mere \emph{existence} of a data
distribution with good coverage can enable sample-efficient online
RL. Concretely, we show that \emph{coverability}---that is, existence of a data distribution that satisfies a ubiquitous
coverage condition called concentrability---can be viewed as a
structural property of the underlying MDP, and can be exploited by
standard algorithms for sample-efficient exploration, even when the
agent does not know said distribution. We complement this result by
proving that several weaker notions of coverage, despite being
sufficient for offline RL, are insufficient for online RL. We also show that existing complexity measures for
online RL, including Bellman rank and Bellman-Eluder dimension, fail to optimally capture coverability, and propose a new complexity measure, the \emph{\compmeasure}, to provide a unification.

 \end{abstract}

\section{Introduction}
\label{sec:intro}

The
last decade has seen development of reinforcement learning algorithms with strong empirical performance in domains including robotics
\citep{kober2013reinforcement,lillicrap2015continuous}, dialogue
systems \citep{li2016deep}, and personalization
\citep{agarwal2016making,tewari2017ads}. While there is great interest
in applying these techniques to real-world decision making
applications, the number of samples (steps of interaction) required to
do so is often prohibitive, with state-of-the-art algorithms requiring
millions of samples to reach human-level performance in challenging domains.
Developing
algorithms with improved sample efficiency, which entails efficiently
generalizing across high-dimensional states and actions while taking
advantage of problem structure as modeled practitioners, remains a major challenge.\loose
Investigation into design and analysis of algorithms for sample-efficient
reinforcement learning has largely focused on two distinct problem formulations:
\begin{itemize}
\item \emph{Online reinforcement learning}, where the learner can
  repeatedly interact with the environment by executing a policy and
  observing the resulting trajectory.
\item \emph{Offline reinforcement learning}, where the learner has
  access to logged transitions ands reward gathered from a fixed
  behavioral policy (e.g., historical data or expert demonstrations),
  but cannot directly interact with the underlying environment.
\end{itemize}
While these formulations share a common goal (learning a near-optimal
policy), the algorithms used to achieve this goal and conditions under
which it can be achieved are seemingly quite different. Focusing on
value function approximation, sample-efficient algorithms for online
reinforcement learning require both (a) \emph{representation
  conditions}, which assert that the function approximator is flexible
enough to represent value functions for the underlying MDP (optimal or
otherwise), and (b) \emph{exploration conditions} (or, structural
conditions) which limit the amount of exploration required to learn a
near-optimal policy---typically by enabling extrapolation across
states or limiting the number of effective state distributions \citep{russo2013eluder,jiang2017contextual,sun2019model,wang2020provably,du2021bilinear,jin2021bellman,foster2021statistical}. Algorithms for offline reinforcement learning
typically require similar representation conditions. However, since
data is collected passively from a fixed logging policy/distribution rather than actively, the exploration
conditions used in online RL are replaced with \emph{coverage
  conditions}, which assert that the data collection distribution
provides sufficient coverage over the state space \citep{antos2008learning,chen2019information,xie2020q,xie2021batch, jin2021pessimism,rashidinejad2021bridging,foster2022offline,zhan2022offline}. The aim for both lines of research (online and offline) is to identify the weakest possible conditions under which
learning is possible, and design algorithms that take
advantage of these conditions. The two lines have largely evolved in parallel, and it is
natural to wonder whether there are deeper connections. Since the
conditions for sample-efficient online RL and offline
RL mainly differ via exploration versus coverage, this leads us to ask:\loose

\begin{quote}
\centering
\emph{If an MDP admits a data distribution with favorable
    coverage for offline RL, what does this imply about our ability to perform
    online RL efficiently?}
\end{quote}
Beyond intrinsic theoretical value, this question is motivated by the
observation that many real-world applications lie on a spectrum between
offline and offline. It is common for the learner to have access to
logged/offline data, yet also have the ability to actively interact
with the underlying environment, possibly subject to limitations such
as an exploration budget \citep{kalashnikov2018scalable}. Building a theory of real-world RL that can lead to algorithm design insights for such
settings requires understanding the interplay between online and
offline RL.

\subsection{Our Results}
We investigate connections between coverage conditions in offline
\arxiv{reinforcement learning}\iclr{RL} and exploration in online \arxiv{reinforcement learning}\iclr{RL} by focusing on the \emph{concentrability} coefficient, the most ubiquitous
notion of coverage in offline RL. Concentrability quantifies the
extent to which the data collection distribution uniformly covers 
the state-action distribution induced by any policy. We introduce a new structural property,
\emph{\coverability}, which reflects the best
concentrability coefficient that can achieved by \emph{any data
  distribution}, possibly designed by an oracle with
knowledge of the underlying MDP.
\iclr{Our main results are as follows:}\loose
\begin{enumerate}
\item We show (\pref{sec:basic_cover}) that coverability (that is, mere existence of a distribution with
  good concentrability) is sufficient for sample-efficient online
  exploration, even when the learner has no prior knowledge of this distribution. This result requires no additional assumptions on the
  underlying MDP beyond standard Bellman completeness,
  and---perhaps surprisingly---is
  achieved using standard algorithms \citep{jin2021bellman}, albeit with analysis ideas
  that go beyond existing techniques.
\item We show (\pref{sec:C_gen}) that several weaker notions of coverage in offline RL, including
  single-policy concentrability~\citep{jin2021pessimism,rashidinejad2021bridging} and conditions based on Bellman residuals~\arxiv{\citep{chen2019information,xie2021bellman,cheng2022adversarially}}\iclr{\citep{chen2019information,xie2021bellman}}, are \emph{insufficient} for sample-efficient online
  exploration. This shows that in general, coverage in offline
  reinforcement learning and exploration in online RL not compatible, and highlights the need for additional investigation
  going forward.
\end{enumerate}
Our results serve as a starting point for systematic study of connections between online and
offline learnability in RL. To this end, we provide several secondary
      results:
      \begin{enumerate}
      \item  We show (\pref{sec:gen_bedim}) that existing complexity measures for online RL,
        including Bellman rank and Bellman-Eluder dimension, do not optimally capture coverability, and provide a new complexity
        measure, the \compmeasure, which unifies these notions.
      \item We establish \arxiv{(\pref{sec:exbmdp})}\iclr{(\pref{app:exbmdp})} connections between coverability and reinforcement
        learning with exogenous noise, with applications to
        learning in exogenous block MDPs \citep{efroni2021provably,efroni2022sample}.\loose
      \item We give algorithms for reward-free exploration \citep{jin2020reward,chen2022statistical} under
        coverability (\pref{sec:reward_free}).
      \end{enumerate}
      While our results primarily concern analysis
      of existing algorithms rather than algorithm design, they
      highlight a number of exciting directions for future research,
      and we are
      optimistic that the notion of coverability can guide the design
      of practical algorithms going forward.

\arxiv{\paragraph{Notation}}
\iclr{\noindent\textbf{Notation.}~~}For an integer $n\in\bbN$, we let $[n]$ denote the set
  $\{1,\dots,n\}$. For a set $\cX$, we let
        $\Delta(\cX)$ denote the set of all probability distributions
        over $\cX$. \arxiv{We adopt non-asymptotic big-oh notation: For functions
	$f,g:\cX\to\bbR_{+}$, we write $f=\bigoh(g)$ (resp. $f=\bigom(g)$) if there exists a constant
	$C>0$ such that $f(x)\leq{}Cg(x)$ (resp. $f(x)\geq{}Cg(x)$)
        for all $x\in\cX$. We write $f=\bigoht(g)$ if
        $f=\bigoh(g\cdot\mathrm{polylog}(T))$, $f=\bigomt(g)$ if $f=\bigom(g/\polylog(T))$, and
        $f=\bigthetat(g)$ if $f=\bigoht(g)$ and $f=\bigomt(g)$. %
	We write $f\propto g$ if $f=\bigthetat(g)$.} \iclr{We adopt standard
        big-oh notation, and write $f=\bigoht(g)$ to denote that $f =
        \bigoh(g\cdot{}\max\crl*{1,\mathrm{polylog}(g)})$ and
        $a\approxleq{}b$ as shorthand for $a=\bigoh(b)$.}

\arxiv{\section[Background: Reinforcement Learning, Coverage, and Coverability]{Background:\;Reinforcement\;Learning,\;Coverage,\;and\;Coverability}}
\iclr{\section{Background: Online/Offline RL, Coverage, and
  Coverability}}
\label{sec:background}
\arxiv{
We begin by formally introducing the online and offline reinforcement
learning problems, then review the concept of coverage in offline
reinforcement learning, focusing on \emph{concentrability}. Based on this notion, we introduce
\emph{\coverability} as a structural property.
}

\paragraph{Markov decision processes}
We consider an episodic reinforcement
learning setting. Formally, a Markov decision process $M=(\cX, \cA, P, R,
H, \sinit)$ consists of a (potentially large) state
space $\cX$, action space $\cA$, horizon $H$, probability transition function
$P=\crl{P_h}_{h=1}^{H}$, where $P_h:\cX\times\cA\to\Delta(\cX)$, reward function
$R=\crl{R_h}_{h=1}^{H}$, where $R_h:\cX\times{}\cA\to\brk*{0,1}$, and
deterministic initial state $\sinit \in \cX$.\footnote{While our
  results assume that the initial state is fixed for simplicity, this
  assumption is straightforward to relax.}
A (randomized) policy is a sequence of per-timestep functions
$\pi=\crl*{\pi_h:\cX\to\Delta(\cA)}_{h=1}^{H}$. The policy induces a
distribution over trajectories $(x_1,a_1, r_1),\ldots,(x_H,a_H,r_H)$
via the following process. For $h=1,\ldots,H$: $a_h\sim\pi(\cdot \mid x_h)$,
$r_h=R_h(x_h,a_h)$, and $x_{h+1}\sim{}P_h(\cdot \mid x_h,a_h)$. For notational
convenience, we use $x_{H+1}$ to denote a
  deterministic terminal state with zero reward. We let $\En\sups{\pi}\brk*{\cdot}$ and $\Pr\sups{\pi}\brk{\cdot}$
  denote expectation and probability under this process,
  respectively. 

\arxiv{
  The expected reward for policy $\pi$ is given $J(\pi) \coloneqq
  \Enpi\brk[\big]{\sum_{h=1}^{H}r_h}$, and the value
  function and $Q$-function for $\pi$ are given by\iclr{ $V_h^{\pi}(x)\coloneqq\E^\pi\big[\sum_{h'=h}^{H}r_{h'}\mid{}x_h=x\big],\mathand
    Q_h^{\pi}(x,a)\coloneqq\E^\pi\big[\sum_{h'=h}^{H}r_{h'}\mid{}x_h=x,
      a_h=a\big]$.}
\arxiv{\[\textstyle
    V_h^{\pi}(x)\coloneqq\Enpi\brk*{\sum_{h'=h}^{H}r_{h'}\mid{}x_h=x},\mathand
    Q_h^{\pi}(x,a)\coloneqq\Enpi\brk*{\sum_{h'=h}^{H}r_{h'}\mid{}x_h=x,
      a_h=a}.\]}
}
\iclr{
The $Q$-function for policy $\pi$ is $Q_h^{\pi}(x,a)\coloneqq\E^\pi\big[\sum_{h'=h}^{H}r_{h'}\mid{}x_h=x, a_h=a\big]$, the value function for $\pi$ is $V_h^{\pi}(x) \coloneqq \E_{a \sim \pi_h(\cdot|x)} \brk{Q_h^{\pi}(x,a)}$, and the expected reward for $\pi$ is $J(\pi) \coloneqq V_1^{\pi}(x_1)$.
}
We let \arxiv{$\pistar=\crl*{\pistar_h}_{h=1}^{H}$}\iclr{$\pistar$} denote the optimal (deterministic) policy,
which\arxiv{ satisfies the Bellman equation and} maximizes $\Qpi_h(x,a)$ for all
$(x,a)\in\cX\times\cA$ simultaneously; we define
$\Vstar_h=V^{\pistar}_h$ and $\Qstar_h=Q^{\pistar}_h$. We define the occupancy measure for policy $\pi$ via
$d^{\pi}_h(x,a) \ldef\Prpi\brk{x_h=x,a_h=a}$ and $d^{\pi}_h(x)
\ldef\Prpi\brk{x_h=x}$. We let $\cT_h$ denote the
  Bellman operator for layer $h$, defined via
  $\brk{\cT_hf}(x,a)=R_h(x,a)+\En_{x'\sim{}P_h(x,a)}\brk{\max_{a'}f(x',a')}$
  for $f:\cX\times\cA\to\bbR$.
\arxiv{\paragraph{Assumptions}}
We\iclr{ also} assume that rewards are normalized such that \arxiv{$\sum_{h=1}^{H}r_h\in\brk*{0,1}$}\iclr{$\sum_{h \in [H]}r_h \in [0,1]$} \arxiv{\citep{jiang2018open,wang2020long,zhang2021reinforcement,jin2021bellman}}. To simplify technical presentation, we assume that $\cX$ and $\cA$ are countable; we anticipate that this assumption can be removed.\loose

  \arxiv{\subsection{Online Reinforcement Learning}}
  \iclr{ \paragraph{Online Reinforcement Learning}}
Our main results concern online reinforcement learning in an episodic
framework, where the learner repeatedly interacts with an unknown MDP
by executing a policy and observing the resulting trajectory,
with the goal of maximizing total reward.

Formally, the protocol proceeds in $T$ rounds, where at each round
$t=1,\ldots,T$, the learner:\iclr{ i) Selects a policy $\pi\ind{t}=\crl*{\pi\ind{t}_h}_{h \in [H]}$ to execute in the
  (unknown) underlying MDP $\Mstar$; ii)  Observe the resulting trajectory $(x_1\ind{t},a_1\ind{t},r_1\ind{t}),\ldots,(x_H\ind{t},a_H\ind{t},r_H\ind{t})$.}
\arxiv{
\begin{itemize}
\item Selects a policy $\pi\ind{t}=\crl*{\pi\ind{t}_h}_{h=1}^{H}$ to execute in the
  (unknown) underlying MDP $\Mstar$.
\item Observe the resulting trajectory $(x_1\ind{t},a_1\ind{t},r_1\ind{t}),\ldots,(x_H\ind{t},a_H\ind{t},r_H\ind{t})$.
\end{itemize}
}
The learner's goal is to minimize their cumulative regret, defined
via\iclr{ $\Reg \ldef \sum_{t\in [T]}J(\pistar) - J(\pi\ind{t})$.}
\arxiv{
\begin{align*}
  \Reg \ldef \sum_{t=1}^{T}J(\pistar) - J(\pi\ind{t}).
\end{align*}
}
\iclr{

}
To achieve sample-efficient online reinforcement learning guarantees
that do not depend on the size of the state space, one typically appeals to \emph{value function approximation} methods that take
advantage of a function class $\cF\subset(\cX\times\cA\to\bbR)$ that
attempts to model the value functions for the underlying MDP $\Mstar$
(optimal or otherwise). An active line of research provides structural
conditions under which such approaches succeed
\citep{russo2013eluder,jiang2017contextual,sun2019model,wang2020provably,du2021bilinear,jin2021bellman,foster2021statistical},
based on assumptions that control the interplay between the
function approximator $\cF$ and the dynamics of the MDP
$\Mstar$. These results require (i) \emph{representation conditions}, which
require that $\cF$ is flexible enough to model value functions of
interest (e.g., $\Qstar\in\cF$ or $\cT_h\cF_{h+1}\subseteq\cF_h$) and (ii) \emph{exploration
  conditions}, which either explicitly or implicitly limit
the amount of exploration required for a deliberate algorithm to learn a
near-optimal policy. This is typically accomplished by either enabling extrapolation from
states already visited, or by limiting the number of effective state
distributions that can be encountered.

\arxiv{\subsection{Offline Reinforcement Learning and Coverage
    Conditions}}
\iclr{\paragraph{Offline Reinforcement Learning and Coverage Conditions}}
Our aim is to investigate parallels between online and offline
reinforcement learning. In offline reinforcement learning,
the learner cannot actively execute policies in the underlying MDP
$\Mstar$. Instead, for each layer $h$, they receive a dataset $D_h$ of $n$ tuples $(x_h, a_h, r_h, x_{h+1})$ with $r_h=R_h(x_h,a_h)$,
$x_{h+1}\sim{}P_h(\cdot \mid x_h,a_h)$, and $(x_h,a_h)\sim\mu_h$ \iid, where
$\mu_h\in\Delta(\cX\times\cA)$ is the \emph{data collection
  distribution}; we define \arxiv{$\mu=\crl{\mu_{h}}_{h=1}^{H}$}\iclr{$\mu=\crl{\mu_{h}}_{h\in [H]}$}. The goal of the learner is to use this data to learn an $\veps$-optimal policy $\pihat$, that is:\iclr{ $J(\pistar) - J(\pihat) \leq \veps$.

}
\arxiv{
\[
J(\pistar) - J(\pihat) \leq \veps.
\]
}
Algorithms for offline reinforcement learning require representation
conditions similar to those required for online RL. However, since it
is not possible to actively explore the underlying MDP, one dispenses
with exploration conditions and instead considers \emph{coverage
  conditions}, which require that each data distribution $\mu_h$
sufficiently covers the state space.
As an example, consider \emph{Fitted Q-Iteration} (FQI), one of the most
well-studied offline reinforcement learning algorithms \citep{munos2007performance,munos2008finite,chen2019information}. The algorithm,
which uses least-squares to approximate Bellman backups,
is known to succeed under (i) a representation condition known as \emph{Bellman
  completeness} (or ``completeness''), which requires that
$\cT_hf\in\cF_h$ for all $f\in\cF_{h+1}$,
and (ii) a coverage condition called \emph{concentrability}.
\iclr{To state,
the result, recall that $\nrm{x}_{\infty}\ldef{}\max_{i}\abs{x_i}$ for $x\in\bbR^{d}$.}
\begin{definition}[Concentrability]
\label{def:low_concentrability_offline}
The concentrability coefficient for a data distribution
$\mu=\crl*{\mu_h}_{h=1}^{H}$ and policy class $\Pi$ is given by\iclr{ $\Cconc(\mu) \ldef \sup_{\pi \in \Pi,h\in\brk{H}}\, \nrm*{{d_h^\pi}/{\mu_h}}_{\infty}$.}
\arxiv{
\begin{align*}
  \Cconc(\mu) %
  \ldef
  \sup_{\pi \in \Pi,h\in\brk{H}}\,
  \nrm*{\frac{d_h^\pi}{\mu_h}}_{\infty}.
\end{align*}
}
\end{definition}
Concentrability requires that the data distribution uniformly covers
all possible induced state distributions. 
With concentrability\footnote{Specifically, FQI requires
  concentrability with $\Pi$ chosen to be the set of all admissible
  policies~\citep[see, e.g.,][]{chen2019information}. Other
algorithms \citep{xie2020q} can leverage concentrability w.r.t
smaller policy classes.} and
completeness, FQI can learn an $\veps$-optimal policy using
$\poly(\Cconc(\mu),\log\abs{\cF}, H, \veps^{-1})$ samples. 
Importantly, this result scales only with the concentrability coefficient $\Cconc(\mu)$ and the capacity $\log\abs{\cF}$ for the function class, and has no explicit dependence on the size of the state space. There is a vast literature which provides algorithms with similar, often more refined guarantees \citep{chen2019information,xie2020q,xie2021batch, jin2021pessimism,rashidinejad2021bridging,foster2022offline,zhan2022offline}.

\arxiv{\subsection{The Coverability Coefficient}}
\iclr{\paragraph{The Coverability Coefficient}}
Having seen that access to a data distribution $\mu$ with low concentrability
$\Conc(\mu)$ is sufficient for sample-efficient offline RL, we now ask what
existence of such a distribution implies about our ability to perform
online RL.
To this end, we introduce a
new structural parameter, the \emph{\coverability coefficient}, whose
value reflects the best concentrability coefficient that can be
achieved with oracle knowledge of the underlying MDP $\Mstar$.

\begin{definition}[\Coverability]
\label{def:low_concentrability}
The \coverability coefficient $\Ccov>0$ for a policy class $\Pi$ is
given by\iclr{ $\Ccov \ldef
  \inf_{\mu_1,\ldots,\mu_H\in\Delta(\cX\times\cA)}\crl{\Cconc(\mu)}$
  .}
\arxiv{
\begin{align*}
  \Ccov \ldef \inf_{\mu_1,\ldots,\mu_H\in\Delta(\cX\times\cA)}  \sup_{\pi \in \Pi,h\in\brk{H}}\,
  \nrm*{\frac{d_h^\pi}{\mu_h}}_{\infty}.
\end{align*}
}
\end{definition}
\Coverability is an intrinsic structural property of the MDP $\Mstar$ which
implicitly restricts the complexity of the set of possible state
distributions. While it is always the case that
$\Ccov\leq\abs{\cX}\cdot\abs{\cA}$, the coefficient can be significantly
smaller (in particular, independent of $\abs{\cX}$) for benign MDPs
such as block MDPs and MDPs with low-rank
structure \citep[Prop 5]{chen2019information}\iclr{; see
  \pref{app:exbmdp} for details}. \arxiv{For example in block
MDPs, the state space $\cX$ is potentially very large (e.g., raw
pixels for an Atari game), but can be mapped down to a small number of
unobserved \emph{latent states} (e.g., the game's underlying state machine) which determine the dynamics. In this
case, the \coverability coefficient scales only with the number of
latent states, not with the size of $\cX$; see \pref{sec:exbmdp} for
further discussion and examples.}

With this definition in mind, we ask: {\em If the MDP $\Mstar$
  satisfies low \coverability, is sample-efficient online reinforcement
  learning possible?}  Note that if the learner were given access to data
from the distribution $\mu$ that
achieves the value of $\Ccov$, it would be possible to simply appeal to offline RL methods such as FQI, but since the learner
has no prior knowledge of $\mu$, this question is non-trivial,
and requires deliberate exploration.\loose

\section{Coverability Implies Sample-Efficient Online Exploration}
\label{sec:basic_cover}

We now present our main result, which shows that low \coverability is sufficient for sample-efficient online exploration. We first \arxiv{give}\iclr{describe} the algorithm and regret bound\arxiv{ (\pref{sec:main})}, then \arxiv{prove the result}\iclr{sketch the proof} and give intuition (\pref{sec:golf_sketch}). \arxiv{We conclude (\pref{sec:exbmdp}) by highlighting additional structural properties of \coverability and, as an application, use these properties along with the main result to give regret bounds for learning in \emph{exogenous block MDPs} \citep{efroni2021provably}.}\iclr{We conclude (\pref{sec:exbmdp}) by applying the main result to give regret bounds for learning in \emph{exogenous block MDPs} \citep{efroni2021provably}, highlighting structural properties of \coverability.\loose}

\arxiv{\subsection{Main Result}
  \label{sec:main}
  }

\iclr{\paragraph{Function approximation}}

We work with a value function class $\cF=\cF_1\times\cdots\times\cF_H$, where $\cF_h\subset(\cX\times\cA\to\brk{0,1})$, with the goal of modeling value functions for the underlying MDP. We adopt the convention that $f_{H+1}=0$, and for each $f\in\cF$, we let $\pi_f$ denote the greedy policy with $\pi_{f,h}(x)\ldef\argmax_{a\in\cA}f_h(x,a)$, and we use $f_h(x,\pi_h) \coloneqq \E_{a \sim \pi_h(\cdot|x)}[f_h(x,a)]$ for any $\pi_h$. We take our policy class to be the induced class $\Pi\ldef{}\crl*{\pi_f\mid{}f\in\cF}$ for the remainder of the paper unless otherwise stated. We make the following standard completeness assumption, which requires that the value function class is closed under Bellman backups \citep{wang2020provably,jin2020provably,wang2021optimism,jin2021bellman}.\loose
  \begin{assumption}[Completeness]
    \label{asm:completeness}
    For all $h\in\brk{H}$, we have $\Tcal_h f_{h+1} \in \Fcal_h$ for all $f_{h+1} \in \Fcal_{h+1}$.
  \end{assumption}
Completeness implies that $\cF$ is \emph{realizable} (that is, $\Qstar\in\cF$), but is a stronger assumption in general.
\arxiv{

}
We assume for simplicity that $\abs{\cF}<\infty$, and our results scale with $\log\abs{\cF}$; this can be extended to infinite classes via covering numbers using a standard analysis.
  
  \arxiv{\paragraph{Algorithm}}
  \iclr{\paragraph{Algorithm and main result}}
  Our result is based on a new analysis of the \golf algorithm of \citet{jin2021bellman}, which is presented in \pref{alg:golf} \iclr{of \cref{app:main} for completeness}. \golf is based on the principle of optimism in the face of uncertainty. At each round, the algorithm restricts to a confidence set $\cF\ind{t}\subseteq\cF$ with the property that $\Qstar\in\cF\ind{t}$, and chooses $\pi\ind{t}=\pi_{f\ind{t}}$ based on the value function $f\ind{t}\in\cF\ind{t}$ with the most optimistic estimate $f_1(x_1,\pi_{f,1}(x_1))$ for the total reward. The confidence sets $\cF\ind{t}$ are based on an empirical proxy to squared Bellman error, and are constructed in a \emph{global} fashion that entails optimizing over $f_h$ for all layers $h\in\brk{H}$ simultaneously~\citep{zanette2020learning}.

Note that while \golf was originally introduced to provide regret bounds based on the notion of Bellman-Eluder dimension, we show (\pref{sec:gen_bedim}) that \coverability cannot be (optimally) captured by this complexity measure, necessitating a new analysis.
\iclr{Our main result, \pref{thm:golf_guarantee_basic}, shows that \golf attains low regret for online reinforcement learning whenever the \coverability coefficient is small.}

\arxiv{%
\begin{algorithm}[th]
\caption{\golf \citep{jin2021bellman}}
\label{alg:golf}
{\bfseries input:} Function class $\Fcal$, confidence width $\beta>0$. \\
{\bfseries initialize:} $\Fcal^\iter{0} \leftarrow \Fcal$, $\Dcal_{h}^\iter{0} \leftarrow \emptyset\;\;\forall h \in [H]$. 
\begin{algorithmic}[1]
\For{episode $t = 1,2,\dotsc,T$}
    \State Select policy $\pi^\iter{t} \leftarrow \pi_{f^\iter{t}}$, where $f^\iter{t} \ldef{} \argmax_{f \in \Fcal^\iter{t-1}}f(x_1,\pi_{f,1}(x_1))$. \label{step:glof_optimism}
    \State Execute $\pi^\iter{t}$ for one episode and obtain trajectory $(x_1^\iter{t},a_1^\iter{t},r_1^\iter{t}),\ldots,(x_H^\iter{t},a_H^\iter{t},r_H^\iter{t})$. \label{step:glof_sampling}
    \State Update dataset: $\Dcal_{h}^\iter{t} \leftarrow \Dcal_{h}^\iter{t-1} \cup \crl[\big]{\prn[\big]{x_h^\iter{t},a_h^\iter{t},x_{h+1}^\iter{t}}}\;\;\forall h \in [H]$.
    \State Compute confidence set:
    \begin{gather*}
    \Fcal^\iter{t} \leftarrow \crl[\bigg]{ f \in \Fcal: \Lcal_{h}^\iter{t}(f_h,f_{h+1}) - \min_{f'_h \in \Fcal_h} \Lcal_{h}^\iter{t}(f'_h,f_{h+1}) \leq \beta\;\;\forall h \in [H] },
    \\
    \nonumber
    \text{where \quad } \Lcal_{h}^\iter{t}(f,f') \coloneqq \sum_{(x,a,r,x') \in \Dcal_{h}^\iter{t}}\prn[\Big]{ f(x,a) - r - \max_{a' \in \Acal} f'(x',a') }^2 ,~\forall f,f' \in \Fcal.
    \end{gather*}
\EndFor
\State Output $\pibar = \unif(\pi^\iter{1:T})$. \algcommentlight{For PAC guarantee only.}
\end{algorithmic}
\end{algorithm} %
}

\arxiv{
\paragraph{Main result}
Our main result, \pref{thm:golf_guarantee_basic}, shows that \golf attains low regret for online reinforcement learning whenever the \coverability coefficient is small.
}
\begin{theorem}[\Coverability implies sample-efficient online RL]
\label{thm:golf_guarantee_basic}
Under \cref{asm:completeness}, there exists an absolute constant $c$ such that for any $\delta \in (0,1]$ and $T \in \NN_+$, if we choose $\beta = c \cdot \log(\nicefrac{TH |\Fcal|}{\delta})$ in \cref{alg:golf}, then with probability at least $1 - \delta$, we have\iclr{ $\Reg  \leq O ( H \sqrt{\Ccov T \log(\nicefrac{TH |\Fcal|}{\delta}) \log(T) } )$,}
\arxiv{
\begin{align*}
  \Reg  \leq O \left( H \sqrt{\Ccov T \log(\nicefrac{TH |\Fcal|}{\delta}) \log(T) } \right),
\end{align*}
}
where $C_\on$ is the coverability coefficient (\cref{def:low_concentrability}).
\end{theorem}
Beyond the coverability parameter $\Ccov$, the regret bound in \pref{thm:golf_guarantee_basic} depends only on standard problem parameters (the horizon $H$ and function class capacity $\log\abs{\cF}$). Hence, this result shows that \coverability, along with completeness, is sufficient for sample-efficient online RL.%
\arxiv{

}
Additional features of \pref{thm:golf_guarantee_basic} are as follows.
\begin{itemize}
\item While \coverability implies that there exists a distribution $\mu$ for which the concentrability coefficient $\Cconc$ is bounded, \pref{alg:golf} has no prior knowledge of this distribution. We find the fact that the \golf algorithm---which does not involve explicitly searching such a distribution---succeeds under this condition to be somewhat surprising (recall that given sample access to $\mu$, one can simply run FQI). Our proof shows that despite the fact that \golf does not explicitly reason about $\mu$, \coverability implicitly restricts the set of possible state distributions, and limits the extent to which the algorithm can be ``surprised'' by substantially new distributions. We anticipate that this analysis will find broader use.\loose
\item Ignoring factors logarithmic in $T$, $H$, and $\delta^{-1}$, the regret bound in \pref{thm:golf_guarantee_basic} scales as
  $H\sqrt{\Ccov{}T\log\abs{\cF}}$, which is optimal for contextual bandits (where $\Ccov=\abs{\cA}$ and $H=2$),\footnote{\arxiv{Since we assume a deterministic starting state, we require $H=2$ rather than $H=1$ to apply the result to contextual bandits.}\iclr{We require $H=2$ to apply the result to contextual bandits due to assuming the deterministic starting state.}} and hence cannot be improved in general \citep{agarwal2012contextual}. The dependence on $H$ matches the regret bound for \golf based on Bellman-Eluder dimension \citep{jin2021bellman}.
\item \golf uses confidence sets based on squared Bellman error, but there are similar algorithms which instead work with average Bellman error \citep{jiang2017contextual,du2021bilinear} and, as a result, require only realizability rather than completeness (\pref{asm:completeness}). While existing complexity measures such as Bellman rank and \betext can be used to analyze both types of algorithm, and our results critically use the non-negativity of squared Bellman error, which facilitates certain ``change-of-measure'' arguments. Consequently, it is unclear whether the completeness assumption can be removed (i.e., whether \coverability and realizability alone suffice for sample-efficient online RL).
\end{itemize}

On the algorithmic side, our results give guarantees for PAC RL via online-to-batch conversion, which we state here for completeness. We also provide an extension to reward-free exploration in \pref{sec:reward_free}. 
\begin{corollary}
\label{cor:basic_c_batch}
Under \cref{asm:completeness}, there exists an absolute constant $c$ such that for any $\delta \in (0,1]$ and $T \in \NN_+$, if we choose $\beta = c \cdot \log(\nicefrac{TH |\Fcal|}{\delta})$ in \cref{alg:golf}, then with probability at least $1 - \delta$, the policy $\pibar$ output by \cref{alg:golf} has\footnote{$\pibar$ is the non-Markov policy obtained by sampling $t\sim\brk{T}$ and playing $\pi\ind{t}$.} \iclr{$J(\pi^\star) - J(\pibar) \leq O \big( H\sqrt{C_\on \log(\nicefrac{TH |\Fcal|}{\delta}) \log(T)/T}\big)$.}
\arxiv{\begin{align*}
J(\pi^\star) - J(\pibar) \leq O \left( H\sqrt{\frac{C_\on \log(\nicefrac{TH |\Fcal|}{\delta}) \log(T)}{T}}\right).
\end{align*}}
\end{corollary}

\arxiv{\subsection{Proof of \creftitle{thm:golf_guarantee_basic}: Why is
    Coverability Sufficient?}}
\iclr{\subsection{Proof Sketch for \creftitle{thm:golf_guarantee_basic}: Why is
  Coverability Sufficient?}}
\label{sec:golf_sketch}
\newcommand{\reachability}{cumulative reachability\xspace}
\newcommand{\Reachability}{Cumulative Reachability\xspace}

\arxiv{We now prove \pref{thm:golf_guarantee_basic}, highlighting the role of \coverability in limiting the complexity of exploration.}
\paragraph{Equivalence to cumulative reachability}

\iclr{
We now sketch the main ideas behind the proof of \pref{thm:golf_guarantee_basic}, highlighting the role of \coverability in limiting the complexity of exploration.
  }

\arxiv{
A key idea underlying the proof of \pref{thm:golf_guarantee_basic} is the equivalence between \coverability and a quantity we term \emph{cumulative reachability}. Define the \emph{reachability} for a tuple $(x,a,h) \in \Xcal \times \Acal \times [H]$ by
$\sup_{\pi \in \Pi} d_h^\pi(x,a)$,
which captures the greatest probability of reaching $(x,a)$ at layer $h$ that can be achieved with any policy. We define cumulative reachability by
\begin{equation*}
\iclr{\textstyle}
\sum_{(x,a) \in \Xcal \times \Acal}\sup_{\pi \in \Pi} d_h^\pi(x,a).
\end{equation*}
Cumulative reachability reflects the variation in visitation probabilities for policies in the class $\Pi$. In particular, cumulative reachability is low when the state-action pairs visited by policies in $\Pi$ have large overlap, and vice versa; see \cref{fig:coverage} for an illustration.  

\arxiv{
\begin{figure}[tp]
    \centering
    \includegraphics[width=0.43\textwidth]{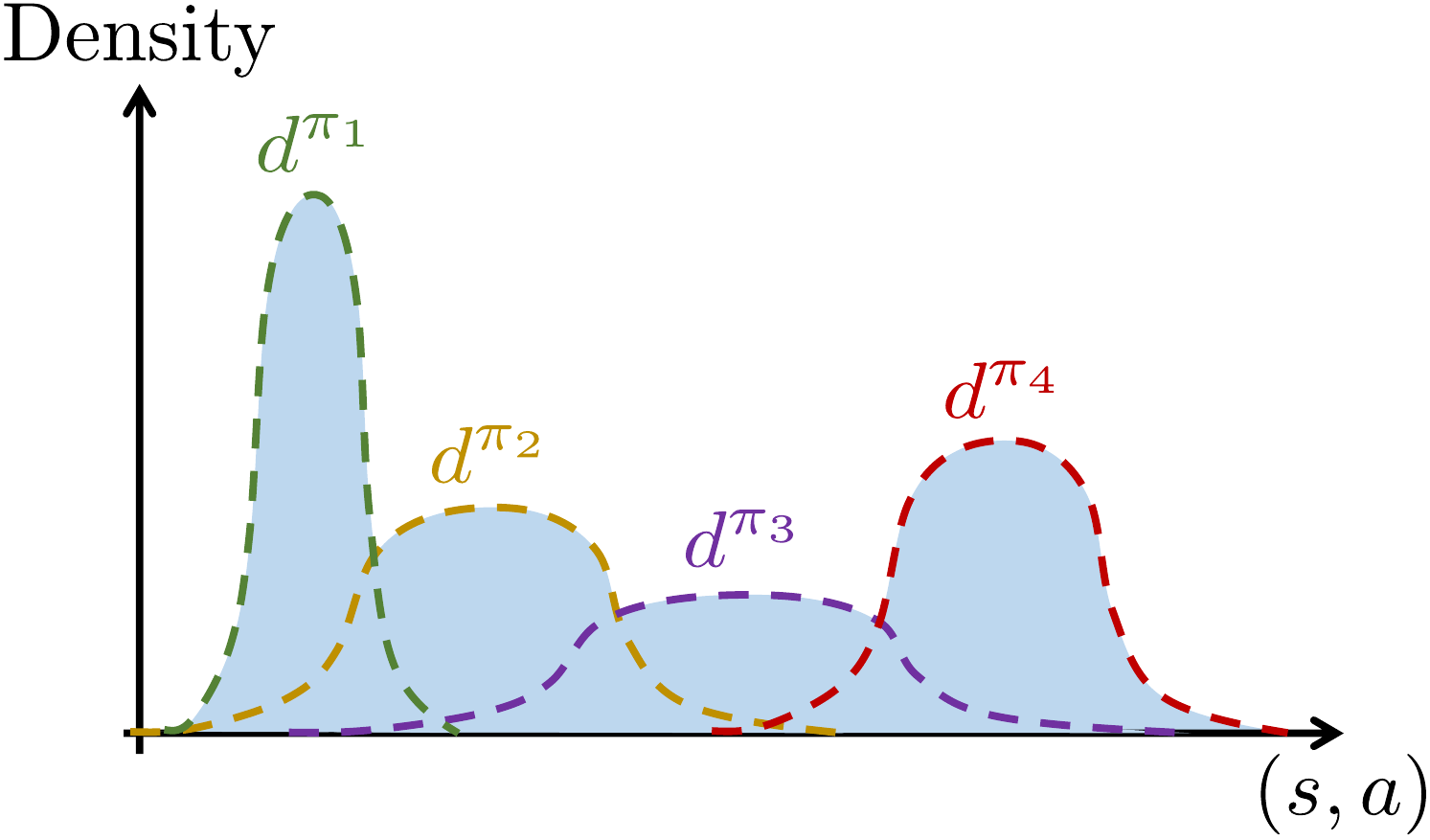}
    \caption{An example illustrating the equivalence of \emph{coverability} and \emph{\reachability}. Here, $\Pi = \{\pi_1,\pi_2,\pi_3,\pi_4\}$, and dashed curves
  plots $d^\pi$ for each $\pi \in \Pi$. The coverability coefficient, via \pref{lem:concen_eq_area},
  is equal to the total area of the shaded region (without
  double-counting overlapping regions).}
  \label{fig:coverage}
\end{figure}
}

The following lemma shows that \reachability and \coverability coincide; we defer the proof to \cref{app:main}.

\begin{lemma}[Equivalence of \coverability and \reachability]
\label{lem:concen_eq_area}
The following definition is equivalent to \pref{def:low_concentrability}:\iclr{ $\Ccov \ldef{} \max_{h\in\brk{H}}\sum_{(x,a) \in \Xcal \times \Acal}\sup_{\pi \in \Pi} d_h^\pi(x,a)$.}
\arxiv{
\begin{align*}
\Ccov \ldef{} \max_{h\in\brk{H}}\sum_{(x,a) \in \Xcal \times \Acal}\sup_{\pi \in \Pi} d_h^\pi(x,a).
\end{align*}
}
\end{lemma}
}

\iclr{
\paragraph{Regret decomposition and change of measure}
For each $t$, we define $\delta^\iter{t}_h(\cdot,\cdot) \coloneqq f_h^\iter{t}(\cdot,\cdot) - (\Tcal_h f_{h+1}^\iter{t})(\cdot,\cdot)$, which
may be viewed as a ``test function'' at level $h$ induced by $f\ind{t} \in \Fcal$. We adopt the shorthand $d_h^\iter{t}\equiv{}d_h^{\pi^\iter{t}}$, and we define $\dtilde_h^\iter{t} (x,a) \coloneqq~ \sum_{i = 1}^{t - 1} d_h^\iter{i} (x,a)$ as the \arxiv{unnormalized average of all state}\iclr{cumulative historical} visitation for rounds prior to step $t$.

A standard regret decomposition for optimistic algorithms (\pref{lem:regret_optimistic}\arxiv{ in \pref{app:main}}) allows us to relate regret to the average Bellman error under the learner's sequence of policies:
\begin{align}
  \label{eq:reg_decomposition_sketch}
  \Reg \leq{} \sum_{t \in [T]} \left( f_1^\iter{t}(x_1,\pi_{f\ind{t}_1,1}(x_1)) -J(\pi\ind{t}) \right) 
  = \sum_{t \in [T]} \sum_{h \in [H]}\E_{d_h^\iter{t}}\big[\underbrace{f_h\ind{t}(x,a)-(\cT_hf_{h+1}\ind{t})(x,a)}_{\rdef\delta_h^\iter{t}(x,a)}\big].
\end{align}
Fix $h\in\brk{H}$. We use a change-of-measure argument to relate the on-policy \emph{average} Bellman error $\E_{(x,a)\sim{}d_h^\iter{t}}[\delta_h^\iter{t}(x,a)]$ to the in-sample \emph{squared} Bellman error under $\dtil_h\ind{t}$, writing \pref{eq:reg_decomposition_sketch} as \loose
\begin{align*}
\resizebox{\linewidth}{!}{$\displaystyle
 ~ \sum_{t \in [T]} \sum_{x,a} d_h^\iter{t}(x,a) \left( \frac{\dtilde_h^\iter{t}(x,a)}{\dtilde_h^\iter{t}(x,a)} \right)^{\nicefrac{1}{2}} \delta_h^\iter{t}(x,a)
\leq~ \underbrace{\sqrt{\sum_{t \in [T]} \sum_{x,a} \frac{\left( d_h^\iter{t}(x,a) \right)^2}{\dtilde_h^\iter{t}(x,a)} }}_{\texttt{(I): extrapolation error}} \cdot  \underbrace{\sqrt{\sum_{t \in [T]} \sum_{x,a} \dtilde_h^\iter{t}(x,a) \left(\delta_h^\iter{t}(x,a)\right)^2}}_{\texttt{(II): in-sample \emph{squared} Bellman error}},
$}
\end{align*}
where the inequality is an application of Cauchy-Schwarz. As an immediate consequence of the confidence set construction in~\cref{eq:def_vspace}, completeness, and a standard concentration argument (\pref{lem:golf_concentration} in \pref{app:main}), we can bound the in-sample error by $\texttt{(II)} \leq O\prn[\big]{\sqrt{\beta T}}$.

\paragraph{Bounding the extrapolation error using \coverability}
To proceed, we show that the extrapolation error \texttt{(I)} is controlled by \coverability. 
We have:
\begin{align*}
\resizebox{\linewidth}{!}{$\displaystyle
\sum_{t \in [T]} \sum_{x,a} \frac{\left( d_h^\iter{t}(x,a) \right)^2}{\dtilde_h^\iter{t}(x,a)}
\leq ~ \sum_{t \in [T]} \sum_{x,a} \max_{t'\in\brk{T}}d_h\ind{t'}(x,a)\cdot\frac{ d_h^\iter{t}(x,a) }{\dtilde_h^\iter{t}(x,a)}
\leq \underbrace{\bigg( \max_{x,a} \sum_{t \in [T]} \frac{ d_h^\iter{t}(x,a)}{\dtilde_h^\iter{t}(x,a)} \bigg)}_{\overset{\texttt{(a)}}{\approxleq} \bigoh(\log(T)) \text{ by \cref{lem:per_sa_ep}}} \cdot \underbrace{\bigg( \sum_{x,a}\max_{t\in\brk{T}} d\ind{t}_h(x,a) \bigg)}_{\overset{\texttt{(b)}}{\leq} C_\on \text{ by \pref{lem:concen_eq_area}}}.
$}
\end{align*}
\iclr{
\begin{wrapfigure}{r}{0.37\textwidth}
\centering
\iclr{\vspace{-5mm}}
\includegraphics[width=\linewidth]{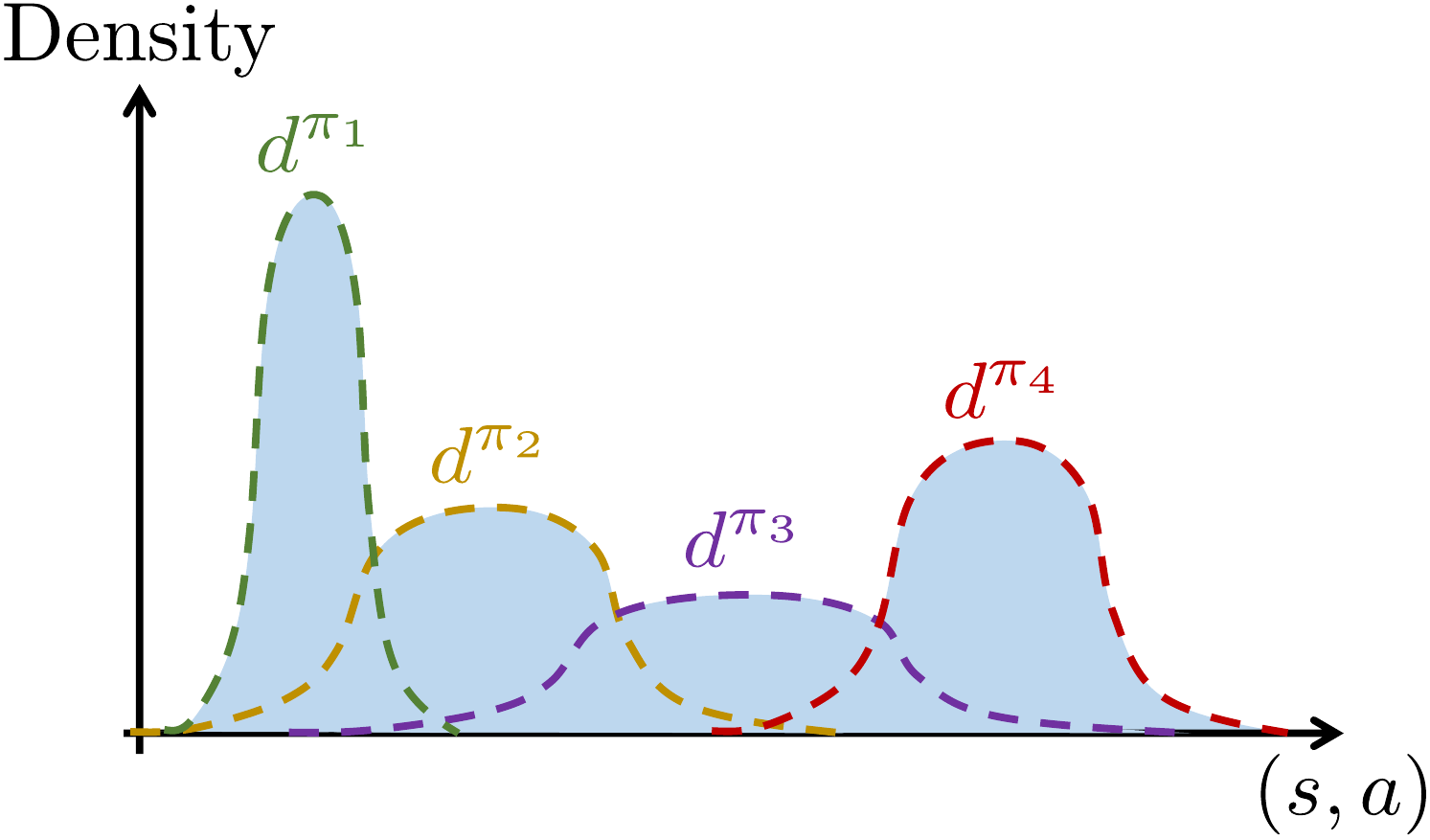}
\iclr{\vspace{-7.5mm}}
\arxiv{
\caption{An example illustrating the equivalence of \emph{coverability} and \emph{\reachability}. Here, $\Pi = \{\pi_1,\pi_2,\pi_3,\pi_4\}$, and dashed curves
  plots $d^\pi$ for each $\pi \in \Pi$. The coverability coefficient, via \pref{lem:concen_eq_area},
  is equal to the total area of the shaded region (without
  double-counting overlapping regions).}
}
\iclr{
\caption{An example of \emph{coverability} $\Longleftrightarrow$~\emph{\reachability} (which is equal to the total area of the shaded region without double-counting overlaps. $\Pi = \{\pi_1,\pi_2,\pi_3,\pi_4\}$, dashed curves is $d^\pi$).\loose}
}
\label{fig:coverage}
\vspace{-2mm}
\iclr{\vspace{-7mm}}
\end{wrapfigure}
}
Here, the inequality \texttt{(a)} uses a \emph{scalar} variant of the elliptic potential lemma (\pref{lem:per_sa_ep}; cf. \citet{lattimore2020bandit}), which we apply on a \emph{per-state basis}.\footnote{Applying this result formally requires a separate argument to handle early rounds in which pairs $(x,a)$ have been visited very little; this is given in \pref{app:main}.} The inequality \texttt{(b)} uses a key result (\pref{lem:concen_eq_area} in \pref{app:main}), which shows that \coverability is equivalent to a quantity we term \emph{cumulative reachability}, defined via $\sum_{(x,a) \in \Xcal \times \Acal}\sup_{\pi \in \Pi} d_h^\pi(x,a)$. Cumulative reachability reflects the variation in visitation probabilities for policies in the class $\Pi$, and boundedness of this quantity (which occurs when state-action pairs visited by policies in $\Pi$ have large overlap) implies that the contributions from potentials for different state-action pairs average out. See \cref{fig:coverage} for an illustration.\loose

To conclude, we substitute the preceding bounds into the term \texttt{(I)}, which gives
$\Reg \leq \sum_{h = 1}^{H} \E_{(x,a)\sim{}d_h^\iter{t}}\left[\delta_h^\iter{t}(x,a)\right] \leq O \big( H \sqrt{\Ccov\cdot\beta T \log(T) } \big)$.

Note that to obtain the expression in term \texttt{(I)}, our proof critically uses that the confidence set construction provides a bound on the \emph{squared Bellman error} $\E_{(x,a)\sim{}\dtilde_h^\iter{t}}[\delta_h^\iter{t}(x,a)^2]$ in the change of measure argument. This contrasts with existing works on online RL with general function approximation~\citep[e.g.,][]{jiang2017contextual,jin2021bellman,du2021bilinear}, which typically move from average Bellman error to squared Bellman error as a lossy step, and only work with squared Bellman error because it permits simpler construction of confidence sets. \arxiv{For the argument in \pref{eq:reg_CS}, confidence}\iclr{Confidence} sets based on average Bellman error will lead to a larger notion of extrapolation error which cannot be controlled using \coverability (cf. \pref{sec:gen_bedim}).\loose

}

\arxiv{
\arxiv{Equipped with \pref{lem:concen_eq_area}, we proceed with the proof of \pref{thm:golf_guarantee_basic}.}
\iclr{Equipped with \pref{lem:concen_eq_area}, we prove \pref{thm:golf_guarantee_basic}.}

\paragraph{Preliminaries}
For each $t$, we define $\delta^\iter{t}_h(\cdot,\cdot) \coloneqq f_h^\iter{t}(\cdot,\cdot) - (\Tcal_h f_{h+1}^\iter{t})(\cdot,\cdot)$, which may be viewed as a ``test function'' at level $h$ induced by $f\ind{t} \in \Fcal$. We adopt the shorthand $d_h^\iter{t}\equiv{}d_h^{\pi^\iter{t}}$, and we define
\begin{align}
\label{eq:def_dbar}
  &\dtilde_h^\iter{t} (x,a) \coloneqq~ \sum_{i = 1}^{t - 1} d_h^\iter{i} (x,a),\mathand
  \mu^\star_h \coloneqq ~ \argmin_{\mu_h\in\Delta(\cX\times\cA)} \sup_{\pi\in\Pi}\, \nrm*{\frac{d_h^\pi}{\mu_h}}_{\infty}.
\end{align}
That is, $\dtilde_h^\iter{t}$ unnormalized average of all state visitations encountered prior to step $t$, and $\mu^\star_h$ is the distribution that attains the value of $\Ccov$ for layer $h$.\footnote{If the minimum in \pref{eq:def_dbar} is not obtained, we can repeat the argument that follows for each element of a limit sequence attaining the infimum.} Throughout the proof, we perform a slight abuse of notation and write $\E_{\dtilde_h^\iter{t}}[f] \coloneqq \sum_{i = 1}^{t - 1} \E_{d_h^\iter{i}}[f]$ for any function $f:\Xcal \times \Acal \to \RR$.

\paragraph{Regret decomposition}
As a consequence of completeness (\pref{asm:completeness}) and the construction of $\cF\ind{t}$, a standard concentration argument (\pref{lem:golf_concentration}\arxiv{ in \pref{app:main}}) guarantees that with probability at least $1-\delta$, for all $t\in\brk{T}$:
\begin{align}
\mathrm{(i)}\;\;\Qstar\in\cF\ind{t},\mathand\mathrm{(ii)}\;\;\sum_{x,a} \dtilde_h^\iter{t}(x,a) \left(\delta_h^\iter{t}(x,a)\right)^2\leq\bigoh(\beta).\label{eq:golf_concentration}
\end{align}
We condition on this event going forward. Since $\Qstar\in\cF\ind{t}$, we are guaranteed that $f\ind{t}$ is optimistic (i.e., $f_1\ind{t}(x_1,\pi_{f\ind{t},1}(x_1))\geq{}\Qstar_1(x_1,\pi_{\fstar,1}(x_1))$), and a regret decomposition for optimistic algorithms (\pref{lem:regret_optimistic}\arxiv{ in \pref{app:main}}) allows us to relate regret to the average Bellman error under the learner's sequence of policies:
\begin{align*}
  \Reg \leq{} \sum_{t = 1}^{T} \left( f_1^\iter{t}(x_1,\pi_{f\ind{t}_1,1}(x_1)) -J(\pi\ind{t}) \right) 
  = \sum_{t = 1}^{T} \sum_{h=1}^{H}\E_{(x,a)\sim{}d_h^\iter{t}}\big[\underbrace{f_h\ind{t}(x,a)-(\cT_hf_{h+1}\ind{t})(x,a)}_{\rdef\delta_h^\iter{t}(x,a)}\big].
\end{align*}
To proceed, we use a change of measure argument to relate the on-policy \emph{average} Bellman error $\E_{(x,a)\sim{}d_h^\iter{t}}[\delta_h^\iter{t}(x,a)]$ appearing above to the in-sample \emph{squared} Bellman error $\E_{(x,a)\sim{}\dtilde_h^\iter{t}}[\delta_h^\iter{t}(x,a)^2]$; the latter is small as a consequence of \pref{eq:golf_concentration}. Unfortunately, naive attempts at applying change-of-measure fail because during the initial rounds of exploration, the on-policy and in-sample visitation probabilities can be very different, making it impossible to relate the two quantities (i.e., any natural notion of extrapolation error will be arbitrarily large).

To address this issue, we introduce the notion of a ``burn-in'' phase for each state-action pair $(x,a)\in\cX\times\cA$ by defining
\[
\tau_h(x,a) = \min\left\{t \mid \dtilde_h^\iter{t}(x,a) \geq \Ccov\cdot\mu^{\star}_h(x,a)\right\},
\] %
which captures the earliest time at which $(x,a)$ has been explored sufficiently; we refer to $t<\tau_h(x,a)$ as the burn-in phase for $(x,a)$.

Going forward, let $h\in\brk{H}$ be fixed. We decompose regret into contributions from the burn-in phase for each state-action pair, and contributions from pairs which have been explored sufficiently and reached a stable phase ``stable phase''.
\begin{align*}
\underbrace{\sum_{t=1}^{T}\E_{(x,a)\sim{}d_h^\iter{t}}\left[\delta_h^\iter{t}(x,a)\right]}_{\text{on-policy average Bellman error}} = \underbrace{\sum_{t=1}^{T}\E_{(x,a)\sim{}d_h^\iter{t}}\left[\delta_h^\iter{t}(x,a)\1[t < \tau_h(x,a)]\right]}_{\text{burn-in phase}} +
  \underbrace{\sum_{t=1}^{T}\E_{(x,a)\sim{}d_h^\iter{t}}\left[\delta_h^\iter{t}(x,a)\1[t \geq \tau_h(x,a)]\right]}_{\text{stable phase}}.
\end{align*}
We will not show that every state-action pair leaves the burn-in phase. Instead, we use \coverability to argue that the contribution from pairs that have not left this phase is small on average.
In particular, we use that $\abs{\delta_h\ind{t}}\leq{}1$ to bound
\begin{align*}
  \sum_{t = 1}^{T} \E_{(x,a)\sim{}d_h^\iter{t}}\left[\delta_h^\iter{t}(x,a)\1[t < \tau_h(x,a)]\right] \leq
  \sum_{x,a}\sum_{t<\tau_h(x,a)} d_h^\iter{t}(x,a)=
  \sum_{x,a}\dtil_h\ind{\tau_h(x,a)}(x,a)
  \leq{}2\Ccov\sum_{x,a}\mustar_h(x,a)=2\Ccov,   
\end{align*}
where the last inequality holds because \[\dtil \ind{\tau_h(x,a)}_h(x,a) = \dtilde_h^\iter{\tau_h(x,a) - 1}(x,a) + d_h^\iter{\tau_h(x,a) - 1}(x,a) \leq 2C_\on\cdot\mu_h^\star(x,a),\] which follows from \cref{eq:def_dbar} and the definition of $\tau_h$.%

For the stable phase, we apply change-of-measure as follows:
\begin{align}
\nonumber
&~ \sum_{t = 1}^{T} \E_{(x,a)\sim{}d_h^\iter{t}}\left[\delta_h^\iter{t}(x,a)\1[t \geq \tau_h(x,a)]\right]
\\
\nonumber
&= ~ \sum_{t = 1}^{T} \sum_{x,a} d_h^\iter{t}(x,a) \left( \frac{\dtilde_h^\iter{t}(x,a)}{\dtilde_h^\iter{t}(x,a)} \right)^{\nicefrac{1}{2}} \delta_h^\iter{t}(x,a) \1[t \geq \tau_h(x,a)]
\\
\label{eq:reg_CS}
&\leq~ \underbrace{\sqrt{\sum_{t = 1}^{T} \sum_{x,a} \frac{\left( \1[t \geq \tau_h(x,a)] d_h^\iter{t}(x,a) \right)^2}{\dtilde_h^\iter{t}(x,a)} }}_{\text{{\tt (I)}: extrapolation error}} \cdot  \underbrace{\sqrt{\sum_{t = 1}^{T} \sum_{x,a} \dtilde_h^\iter{t}(x,a) \left(\delta_h^\iter{t}(x,a)\right)^2}}_{\text{{\tt (II)}: in-sample \emph{squared} Bellman error}},
\end{align}
where the last inequality is an application of Cauchy-Schwarz. Using part \texttt{(II)} of \pref{eq:golf_concentration}, we bound the in-sample error above by
\begin{align}
\label{eq:reg_CS_term2_main}
\texttt{(II)} \leq O\prn[\big]{\sqrt{\beta T}}.
\end{align}

\paragraph{Bounding the extrapolation error using \coverability}
To proceed, we show that the extrapolation error \texttt{(I)} is controlled by \coverability. 
We begin with a scalar variant of the standard elliptic potential lemma \citep{lattimore2020bandit}; \arxiv{this result is proven in \cref{app:main} for completeness.}\iclr{this result is proven in the sequel.}
\begin{lemma}[Per-state-action elliptic potential lemma]
\label{lem:per_sa_ep}
Let $d^\iter{1}, d^\iter{2}, \dotsc, d^\iter{T}$ be an arbitrary sequence of distributions over a set $\Zcal$ (e.g., $\Zcal = \Xcal \times \Acal$), and let $\mu\in\Delta(\Zcal)$ be a distribution such that $d^\iter{t}(z) / \mu(z) \leq C$ for all $(z,t) \in \Zcal \times [T]$. Then for all $z \in \Zcal$, we have
\begin{align*}
\sum_{t = 1}^{T} \frac{d^\iter{t}(z)}{\sum_{i < t} d^\iter{i}(z) + C \cdot \mu(z)} \leq O \left(\log\left( T \right) \right).
\end{align*}
\end{lemma}
We bound the extrapolation error \texttt{(I)} by applying \pref{lem:per_sa_ep} on a \emph{per-state basis}, then using \coverability (and the equivalence to \reachability) to argue that the potentials from different state-action pairs average out. Observe that by the definition of $\tau_h$, we have that for all $t \geq \tau_h(s,a)$, $\dtilde_h^\iter{t}(x,a) \geq C_\on \mu_h^\star(x,a) \Rightarrow \dtilde_h^\iter{t}(x,a)\geq \frac{1}{2}(\dtilde_h^\iter{t}(x,a) + C_\on \mu_h^\star(x,a))$, which allows us to bound term \texttt{(I)} of extrapolation error by
\begin{align}
\nonumber
\sum_{t = 1}^{T} \sum_{x,a} \frac{\left( \1[t \geq \tau_h(x,a)] d_h^\iter{t}(x,a) \right)^2}{\dtilde_h^\iter{t}(x,a)} \leq &~ 2 \sum_{t = 1}^{T} \sum_{x,a} \frac{ d_h^\iter{t}(x,a) \cdot d_h^\iter{t}(x,a)}{\dtilde_h^\iter{t}(x,a) + C_\on \cdot \mu_h^\star(x,a)}
\\
\leq &~ 2\sum_{t = 1}^{T} \sum_{x,a} \max_{t'\in\brk{T}}d_h\ind{t'}(x,a)\cdot\frac{ d_h^\iter{t}(x,a) }{\dtilde_h^\iter{t}(x,a) + C_\on \cdot \mu_h^\star(x,a)}
       \nonumber
\\
\nonumber
\leq &~ 2\underbrace{\left( \max_{(s,a) \in \Scal \times \Acal} \sum_{t = 1}^{T} \frac{ d_h^\iter{t}(x,a)}{\dtilde_h^\iter{t}(x,a) + C_\on \cdot \mu_h^\star(x,a)} \right)}_{\leq \bigoh(\log(T)) \text{ by \cref{lem:per_sa_ep}}} \cdot \underbrace{\left( \sum_{x,a}\max_{t\in\brk{T}} d\ind{t}_h(x,a) \right)}_{\leq C_\on \text{ by \pref{lem:concen_eq_area}}}
\\
\label{eq:reg_CS_term1_main}
\leq &~ O \left( C_\on \log\left( T \right) \right).
\end{align}
To conclude, we substitute \cref{eq:reg_CS_term2_main,eq:reg_CS_term1_main} into \cref{eq:reg_CS}, which gives
\begin{align*}
\Reg \leq \sum_{h = 1}^{H} \E_{(x,a)\sim{}d_h^\iter{t}}\left[\delta_h^\iter{t}(x,a)\right] \leq O \left( H \sqrt{\Ccov\cdot\beta T \log(T) } \right).
\end{align*}
\arxiv{
\qed

  To obtain the expression in \pref{eq:reg_CS} (term \texttt{(I)}), our proof critically uses that the confidence set construction provides a bound on the \emph{squared Bellman error} $\E_{(x,a)\sim{}\dtilde_h^\iter{t}}[\delta_h^\iter{t}(x,a)^2]$ in the change of measure argument. This contrasts with existing works on online RL with general function approximation~\citep[e.g.,][]{jiang2017contextual,jin2021bellman,du2021bilinear}, which typically move from average Bellman error to squared Bellman error as a lossy step, and only work with squared Bellman error because it permits simpler construction of confidence sets. For the argument in \pref{eq:reg_CS}, confidence sets based on average Bellman error will lead to a larger notion of extrapolation error which cannot be controlled using \coverability (cf. \pref{sec:gen_bedim}).
}
 }

\subsection{Rich Observations and Exogenous Noise: Application to Block MDPs}
\label{sec:exbmdp}
\newcommand{\Pendo}{P^{\mathrm{endo}}}%
\newcommand{\Pexo}{P^{\mathrm{exo}}}%
\newcommand{\exmdp}{Ex-BMDP\xspace}%
\newcommand{\exmdps}{Ex-BMDPs\xspace}%
\newcommand{\phistar}{\phi^{\star}}

\iclr{
  As an application of \pref{thm:golf_guarantee_basic}, we consider the problem of reinforcement learning in \emph{Exogenous Block MDPs} (\exmdps), a problem which has received extensive recent interest \citep{efroni2021provably,efroni2022sample,efroni2022sparsity,lamb2022guaranteed}. Recall that the block MDP \citep{jiang2017contextual,du2019latent,misra2020kinematic} is a model in which the (``observed'') state space $\cX$ is large/high-dimensional, but can be mapped by an (unknown) decoder $\phistar$ to a small \emph{latent} state space which governs the dynamics. Exogenous block MDPs generalize this model further by factorizing the latent state space into small controllable (``endogenous'') component $\cS$ and a large irrelevant (``exogenous'') component $\Xi$, which may be temporally correlated.

The main challenge of learning in block MDPs is that the decoder $\phistar$ is not known to the learner in advance. Indeed, given access to the decoder, one can obtain regret $\poly(H,\abs{\cS},\abs{\cA})\cdot\sqrt{T}$ by applying tabular reinforcement learning algorithms to the latent state space. In light of this, the aim of the \exmdp setting is to obtain sample complexity guarantees that are independent of the size of the observed state space $\abs{\cX}$ and exogenous state space $\abs{\Xi}$, and scale as $\poly(\abs{\cS}, \abs{\cA}, H,\log\abs{\cF})$, where $\cF$ is an appropriate class of function approximators (typically either a value function class $\cF$ or a class of decoders $\Phi$ that attempts to model $\phistar$ directly).  

We show (\pref{prop:ex_bmdp} in \pref{app:exbmdp}) that for any \exbmdp, one has $\Ccov\leq\abs{\cS}\cdot\abs{\cA}$, which---through \pref{thm:golf_guarantee_basic}---implies that \golf attains $\Reg\leq \bigoh\prn[\big]{
    H\sqrt{\abs{\cS}\abs{\cA}T \log(\nicefrac{TH |\Fcal|}{\delta})\log(T)}
    }$
  whenever \pref{asm:completeness} holds; critically, this result scales only with the cardinality $\abs{\cS}$ for the endogenous latent state space, and with the capacity $\log\abs{\cF}$ for the value function class. It is the first result for this setting that allows for stochastic latent dynamics and emission process, albeit with the extra assumption of completeness. Existing algorithms either require that the endogenous latent dynamics $\Pendo$ are deterministic \citep{efroni2021provably} or allow for stochastic dynamics but heavily restrict the observation process \citep{efroni2022sample}, and existing complexity measures such as Bellman Rank and \betext can be arbitrarily large for this setting (see discussion in \pref{sec:gen_bedim}). See \pref{app:exbmdp} for details and discussion.\loose

}

\arxiv{
As an application of \pref{thm:golf_guarantee_basic}, we consider the problem of reinforcement learning in \emph{Exogenous Block MDPs} (\exmdps), a problem which has received extensive recent interest \citep{efroni2021provably,efroni2022sample,efroni2022sparsity,lamb2022guaranteed}. Recall that the block MDP \citep{jiang2017contextual,du2019latent,misra2020kinematic} is a model in which the (``observed'') state space $\cX$ is large/high-dimensional, but the dynamics are governed by a (small) latent state space. Exogenous block MDPs generalize this model further by factorizing the latent state space into small controllable (``endogenous'') component and a large irrelevant (``exogenous'') component, which may be temporally correlated.

Following \cite{efroni2021provably}, an \exbmdp $M=(\cX,\cA,P,R,H,x_1)$ is defined by an (unobserved) \emph{latent state space}, which consists of an \emph{endogenous} state $s_h\in\cS$ and \emph{exogenous} state $\xi_h\in\Xi$, and an \emph{observation process} which generates the observed state $x_h$. We first describe the dynamics for the latent space. Given initial endogenous and exogenous states $s_1\in\cS$ and $\xi_1\in\Xi$, the latent states evolve via
\[
  s_{h+1}\sim{}\Pendo_h(s_h,a_h),\mathand\xi_{h+1}\sim\Pexo_h(\xi_h);
\]
that is while both states evolve in a temporally correlated fashion, only the endogenous state $s_h$ evolves as a function of the agent's action. The latent state $(s_h,\xi_h)$ is not observed. Instead, we observe
\[
x_h\sim{}q_h(s_h,\xi_h),
\]
where $q_h:\cS\times\Xi\to\Delta(\cX)$ is an \emph{emission distribution} with the property that $\supp(q_h(s,\xi))\cap\supp(q_h(s',\xi'))=\emptyset$ if $(s,\xi)\neq(s',\xi')$. This property (\emph{decodability}) ensures that there exists a unique mapping $\phistar_h:\cX\to\cS$ that maps the observed state $x_h$ to the corresponding endogenous latent state $s_h$. We assume that $R_h(x,a)=R_h(\phistar_h(x),a)$, which implies that optimal policy $\pistar$ depends only on the endogenous latent state, i.e. $\pistar_h(x)=\pistar_h(\phistar_h(x))$.

The main challenge of learning in block MDPs is that the decoder $\phistar$ is not known to the learner in advance. Indeed, given access to the decoder, one can obtain regret $\poly(H,\abs{\cS},\abs{\cA})\cdot\sqrt{T}$ by applying tabular reinforcement learning algorithms to the latent state space. In light of this, the aim of the \exmdp setting is to obtain sample complexity guarantees that are independent of the size of the observed state space $\abs{\cX}$ and exogenous state space $\abs{\Xi}$, and scale as $\poly(\abs{\cS}, \abs{\cA}, H,\log\abs{\cF})$, where $\cF$ is an appropriate class of function approximators (typically either a value function class $\cF$ or a class of decoders $\Phi$ that attempts to model $\phistar$ directly).

\exmdps present substantial additional difficulties compared to classical block MDPs because we aim to avoid dependence on the size $\abs{\Xi}$ of the exogenous latent state space. Here, the main challenge is that executing policies $\pi$ whose actions depend on $\xi_h$ can lead to spurious correlations between endogenous exogenous states. In spite of this apparent difficulty, we show that the \coverability coefficient for this setting is always bounded by the number of \emph{endogenous states}.
\begin{proposition}
  \label{prop:ex_bmdp}
  For any \exmdp, $\Ccov\leq\abs{\cS}\cdot\abs{\cA}$.
\end{proposition}
This bound is a consequence of a structural result from \cite{efroni2021provably}, which shows that for any $(s,a)\in\cS\times\cA$, all $x\in\cX$ with $\phistar(x)=s$ admit a common policy that maximizes $d_h^{\pi}(x,a)$, and this policy is \emph{endogenous}, i.e., only depends on the endogenous state $s_h=\phistar_h(x_h)$. As a corollary, we obtain the following regret bound.
\begin{corollary}
  \label{cor:ex_bmdp}
  For the \exmdp setting, under \pref{asm:completeness}, \pref{alg:golf} ensures that with probability at least $1-\delta$,
  \[
    \Reg\leq \bigoh\prn[\big]{
    H\sqrt{\abs{\cS}\abs{\cA}T \log(\nicefrac{TH |\Fcal|}{\delta})\log(T)}
    }.
  \]
\end{corollary}
Critically, this result scales only with the cardinality $\abs{\cS}$ for the endogenous latent state space, and with the capacity $\log\abs{\cF}$ for the value function class.

\pref{cor:ex_bmdp} is the first result for this setting that allows for stochastic latent dynamics and emission process, albeit with the extra assumption of completeness. Existing algorithms either require that the endogenous latent dynamics $\Pendo$ are deterministic \citep{efroni2021provably} or allow for stochastic dynamics but heavily restrict the observation process \citep{efroni2022sample}; complexity measures such as Bellman Rank and \betext can be arbitrarily large (see discussion in \pref{sec:gen_bedim}). Our result is best thought of as a ``luckiness'' guarantee, in the sense that it is unclear how to construct a value function class that is complete for every problem instance,\footnote{For example, it is not clear how to construct a complete value function class given access to a class of decoders $\Phi$ that contains $\phistar$.} but the algorithm will succeed whenever $\cF$ does happen to be complete for a given instance. Understanding whether general \exmdps are learnable without completeness is an interesting question for future work, and we are hopeful that the perspective of \coverability will lead to further insights for this setting.

\paragraph{Invariance of \coverability}%
\pref{prop:ex_bmdp} is a consequence of two general \emph{invariance} properties of \coverability, which show that $\Ccov$ is unaffected by the following augmentations to the underlying MDP: (i) addition of rich observations, and (ii) addition of exogenous noise.

The first property shows that for a given MDP $M$, creating a new block MDP $M'$ by equipping $M$ with a decodable emission process (so that $M$ acts as a \emph{latent MDP}), does not increase \coverability. %
\begin{proposition}[Invariance to rich observations]
  \label{prop:rich_obs}
  Let an MDP $M=(\cS,\cA,P,R,H,s_1)$. Let $M'=(\cX,\cA,P', R',H, x_1)$ be the MDP defined implicitly by the following process. For each $h\in\brk*{H}$:
  \begin{itemize}
  \item $s_{h+1}\sim{}P_h(s_h,a_h)$ and $r_h=R_h(s_h,a_h)$. Here, $s_h$ is unobserved, and may be thought of as a latent state.
  \item $x_{h}\sim{}q_h(s_h)$, where $q_h:\cS\to\Delta(\cX)$ is an \emph{emission distribution} with the property that $\supp(q_h(s))\cap\supp(q_h(s'))=\emptyset$ for $s\neq{}s'$.
  \end{itemize}
  Then, writing $\Ccov(M)$ to make the dependence on $M$ explicit, we have
  \[
    \Ccov(M') \leq \Ccov(M).
  \]
\end{proposition}
The second result shows that \coverability is also preserved if we expand the state space to include temporally correlated exogenous state whose evolution does not depend on the agent's actions. %
\begin{proposition}[Invariance to exogenous noise]
  \label{prop:exo}
  Let an MDP $M=(\cS,\cA,P,R,H,s_1)$, conditional distribution $\Pexo:\Xi\to\Delta(\Xi)$, and $\xi_1\in\Xi$ be given, where $\Xi$ is an abstract set. Let $\cX\ldef{}\cS\times\Xi$, and let $M'=(\cX,\cA,P', R',H, x_1)$ be the MDP with state $x_h=(s_h,\xi_h)$ defined implicitly by the following process. For each $h\in\brk{H}$:
  \begin{itemize}
  \item $s_{h+1}\sim{}P_{h}(s_h,a_h)$, $r_h=R_h(s_h,a_h)$.
  \item $\xi_{h+1}\sim{}\Pexo_h(\xi_h)$.
  \end{itemize}
  Then we have
  \[
    \Ccov(M') \leq \Ccov(M).
  \]
\end{proposition}
This result is non-trivial because policies that act based on the endogenous state $s_h$ and $\xi_h$ can cause these processes to become coupled \citep{efroni2021provably}, but holds nonetheless.

\pref{prop:ex_bmdp} can be deduced by combining \pref{prop:rich_obs,prop:exo} with the observation that any tabular (finite-state/action) MDP with $S$ states and $A$ actions has $\Ccov\leq{}SA$. However, \pref{prop:rich_obs,prop:exo} yield more general results, since they imply that starting with any (potentially non-tabular) class of MDPs $\cM$ with low \coverability and augmenting it with rich observations and exogenous noise preserves coverability.
}

\section{Are Weaker Notions of Coverage Sufficient?}
\label{sec:C_gen}

In \cref{sec:basic_cover}, we showed that existence of a distribution
with good concentrability (\coverability) is sufficient for sample-efficient
online RL. However, while concentrability is the most ubiquitous coverage
condition in offline RL, there are several weaker notions of
coverage which also lead to sample-efficient offline RL algorithms. In
this section, we show that analogues of \coverability based on these
conditions, \emph{single-policy concentrability} and \emph{generalized
concentrability} for Bellman residuals, do not suffice for
sample-efficient online RL. This indicates that in general, the interplay between
offline coverage and online exploration is nuanced. \loose

\paragraph{Single-policy concentrability}
\emph{Single-policy concentrability} is a widely used coverage
assumption in offline RL which weakens concentrability by requiring
only that the state distribution induced by $\pistar$ is covered by
the offline data distribution $\mu$, as opposed to requiring coverage for all policies \citep{jin2021pessimism,rashidinejad2021bridging}.
\begin{definition}[Single-policy concentrability]
The single-policy concentrability coefficient for a
data distribution $\mu=\crl*{\mu_h}_{h=1}^{H}$ is given by
\arxiv{
\begin{align*}
  \Cconc^\star(\mu) \ldef{} \nrm*{\frac{d_h^{\pi^\star}}{\mu_h}}_{\infty}.%
\end{align*}
}
\iclr{$\Cconc^\star(\mu) \ldef{} \nrm*{{d_h^{\pi^\star}}/{\mu_h}}_{\infty}$.}
\end{definition}
For offline RL, algorithms based on pessimism provide sample guarantee
complexity guarantees that scale with $\Cconc^{\star}(\mu)$ \citep{jin2021pessimism,rashidinejad2021bridging}. However, for
the online setting, it is trivial to show that an analogous notion of
``single-policy coverability'' (i.e., existence of a distribution with
good single-policy coverability) is not sufficient
for sample-efficient learning, since for any MDP, one can take $\mu
=d^{\pi^\star}$ to attain $\Cconc^\star(\mu) = 1$. This suggests that any
notion of coverage that suffices for online RL must be more
uniform in nature.\loose

\paragraph{Generalized concentrability for Bellman residuals}

Another approach to weaker coverage in offline RL
is to relaxed concentrability by only requiring
coverage with respect to the Bellman residuals for value functions in
$\cF$
\citep{chen2019information,xie2021bellman,cheng2022adversarially};
the following definition adapts this notion to the finite-horizon setting.\loose
\begin{definition}[Generalized concentrability]
\label{def:C_gen_offline}
We define the generalized concentrability coefficient
$\Cconcgen(\mu,\cF)$ for a policy
class $\Pi$ and value function class $\cF$ as the least
constant $C> 0$ such that the offline data distribution $\mu=\crl*{\mu_h}_{h=1}^{H}$ satisfies that for all $f \in \Fcal$ and $\pi \in \Pi$,\iclr{ $\sum_{h \in [H]} \E_{d_h^\pi}\big[ ( f_h(s_h,a_h) - (\Tcal_h f_{h+1})(s_h,a_h))^2 \big] \leq C \cdot \sum_{h \in [H]}\E_{\mu_h}\big[ \big( f_h(s_h,a_h) - (\Tcal_h f_{h+1})(s_h,a_h)\big)^2 \big]$.}
\arxiv{
\begin{align*}
\sum_{h=1}^{H} \E_{d_h^\pi}\left[ \left( f_h(s_h,a_h) - (\Tcal_h f_{h+1})(s_h,a_h)\right)^2 \right] \leq C \cdot \sum_{h=1}^{H} \E_{\mu_h}\left[ \left( f_h(s_h,a_h) - (\Tcal_h f_{h+1})(s_h,a_h)\right)^2 \right].
\end{align*}
}
\end{definition}
Note that $\Cgen_\off(\mu,\cF)\leq\Cconc(\mu)$ (in particular, they coincide if one
chooses $\cF$ to be the set of all functions over $\cX\times\cA$) but in general $\Cgen_\off(\mu,\cF)$ can be
much smaller. For example, in the linear Bellman-complete setting, it is possible to
bound $\Cgen_\off(\mu,\cF)$ in terms of feature coverage conditions
\citep{wang2021statistical,zanette2021provable}. Using offline data
from $\mu$, sample complexity guarantees
that scale with $\Cgen_\off(\mu,\cF)$ can be obtained under \cref{asm:completeness} via MSBO~\citep[see,
e.g.,][Section 5]{xie2020q} or by running a ``one-step'' variant of \golf
(\cref{alg:golf}); we provide this result
(\pref{prop:generalized_offline}) in \cref{app:lower} for completeness.
Given that this notion leads to positive results for
offline RL, it is natural to consider a generalized notion of
coverability based upon it. %
\begin{definition}[Generalized coverability]
\label{def:C_gen_online}
We define the generalized \coverability coefficient for a policy class
$\Pi$ value function class $\cF$ and as \[\Ccovgen(\cF) = \inf_{\mu_1,\ldots,\mu_H\in\Delta(\cX\times\cA)}\crl{\Cconcgen(\mu,\cF)}.\]
\end{definition}
Unfortunately, we show that this condition does not suffice for
sample-efficient online RL, even when the number of actions is
constant and \cref{asm:completeness} is satisfied.
\begin{theorem}
  \label{thm:generalized_lower}
  For any $X,H,C\in\bbN$, there exists a family of MDPs with
  $\abs{\cX}=X$, $|\Acal| = 2$ and horizon $H$ and a function class $\Fcal$ with
  $\log\abs{\cF}\leq{}H\log(2\abs{\cX})$ such that: i) \cref{asm:completeness} (completeness) is satisfied for
  $\Fcal$ and we have $\Cgen_\on(\cF)\leq{}C$
  and ii) Any online RL algorithm that
    returns a $0.1$-optimal policy with probability $0.9$
requires at least
        \[
        \Omega\left( \min\left\{ X, 2^{\bigom(H)}, 2^{\bigom(C)} \right\} \right)
        \]
    trajectories. 
\end{theorem}
\cref{thm:generalized_lower} highlights that in general, notions of
coverage that suffice for offline RL---even those that are uniform in nature---can fail to lead
to useful structural conditions for online RL. Briefly, the issue is
that bounding regret for online
    RL entails controlling the extent to which a deliberate algorithm
    that has observed state distributions \arxiv{$d_h^{\pi\ind{1}},\ldots,d_h^{\pi\ind{t-1}}$}\iclr{$d_h\ind{1},\ldots,d_h\ind{t-1}$}
    can be ``surprised'' by a substantially new state distribution
    \arxiv{$d_h^{\pi\ind{t}}$}\iclr{$d_h\ind{t}$}; here, surprise is typically measure in terms
    of Bellman residual. The proof of \cref{thm:generalized_lower}
    shows that existence of a distribution with good coverage with respect
    to Bellman residuals does suffice to provide meaningful control of
    distribution shift. We caution, however, that the lower bound
    construction makes use of the fact that \pref{def:C_gen_online} requires
    coverage only \emph{on average} across layers, and it is unclear
    whether a similar lower bound holds under uniform coverage across
    layers. Developing a more unified and fine-grained understanding of what coverage
    conditions lead to efficient exploration is an important question for future research.

\section{A New Structural Condition for Sample-Efficient Online RL}
\label{sec:gen_bedim}

Having shown that \coverability \arxiv{serves a structural condition that}
facilitates sample-efficient online \arxiv{reinforcement learning}\iclr{RL}, an
immediate question is whether this structural condition is related to existing
complexity measures such as Bellman-Eluder
dimension~\citep{jin2021bellman} and Bellman/Bilinear
rank~\citep{jiang2017contextual,du2021bilinear}, which attempt to
unify existing approaches to sample-efficient RL. We now show that
these complexity measures are insufficient to capture \coverability,
then provide a new complexity measure, the \CompMeasure, which bridges
the gap.\loose

\subsection{Insufficiency of Existing Complexity Measures}
Bellman-Eluder
dimension~\citep{jin2021bellman} and Bellman/Bilinear
rank~\citep{jiang2017contextual,du2021bilinear} can fail to capture
\coverability for two reasons: (i) insufficiency of average Bellman error
(as opposed to squared Bellman error), and (ii) incorrect dependence
on scale. To
highlight these issues, we focus on \emph{$Q$-type} Bellman-Eluder dimension
\citep{jin2021bellman}, which subsumes Bellman rank.\footnote{$Q$-type and $V$-type
  are similar, but define the Bellman residual with respect to
  different action distributions.} See \pref{app:gen_bedim} for
discussion of other complexity measures\arxiv{, including Bilinear rank}.\loose

Let $\Dset_h^\Pi \coloneqq \{d^\pi_h: \pi \in \Pi\}$ and $\Fcal_h - \Tcal_h \Fcal_{h+1} \coloneqq \{f_h - \Tcal_h f_{h+1}: f \in \Fcal\}$.
Following~\citet{jin2021bellman}, we define the ($Q$-type) \betext as
follows.
\begin{definition}[\betext]
  \label{def:be_dim}
  The \betext $\bedim(\Fcal, \Pi, \varepsilon,h)$ for the layer $h$ is the
  largest $d\in\bbN$, such that there exist sequences
  $\{d_h^\iter{1},d_h^\iter{2},\dotsc,d_h^\iter{d}\} \subseteq
  \Dset^\Pi_h$ and
  $\{\delta_h\ind{1},\ldots,\delta_h^\iter{d}\} \subseteq \Fcal_h - \Tcal_h
  \Fcal_{h+1}$ such that for all $t\in\brk{d}$,\iclr{ $\abs{\E_{d_h^\iter{t}}[\delta_h^\iter{t}]} >
    \varepsilon^\iter{t},\mathand \sqrt{\sum_{i = 1}^{t - 1} \prn[\big]{\E_{d_h^\iter{i}}[\delta_h^\iter{t}]}^2} \leq \varepsilon^\iter{t}$},
  \arxiv{
  \begin{align}
    \label{eq:bedim}
    \abs{\E_{d_h^\iter{t}}[\delta_h^\iter{t}]} >
    \varepsilon^\iter{t},\mathand \sqrt{\sum_{i = 1}^{t - 1} \prn[\big]{\E_{d_h^\iter{i}}[\delta_h^\iter{t}]}^2} \leq \varepsilon^\iter{t},
  \end{align}
  }
  for $\varepsilon\ind{1},\ldots,\veps\ind{d} \geq \varepsilon$. We define $\bedim(\cF,\Pi,\veps)=\max_{h\in\brk{H}}\bedim(\cF,\Pi,\veps,h)$.
\end{definition}

\paragraph{Issue {\sf\#}1: Insufficiency of average (vs.~squared) Bellman error}
The Bellman-Eluder dimension reflects the length of the
longest consecutive sequence of value function pairs for which
we can be ``surprised'' by a large Bellman residual for a new policy
if the value function has low Bellman residual on all preceding
policies. Note that via \arxiv{\pref{eq:bedim}}\iclr{\pref{def:be_dim}}, the \betext measures the size
of the surprise and the error on preceding points via \emph{average}
Bellman error (e.g., $\E_{d_h^\iter{i}}[\delta_h^\iter{t}]$). On the other hand, the proof of
\pref{thm:golf_guarantee_basic} critically uses \emph{squared} Bellman
error $\E_{d_h^\iter{i}}[(\delta_h^\iter{t})^2]$ bound regret by
\coverability; this is because the (point-wise) nonnegativity of
squared Bellman error facilitates change-of-measure in a similar
fashion to offline reinforcement learning. The following result shows
that this issue is fundamental, and Bellman-Eluder dimension can be exponential large relative to the regret bound in
\pref{thm:golf_guarantee_basic}.\loose

\begin{proposition}
  \label{prop:average_bellman_lower_bound}
  For any $d\in\bbN$, there exists an MDP $M$ with $H=2$ and $\abs{\cA}=2$, policy class $\Pi$ with
  $\abs{\Pi}=d$, and
  value function class $\Fcal$ with $\abs{\cF}=d$ satisfying completeness, such that $C_\on =
  \Ocal(1)$, but the \betext has $\bedim(\Fcal, \Pi,
  \varepsilon) = \Omega(\min\{ |\Fcal|, |\Pi| \})=\bigom(d)$ for any $\veps\leq 1/2$.
\end{proposition}

The lower bound in \pref{prop:average_bellman_lower_bound} is
realized by an exogenous block MDP (\arxiv{\pref{sec:exbmdp}}\iclr{\pref{app:exbmdp}}), with $d$
representing the number of \emph{exogenous} states. The result
gives an exponential separation between what can be achieved using \betext and
\coverability, because \golf attains 
$\Reg\leq\bigoht\big(\sqrt{T\log(d)}\big)$ (cf. \pref{cor:ex_bmdp}), yet we have
$\bedim(\cF,\Pi,1/2)=\bigom(d)$. This exponential separation can
  also be shown to apply to  \emph{algorithms} based on average
  Bellman error: \cref{prop:lb_olive} (\pref{app:gen_bedim}) shows that
  \olive~\citep{jiang2017contextual} requires $\Omega(d)$ trajectories
  to obtain a near-optimal policy.
The construction, which is based on
\citet[Section B.1]{efroni2022sparsity}, critically leverages cancellations in
the average Bellman error; these cancellations are ruled out by
squared Bellman error, which is why \pref{thm:golf_guarantee_basic}
gives a regret bound that scales only \emph{logarithmically} in
$d$. Bilinear rank \citep{du2021bilinear} and $V$-type Bellman rank suffer from similar
  drawbacks; see \pref{app:gen_bedim} for further
  discussion.

\paragraph{Issue {\sf\#}2: Incorrect dependence on scale}
\iclr{
In light of the previous example, a seemingly reasonable fix is to adapt the
\betext to consider squared Bellman error rather than average Bellman
error (i.e., use $\sqrt{\sum_{i = 1}^{t - 1}
  \prn[\big]{\E_{d_h^\iter{i}}[(\delta_h^\iter{t})^2]}} \leq
\varepsilon^\iter{t}$ in \arxiv{\pref{eq:bedim}}\iclr{\pref{def:be_dim}}). We show (\pref{app:sndim_bedim}) that while it is possible to
bound this modified \betext in terms of the \coverability parameter,
the dependence on the scale parameter $\veps$ is \emph{polynomial},
and it is not possible to derive regret bounds better than $T^{2/3}$ under
\coverability with this approach. Informally, the issue is \emph{scale}: Bellman-eluder dimension only checks whether the average Bellman
error violates the threshold $\veps$, and does not consider how far
the error violates the threshold (e.g.,
$\abs{\E_{d_h^\iter{t}}[\delta_h^\iter{t}]}>\veps$ and
$\abs{\E_{d_h^\iter{t}}[\delta_h^\iter{t}]}>1$ are counted the
same).\loose

}

\arxiv{
  In light of the previous example, a seemingly reasonable fix is to adapt the
\betext to consider squared Bellman error rather than average Bellman
error. Consider the following variant.
\begin{definition}[\betextsq]
\label{def:be_dim_sq}
  We define the \betextsq $\bedimsq(\Fcal, \Pi, \varepsilon,h)$ for layer $h$ is the
  largest $d\in\bbN$ such that there exist sequences
  $\{d_h^\iter{1},d_h^\iter{2},\dotsc,d_h^\iter{d}\} \subseteq
  \Dset^\Pi_h$ and
  $\{\delta_h\ind{1},\ldots,\delta_h^\iter{d}\} \subseteq \Fcal_h - \Tcal_h
  \Fcal_{h+1}$ such that for all $t\in\brk{d}$,
  \begin{align}
    \label{eq:bedimsq}
    \abs{\E_{d_h^\iter{t}}[\delta_h^\iter{t}]} >
    \varepsilon^\iter{t},\mathand \sqrt{\sum_{i = 1}^{t - 1} \E_{d_h^\iter{i}}[(\delta_h^\iter{t})^2]} \leq \varepsilon^\iter{t},
  \end{align}
  for $\varepsilon\ind{1},\ldots,\veps\ind{d} \geq \varepsilon$. We define $\bedimsq(\cF,\Pi,\veps)=\max_{h\in\brk{H}}\bedimsq(\cF,\Pi,\veps,h)$.
\end{definition}
This definition is identical to \pref{def:be_dim}, except that the
constraint $\sqrt{\sum_{i = 1}^{t - 1}
  \prn{\E_{d_h^\iter{i}}[\delta_h^\iter{t}]}^2} \leq
\varepsilon^\iter{t}$ in \pref{eq:bedim} has been replaced by the
constraint $\sqrt{\sum_{i = 1}^{t - 1}
  \E_{d_h^\iter{i}}[(\delta_h^\iter{t})^2]} \leq
\varepsilon^\iter{t}$, which uses squared Bellman error instead of
average Bellman error. By adapting the analysis of
\citet{jin2021bellman} it is possible to show that this definition
yields
$\Reg\leq\bigoht\prn[\big]{H\sqrt{\inf_{\veps>0}\crl{\veps^2T+\bedimsq(\cF,\Pi,\veps)}\cdot{}T\log\abs{\cF}}}$. If
one could show that
$\dimbesq(\cF,\Pi,\veps)\approxleq{}\Ccov\cdot\polylog(\veps^{-1})$, 
this would recover \pref{thm:golf_guarantee_basic}. Unfortunately, it
turns out that in general, one can have
$\dimbesq(\cF,\Pi,\veps)=\bigom(\Ccov/\veps)$, which leads to
suboptimal $T^{2/3}$-type regret using the result above. The following
result shows that this guarantee cannot be improved without changing
the complexity measure under consideration.
\begin{proposition}
\label{prop:bedim_lower_bound}
Fix $T\in\bbN$, and let $\veps_T\ldef{}T^{-1/3}$. There exist MDP class/policy class/value function class tuples
$(\cM_1,\Pi_1,\cF_1)$ and $(\cM_2,\Pi_2,\cF_2)$ with the following
properties.
\begin{enumerate}
\item All MDPs in $\cM_1$ (resp. $\cM_2$) satisfy
  \pref{asm:completeness} with respect to $\cF_1$ (resp. $\cF_2$). In
  addition, $\log\abs{\cF_1}=\log\abs{\cF_2}=\bigoht(1)$.
\item For all MDPs in $\cM_1$, we have
  $\dimbesq(\cF_1,\Pi_1,\veps_T)\propto{}1/\veps_T$, and any algorithm
  must have $\En\brk*{\Reg}\geq{}\bigom(T^{2/3})$ for some MDP in the class
\item For all MDPs in $\cM_2$, we also have
  $\dimbesq(\cF_2,\Pi_2,\veps_T)\propto{}1/\veps_T$, yet
  $\Ccov=\bigoh(1)$ and \golf attains $\En\brk*{\Reg}\leq\bigoht(\sqrt{T})$.
\end{enumerate}
\end{proposition}
This result shows that there are two classes for which the optimal
rate differs polynomially ($\bigom(T^{2/3})$ vs. $\bigoht(\sqrt{T})$),
yet the \betext has the same size, and implies that the \betext
cannot provide rates better
than $\bigom(T^{2/3})$ for classes with low coverability in general. Informally, the reason why \betext fails capture the optimal rates for
the problem instances in \pref{prop:bedim_lower_bound} is that the
definition in \pref{eq:bedimsq} only checks whether the average Bellman
error violates the threshold $\veps$, and does not consider how far
the error violates the threshold ($\abs{\E_{d_h^\iter{t}}[\delta_h^\iter{t}]}>\veps$ and
$\abs{\E_{d_h^\iter{t}}[\delta_h^\iter{t}]}>1$ are counted the same).

}

\subsection{The \CompMeasure}
\label{sec:snc_main}

To address the issues above, we introduce a new complexity measure,
the \CompMeasure ($\SNC$), which i) leads to regret bounds via \golf and ii)
subsumes both \coverability and the \betext. Conceptually, the
\CompMeasure should be thought of as a minimal abstraction of the main
ingredient in regret bounds based on \golf and other optimistic
algorithms: extrapolation from in-sample error to on-policy error.
We begin by stating a variant of the \CompMeasure for abstract
function classes, then specialize it to reinforcement learning.
\begin{definition}[\CompMeasure]
  \label{def:online_c}
  Let $\cZ$ be an abstract set. Given a \emph{test function class}
  $\Psi\subset(\cZ\to\bbR)$ and \emph{distribution class}
  $\Dset\subset\Delta(\cZ)$, the \compmeasure for length $T$ is given by
\iclr{
\begin{align*}
\cdim(\Psi,\Dset,T) \ldef{} \sup_{\psi\ind{1},\ldots,\psi\ind{T} \in \Psi} ~ \sup_{d\ind{1},\ldots,d\ind{T} \in \Dset}\Bigg\{\sum_{t \in [T]} \frac{\E_{d\ind{t}}[\psi\ind{t}]^2 }{1 \vee \sum_{i = 1}^{t - 1} \E_{d\ind{i}}[(\psi\ind{t})^2]}\Bigg\}.
\end{align*}
}
\arxiv{
\begin{align*}
\cdim(\Psi,\Dset,T) \ldef{} \sup_{\{\psi\ind{1},\ldots,\psi\ind{T}\} \subseteq \Psi} \sup_{\{d\ind{1},\ldots,d\ind{T}\} \subseteq \Dset}\crl*{\sum_{t = 1}^{T} \frac{\E_{d\ind{t}}[\psi\ind{t}]^2 }{1 \vee \sum_{i = 1}^{t - 1} \E_{d\ind{i}}[(\psi\ind{t})^2]}}.
\end{align*}
}
\end{definition}
To apply the \CompMeasure to RL, we use Bellman residuals for $\cF$ as test
functions and consider state-action distributions induced by policies
in $\Pi$.
\begin{definition}[$\SNC$ for RL]
  \label{def:snc_qtype}
We define \iclr{$\cdimrl(\Fcal,\Pi,T) \coloneqq \max_{h \in [H]} \cdim(\Fcal_h - \Tcal_h \Fcal_{h+1},\Dset_h^\Pi,T)$.}

\arxiv{
For each $h \in [H]$, let $\Dset_h^\Pi \coloneqq \{d^\pi_h: \pi \in
\Pi\}$ and $\Fcal_h - \Tcal_h \Fcal_{h+1} \coloneqq \{f_h - \Tcal_h
f_{h+1}: f \in \Fcal\}$. We define \iclr{$\cdimrl(\Fcal,\Pi,T) \coloneqq \max_{h \in [H]} \cdim(\Fcal_h - \Tcal_h \Fcal_{h+1},\Dset_h^\Pi,T)$.}
\begin{align*}
\cdimrl(\Fcal,\Pi,T) \coloneqq \max_{h \in [H]} \cdim(\Fcal_h - \Tcal_h \Fcal_{h+1},\Dset_h^\Pi,T).
\end{align*}}
\end{definition}
The following result, which is a near-immediate consequence of the
definition, shows that the \CompMeasure leads to regret bounds via
\golf; recall that $\Pi=\crl{\pi_f\mid{}f\in\cF}$ is the set of greedy
policies induced by $\cF$.
\begin{theorem}
\label{thm:regret_online_c}
Under \cref{asm:completeness}, there exists an absolute constant $c$
such that for any $\delta \in (0,1]$ and $T \in \NN_+$, if we choose
$\beta = c \cdot \log(\nicefrac{TH |\Fcal|}{\delta})$ in
\cref{alg:golf}, then with probability at least $1 - \delta$, we have \iclr{$\Reg \leq
         O\big( H \sqrt{\cdimrl(\Fcal,\Pi,T)\cdot{}T\cdot\log(\nicefrac{TH |\Fcal|}{\delta})} \big)$.}
\arxiv{\begin{align*}
  \Reg &\leq
         O\left( H \sqrt{\cdimrl(\Fcal,\Pi,T)\cdot{}T\cdot\log(\nicefrac{TH |\Fcal|}{\delta})} \right).%
       \end{align*}
       }
\end{theorem}
We defer the proof of \cref{thm:regret_online_c} to
\cref{app:gen_bedim}, and conclude by showing that the \CompMeasure subsumes coverability\iclr{ $C_\on$ (\cref{def:low_concentrability})} and \betext.
\begin{proposition}[Coverability $\Longrightarrow$ $\SNC$]
  \label{prop:sedim_coverability}
  \arxiv{
  Let $C_\on$ be the coverability coefficient
  \arxiv{(\cref{def:low_concentrability}) }for policy class $\Pi$. Then for
  any value function class $\cF$, $\cdimrl(\Fcal,\Pi,T) \leq O \left( C_\on\cdot\log(T) \right)$.}
  \iclr{$\cdimrl(\Fcal,\Pi,T) \leq O \left( C_\on\cdot\log(T) \right)$.}
\end{proposition}

\begin{proposition}[\arxiv{\betext}\iclr{\betextshort} $\Longrightarrow$ $\SNC$]
  \label{prop:sedim_bedim}
\arxiv{Suppose $\bedimv(\Fcal,\Pi,\veps)$ be \betext (\cref{def:be_dim}) with function class $\Fcal$ and policy $\Pi$, then
\begin{align*}
     \cdimrl(\Fcal,\Pi,T) \leq O\left(\inf_{\veps>0}\crl*{\veps^2T + \bedim(\Fcal,\Pi,\veps)}\cdot\log(T)\right).
 \end{align*}}
\iclr{$\cdimrl(\Fcal,\Pi,T) \leq O(\bedim(\Fcal,\Pi,\sqrt{\nicefrac{1}{T}})\cdot\log(T))$.}
\end{proposition}
\arxiv{Note that since Bellman rank upper bounds the \betext, this shows that
$\cdimrl(\cF,\Pi,T)\leq\bigoht(d)$ whenever the $Q$-type Bellman rank is
$d$.

The \CompMeasure can likely be generalized
in many directions (e.g., by allowing for different test
functions in the vein of \citet{du2021bilinear}). This is beyond the
scope of the present paper, but further unifying these notions is an
interesting question for future research; see
\pref{app:v_type,app:bilinear} for further discussion.
}

\iclr{The \CompMeasure can likely be generalized
further along many directions (e.g., by allowing for different test
functions in the vein of \citet{du2021bilinear}). Further unifying these notions is an
interesting question for future research; see
\pref{app:complexity_additional} for further discussion.
}

\arxiv{

\section{Discussion}
\label{sec:discussion}

\arxiv{Our results initiate the systematic study of connections between online and
offline learnability for RL, and highlight deep connections between coverage
in offline RL and exploration in online RL. In what follows we discuss
additional related work, and close with some future directions.}

\iclr{Toward a general theory beyond this paper, we highlight some
exciting and challenging open problems for future research.
}

\arxiv{
\subsection{Related Work}
\label{sec:related}

\arxiv{Let us briefly highlight some relevant related work not already
  covered.}
\iclr{In this section we briefly highlight some relevant related work not otherwise discussed.}

\paragraph{Online RL with access to offline data}
A separate line of work develops algorithms for online
reinforcement learning that assume additional access to offline data
gathered with a known data distribution $\mu$ or known exploratory
policy \citep{abbasi2019exploration,xie2021policy}. These results are complementary
to our own, since we assume only that a good exploratory distribution
exists, but do not assume that such a distribution is known to the learner.

\paragraph{Further structural conditions for online RL}
While we have already discussed connections to Bellman Rank, Bilinear
Classes, and Bellman-Eluder Dimension, another more general complexity measure
is the \emph{Decision-Estimation Coefficient}
\citep{foster2021statistical}. One can show that the \CompText is
bounded by \coverability, but to apply the algorithm in
\citet{foster2021statistical}, one must assume access to a
realizable \emph{model class} $\cM$, which leads to regret bounds
that scale with $\log\abs{\cM}$ rather than
$\log\abs{\cF}$.

\paragraph{Instance-dependent algorithms}
\citet{wagenmaker2022beyond} provide instance-dependent guarantees for
tabular PAC-RL which scale with a quantity called \emph{gap-visitation
  complexity}. It is possible to bound the gap-visitation complexity in
terms of \coverability, but the lower-order sample complexity terms in
this result have explicit dependence on the number of states, which
our results avoid. For future work, it would be interesting to
understand deeper connections between \coverability and
instance-dependent complexity measures
\citep{wagenmaker2022beyond,wagenmaker2022instance,dong2022asymptotic}. See also \citet{wagenmaker2022instance},
which provides similar guarantees for linear MDPs.

 }

\arxiv{\subsection{Future Directions}
Toward building a general
theory that bridges offline and online RL, let us highlight some
exciting questions for future research.
}
\begin{itemize}
\item \emph{Weaker notions of coverage.} Our results in \pref{sec:C_gen}
  show that the generalized coverability condition
  (\pref{def:C_gen_offline}), which exploits the structure of the
  value function class $\cF$, is not sufficient for online
  exploration. For the special case of linear functions ($\cF=\crl*{(x,a)\mapsto\tri*{\phi(x,a),\theta}\mid\theta\in\Theta\subset\bbR^{d}}$) a natural
  strengthening of this condition
  \citep{wang2021statistical,zanette2021provable} is to assert the
  existence of a data distribution $\mu=\crl*{\mu_h}_{h=1}^{H}$ such that
  $\En_{d_h^{\pi}}\brk*{\phi(x_h,a_h)\phi(x_h,a_h)^{\trn}}\psdleq{}C\cdot\En_{\mu_h}\brk*{\phi(x_h,a_h)\phi(x_h,a_h)^{\trn}}$
  for some coverage parameter $C$. Is this condition (or a variant) sufficient for
  sample-efficient online exploration? More broadly, are there other natural
  ways to strengthen \pref{def:C_gen_offline} that lead to positive results?
\item \emph{Further conditions from offline RL.}
  There are many conditions used to provide
  sample-efficient learning guarantees in offline RL beyond those
  considered in this paper, including (i) pushforward concentrability
  \citep{munos2003error,xie2021batch}, (ii) $L_p$ variants of concentrability \citep{farahmand2010error,xie2020q}, and (iii) weight function\arxiv{/density ratio}
  realizability \citep{xie2020q,jiang2020minimax,zhan2022offline}. Which of these conditions can be adapted
  for online exploration, and to what extent?\loose
\end{itemize}
\arxiv{
Beyond these questions, it will be interesting to explore whether the
notion of coverability can guide the design of practical algorithms.
}
 }

\iclr{
\vspace{-6.5pt}
\section{Conclusion}
\vspace{-6.5pt}

This paper initiates the systematic study of parallels between online and offline learnability in reinforcement learning and uncovers surprising new connections. The possible future directions include general theories under weaker notions of \coverability or approximation conditions (see~\cref{sec:discussion} for open problems) as well as the connection to the practical algorithm design.
}

\arxiv{
\subsection*{Acknowledgements}
\arxiv{%
Nan Jiang acknowledges funding support from ARL Cooperative Agreement W911NF-17-2-0196, NSF IIS-2112471, NSF CAREER award, and Adobe Data Science Research Award. Sham Kakade acknowledges funding from the Office of Naval Research under award N00014-22-1-2377 and the National Science Foundation Grant under award {\sf\#}CCF-1703574. %
}
}

\bibliography{ref,dylan_refs}

\clearpage

\appendix
\onecolumn

\arxiv{{
\centering
\begin{center}
{\huge \textbf{Appendix}}
\end{center}
}}

\renewcommand{\contentsname}{}
\addtocontents{toc}{\protect\setcounter{tocdepth}{2}}
{\hypersetup{hidelinks}
\tableofcontents
}

\iclr{{
\centering
\part*{\hfil \LARGE\rm\sc Appendix \hfil}
}}

\iclr{
{
\hypersetup{hidelinks}
\tableofcontents
}
}

\allowdisplaybreaks

\iclr{
\section{Discussion and Open Problems}
\label{sec:discussion}

\arxiv{Our results initiate the systematic study of connections between online and
offline learnability for RL, and highlight deep connections between coverage
in offline RL and exploration in online RL. In what follows we discuss additional related work, and close with some open problems.}

\iclr{Toward a general theory beyond this paper, we highlight some
exciting and challenging open problems for future research.
}

\arxiv{
\subsection{Related Work}
\label{sec:related}

\arxiv{Let us briefly highlight some relevant related work not already
  covered.}
\iclr{In this section we briefly highlight some relevant related work not otherwise discussed.}

\paragraph{Online RL with access to offline data}
A separate line of work develops algorithms for online
reinforcement learning that assume additional access to offline data
gathered with a known data distribution $\mu$ or known exploratory
policy \citep{abbasi2019exploration,xie2021policy}. These results are complementary
to our own, since we assume only that a good exploratory distribution
exists, but do not assume that such a distribution is known to the learner.

\paragraph{Further structural conditions for online RL}
While we have already discussed connections to Bellman Rank, Bilinear
Classes, and Bellman-Eluder Dimension, another more general complexity measure
is the \emph{Decision-Estimation Coefficient}
\citep{foster2021statistical}. One can show that the \CompText is
bounded by \coverability, but to apply the algorithm in
\citet{foster2021statistical}, one must assume access to a
realizable \emph{model class} $\cM$, which leads to regret bounds
that scale with $\log\abs{\cM}$ rather than
$\log\abs{\cF}$.

\paragraph{Instance-dependent algorithms}
\citet{wagenmaker2022beyond} provide instance-dependent guarantees for
tabular PAC-RL which scale with a quantity called \emph{gap-visitation
  complexity}. It is possible to bound the gap-visitation complexity in
terms of \coverability, but the lower-order sample complexity terms in
this result have explicit dependence on the number of states, which
our results avoid. For future work, it would be interesting to
understand deeper connections between \coverability and
instance-dependent complexity measures
\citep{wagenmaker2022beyond,wagenmaker2022instance,dong2022asymptotic}. See also \citet{wagenmaker2022instance},
which provides similar guarantees for linear MDPs.

 }

\arxiv{\subsection{Open Problems}
Toward building a general
theory that bridges offline and online RL, let us highlight some
exciting and challenging open problems for future research.
}
\begin{itemize}
\item \emph{Are coverability and realizability sufficient?} Our results show that coverability, along with Bellman
  completeness, is sufficient for sample-efficient online
  reinforcement learning. Does the same result hold if one only
  assumes realizability ($\Qstar\in\cF$) rather than completeness, or
  is there a lower bound? This question is open even for the special
  case where $\Qstar$ is linear.
\item \emph{Linear feature coverage.} Our results in \pref{sec:C_gen}
  show that the generalized coverability condition
  (\pref{def:C_gen_offline}), which exploits the structure of the
  value function class $\cF$, is not sufficient for online
  exploration. For the special case of linear functions ($\cF=\crl*{(x,a)\mapsto\tri*{\phi(x,a),\theta}\mid\theta\in\Theta\subset\bbR^{d}}$) a natural
  strengthening of this condition
  \citep{wang2021statistical,zanette2021provable} is to assert the
  existence of a data distribution $\mu=\crl*{\mu_h}_{h=1}^{H}$ such that
  $\En_{d_h^{\pi}}\brk*{\phi(x_h,a_h)\phi(x_h,a_h)^{\trn}}\psdleq{}C\cdot\En_{\mu_h}\brk*{\phi(x_h,a_h)\phi(x_h,a_h)^{\trn}}$
  for some coverage parameter $C$. Is this condition (or a variant) sufficient for
  sample-efficient online exploration?
\item \emph{Further conditions from offline RL.}
  There are many conditions used to provide
  sample-efficient learning guarantees in offline RL beyond those
  considered in this paper, including (i) pushforward concentrability
  \citep{xie2021batch}, (ii) $L_p$ variants of concentrability \citep{xie2020q}, and (iii) weight function\arxiv{/density ratio}
  realizability \citep{xie2020q,jiang2020minimax,zhan2022offline}. Which of these conditions can be adapted
  for online exploration, and to what extent?\loose
\end{itemize}
\arxiv{
Beyond these questions, it will be interesting to explore whether the
notion of coverability can guide the design of practical algorithms.
}
 }

\iclr{
  \section{Additional Related Work}
\label{sec:related}

\arxiv{Let us briefly highlight some relevant related work not already
  covered.}
\iclr{In this section we briefly highlight some relevant related work not otherwise discussed.}

\paragraph{Online RL with access to offline data}
A separate line of work develops algorithms for online
reinforcement learning that assume additional access to offline data
gathered with a known data distribution $\mu$ or known exploratory
policy \citep{abbasi2019exploration,xie2021policy}. These results are complementary
to our own, since we assume only that a good exploratory distribution
exists, but do not assume that such a distribution is known to the learner.

\paragraph{Further structural conditions for online RL}
While we have already discussed connections to Bellman Rank, Bilinear
Classes, and Bellman-Eluder Dimension, another more general complexity measure
is the \emph{Decision-Estimation Coefficient}
\citep{foster2021statistical}. One can show that the \CompText is
bounded by \coverability, but to apply the algorithm in
\citet{foster2021statistical}, one must assume access to a
realizable \emph{model class} $\cM$, which leads to regret bounds
that scale with $\log\abs{\cM}$ rather than
$\log\abs{\cF}$.

\paragraph{Instance-dependent algorithms}
\citet{wagenmaker2022beyond} provide instance-dependent guarantees for
tabular PAC-RL which scale with a quantity called \emph{gap-visitation
  complexity}. It is possible to bound the gap-visitation complexity in
terms of \coverability, but the lower-order sample complexity terms in
this result have explicit dependence on the number of states, which
our results avoid. For future work, it would be interesting to
understand deeper connections between \coverability and
instance-dependent complexity measures
\citep{wagenmaker2022beyond,wagenmaker2022instance,dong2022asymptotic}. See also \citet{wagenmaker2022instance},
which provides similar guarantees for linear MDPs.

 }

\iclr{
  \section{Application to Exogenous Block MDPs}
  \label{app:exbmdp}
As an application of \pref{thm:golf_guarantee_basic}, we consider the problem of reinforcement learning in \emph{Exogenous Block MDPs} (\exmdps). Following \cite{efroni2021provably}, an \exbmdp $M=(\cX,\cA,P,R,H,x_1)$ is defined by an (unobserved) \emph{latent state space}, which consists of an \emph{endogenous} state $s_h\in\cS$ and \emph{exogenous} state $\xi_h\in\Xi$, and an \emph{observation process} which generates the observed state $x_h$. We first describe the dynamics for the latent space. Given initial endogenous and exogenous states $s_1\in\cS$ and $\xi_1\in\Xi$, the latent states evolve via
\[
  s_{h+1}\sim{}\Pendo_h(s_h,a_h),\mathand\xi_{h+1}\sim\Pexo_h(\xi_h);
\]
that is while both states evolve in a temporally correlated fashion, only the endogenous state $s_h$ evolves as a function of the agent's action. The latent state $(s_h,\xi_h)$ is not observed. Instead, we observe
\[
x_h\sim{}q_h(s_h,\xi_h),
\]
where $q_h:\cS\times\Xi\to\Delta(\cX)$ is an \emph{emission distribution} with the property that $\supp(q_h(s,\xi))\cap\supp(q_h(s',\xi'))=\emptyset$ if $(s,\xi)\neq(s',\xi')$. This property (\emph{decodability}) ensures that there exists a unique mapping $\phistar_h:\cX\to\cS$ that maps the observed state $x_h$ to the corresponding endogenous latent state $s_h$. We assume that $R_h(x,a)=R_h(\phistar_h(x),a)$, which implies that optimal policy $\pistar$ depends only on the endogenous latent state, i.e. $\pistar_h(x)=\pistar_h(\phistar_h(x))$.

The main challenge of learning in block MDPs is that the decoder $\phistar$ is not known to the learner in advance. Indeed, given access to the decoder, one can obtain regret $\poly(H,\abs{\cS},\abs{\cA})\cdot\sqrt{T}$ by applying tabular reinforcement learning algorithms to the latent state space. In light of this, the aim of the \exmdp setting is to obtain sample complexity guarantees that are independent of the size of the observed state space $\abs{\cX}$ and exogenous state space $\abs{\Xi}$, and scale as $\poly(\abs{\cS}, \abs{\cA}, H,\log\abs{\cF})$, where $\cF$ is an appropriate class of function approximators (typically either a value function class $\cF$ or a class of decoders $\Phi$ that attempts to model $\phistar$ directly).

\exmdps present substantial additional difficulties compared to classical block MDPs because we aim to avoid dependence on the size $\abs{\Xi}$ of the exogenous latent state space. Here, the main challenge is that executing policies $\pi$ whose actions depend on $\xi_h$ can lead to spurious correlations between endogenous exogenous states. In spite of this apparent difficulty, we show that the \coverability coefficient for this setting is always bounded by the number of \emph{endogenous states}.
\begin{proposition}
  \label{prop:ex_bmdp}
  For any \exmdp, $\Ccov\leq\abs{\cS}\cdot\abs{\cA}$.
\end{proposition}
This bound is a consequence of a structural result from \cite{efroni2021provably}, which shows that for any $(s,a)\in\cS\times\cA$, all $x\in\cX$ with $\phistar(x)=s$ admit a common policy that maximizes $d_h^{\pi}(x,a)$, and this policy is \emph{endogenous}, i.e., only depends on the endogenous state $s_h=\phistar_h(x_h)$. As a corollary, we obtain the following regret bound.
\begin{corollary}
  \label{cor:ex_bmdp}
  For the \exmdp setting, under \pref{asm:completeness}, \pref{alg:golf} ensures that with probability at least $1-\delta$,
  \[
    \Reg\leq \bigoh\prn[\big]{
    H\sqrt{\abs{\cS}\abs{\cA}T \log(\nicefrac{TH |\Fcal|}{\delta})\log(T)}
    }.
  \]
\end{corollary}
Critically, this result scales only with the cardinality $\abs{\cS}$ for the endogenous latent state space, and with the capacity $\log\abs{\cF}$ for the value function class.

Let us briefly compare to prior work. For general \exmdps, existing complexity measures such as Bellman Rank and \betext can be arbitrarily large (see discussion in \pref{sec:gen_bedim}). Existing algorithms either require that the endogenous latent dynamics $\Pendo$ are deterministic \citep{efroni2021provably} or allow for stochastic dynamics but heavily restrict the observation process \citep{efroni2022sample}. \pref{cor:ex_bmdp} is the first result for this setting that allows for stochastic latent dynamics and emission process, albeit with the extra assumption of completeness. This result is best thought of as a ``luckiness'' guarantee in the sense that it is unclear how to construct a value function class that is complete for every problem instance,\footnote{For example, it is not clear how to construct a complete value function class given access to a class of decoders $\Phi$ that contains $\phistar$.} but the algorithm will succeed whenever $\cF$ does happen to be complete for a given instance. Understanding whether general \exmdps are learnable without completeness is an interesting question for future work, and we are hopeful that the perspective of \coverability will lead to further insights for this setting.

\subsection{Invariance of \coverability}%
\pref{prop:ex_bmdp} is a consequence of two general \emph{invariance} properties of \coverability, which show that $\Ccov$ is unaffected by the following augmentations to the underlying MDP: (i) addition of rich observations, and (ii) addition of exogenous noise.

The first property shows that for a given MDP $M$, creating a new block MDP $M'$ by equipping $M$ with a decodable emission process (so that $M$ acts as a \emph{latent MDP}), does not increase \coverability. %
\begin{proposition}[Invariance to rich observations]
  \label{prop:rich_obs}
  Let an MDP $M=(\cS,\cA,P,R,H,s_1)$. Let $M'=(\cX,\cA,P', R',H, x_1)$ be the MDP defined implicitly by the following process. For each $h\in\brk*{H}$:
  \begin{itemize}
  \item $s_{h+1}\sim{}P_h(s_h,a_h)$ and $r_h=R_h(s_h,a_h)$. Here, $s_h$ is unobserved, and may be thought of as a latent state.
  \item $x_{h}\sim{}q_h(s_h)$, where $q_h:\cS\to\Delta(\cX)$ is an \emph{emission distribution} with the property that $\supp(q_h(s))\cap\supp(q_h(s'))=\emptyset$ for $s\neq{}s'$.
  \end{itemize}
  Then, writing $\Ccov(M)$ to make the dependence on $M$ explicit, we have
  \[
    \Ccov(M') \leq \Ccov(M).
  \]
\end{proposition}
The second result shows that \coverability is also preserved if we expand the state space to include temporally correlated exogenous state whose evolution does not depend on the agent's actions. %
\begin{proposition}[Invariance to exogenous noise]
  \label{prop:exo}
  Let an MDP $M=(\cS,\cA,P,R,H,s_1)$, conditional distribution $\Pexo:\Xi\to\Delta(\Xi)$, and $\xi_1\in\Xi$ be given, where $\Xi$ is an abstract set. Let $\cX\ldef{}\cS\times\Xi$, and let $M'=(\cX,\cA,P', R',H, x_1)$ be the MDP with state $x_h=(s_h,\xi_h)$ defined implicitly by the following process. For each $h\in\brk{H}$:
  \begin{itemize}
  \item $s_{h+1}\sim{}P_{h}(s_h,a_h)$, $r_h=R_h(s_h,a_h)$.
  \item $\xi_{h+1}\sim{}\Pexo_h(\xi_h)$.
  \end{itemize}
  Then we have
  \[
    \Ccov(M') \leq \Ccov(M).
  \]
\end{proposition}
This result is non-trivial because policies that act based on the endogenous state $s_h$ and $\xi_h$ can cause these processes to become coupled \citep{efroni2021provably}, but holds nonetheless.

\pref{prop:ex_bmdp} can be deduced by combining \pref{prop:rich_obs,prop:exo} with the observation that any tabular (finite-state/action) MDP with $S$ states and $A$ actions has $\Ccov\leq{}SA$. However, \pref{prop:rich_obs,prop:exo} yield more general results, since they imply that starting with any (potentially non-tabular) class of MDPs $\cM$ with low \coverability and augmenting it with rich observations and exogenous noise preserves coverability.

\subsection{Proofs}

\begin{proof}[\pfref{prop:ex_bmdp}]
  Let $h\in\brk{H}$ be fixed. Let $z_h\ldef{}(s_h,\xi_h)$. For each $z=(s,\xi)\in\cS\times\Xi$, let
  $d^{\pi}_h(z)\ldef{}\bbP^{\pi}(z_h=z)$. Proposition 4 of
  \citet{efroni2021provably} shows that for all $z=(s,\xi)$, if we
  define $\pi_{s}=\argmax_{\pi\in\Pi}\bbP^{\pi}(s_h=s)$, then
  \begin{equation}
    \label{eq:exbmdp_policy_cover}
    \max_{\pi\in\Pi}d^{\pi}_h(z)= d_h^{\pi_s}(z).
  \end{equation}
  That is, $\pi_s$ maximizes $\bbP^{\pi}(z_h=(s,\xi))$ for all
  $\xi\in\Xi$ simultaneously. With this in mind, let us define
  \[
    \mu_h(x,a) = \frac{1}{\abs{\cS}\abs{\cA}}\sum_{s\in\cS}d_h^{\pi_s}(x).
  \]
  We proceed to bound the concentrability coefficient for $\mu$. Fix $\pi\in\Pi$ and $x\in\cX$, and let $z=(s,\xi)\in\cS\times\Xi$ be the unique latent
  state such that $x\in\supp(q_h(s,\xi))$. We first observe that
  \begin{align*}
    \frac{d^{\pi}_h(x,a)}{\mu_h(x,a)}
    \leq{}     \abs{\cS}\abs{\cA}\cdot{}\frac{d^{\pi}_h(x)}{d^{\pi_s}_h(x)}.
  \end{align*}
  Next, since $x_h\sim{}q_h(z_h)$, we have
  \begin{align*}
    \frac{d^{\pi}_h(x)}{d^{\pi_s}_h(x)}
    = \frac{q_h(x\mid{}z)d^{\pi}_h(z)}{q_h(x\mid{}z)d^{\pi_s}_h(z)} = \frac{d^{\pi}_h(z)}{d^{\pi_s}_h(z)}.
  \end{align*}
  Finally, by \pref{eq:exbmdp_policy_cover}, we have
  \begin{align*}
    \frac{d^{\pi}_h(z)}{d^{\pi_s}_h(z)}
    \leq{} \frac{\max_{\pi}d^{\pi}_h(z)}{d^{\pi_s}_h(z)}
    = \frac{d^{\pi_s}_h(z)}{d^{\pi_s}_h(z)} = 1.
  \end{align*}
Since this holds for all $x\in\cX$ simultaneously, this choice for
$\mu_h$ certifies that 
that $\Ccov\leq{}\abs{\cS}\abs{\cA}$.
\end{proof}

\begin{proof}[\pfref{prop:rich_obs}]
  Let $\Pi$ denote the space of all randomized policies acting on the
  latent state space $\cS$, and let $\Pi'$ denote the space of all
  randomized policies acting on the observed state space $\cX$. Let
  $\bbP^{\pi}$ denote distribution over trajectories in $M$ induced by
  $\pi\in\Pi$, and let $\bbQ^{\pi'}$ denote the distribution over
  trajectories in $M$ induced by $\pi'\in\Pi'$.

  Fix $h\in\brk{H}$, and let $\mu_h\in\Delta(\cS\times\cA)$ witness
  the \coverability coefficient for $M$. Define
  \[
    \mu'_h(x,a)=q_h(x\mid{}\phistar(x))\mu_h(\phistar(x),a),
  \]
  where $\phistar_h:\cX\to\cS$ is the decoder that maps $x\in\cX$ to
  the   unique state $s\in\cS$ such that $x\in\supp(q_h(s))$. For any
  $\pi'\in\Pi'$ and $(x,a)\in\cX\times\cA$, letting $s=\phistar_h(x)$, we have
  \begin{align*}
    \frac{d^{\pi'}_h(x,a)}{\mu'_h(x,a)}
    =     \frac{q_h(x\mid{}s)\bbQ^{\pi'}(s_h=s,a_h=a)}{q_h(x\mid{}s)\mu_h(s,a)}
    =     \frac{\bbQ^{\pi'}(s_h=s,a_h=a)}{\mu_h(s,a)}
    \leq{} \frac{\max_{\pi'\in\Pi}\bbQ^{\pi'}(s_h=s,a_h=a)}{\mu_h(s,a)}.
  \end{align*}
Finally, because the observation process is decodable, we have
$\max_{\pi'\in\Pi}\bbQ^{\pi'}(s_h=s,a_h=a)=\max_{\pi\in\Pi}\bbP^{\pi}(s_h=s,a_h=a)$,
and
\[
  \frac{\max_{\pi\in\Pi}\bbP^{\pi}(s_h=s,a_h=a)}
  {\mu_h(s,a)} \leq{} \Ccov(M).
\]

\end{proof}

\begin{proof}[\pfref{prop:exo}]
    Let $\Pi$ denote the space of all randomized policies acting on the
  latent state space $\cS$, and let $\Pi'$ denote the space of all
  randomized policies acting on the observed state space $\cX$. Let
  $\bbP^{\pi}$ denote distribution over trajectories in $M$ induced by
  $\pi\in\Pi$, and let $\bbQ^{\pi'}$ denote the distribution over
  trajectories in $M$ induced by $\pi'\in\Pi'$.

    Fix $h\in\brk{H}$, and let $\mu_h\in\Delta(\cS\times\cA)$ witness
  the \coverability coefficient for $M$. For
  $x=(s,\xi)\in\cS\times\Xi$, let 
  \[
    \mu'_h(x,a)=\bbQ(\xi_h=\xi)\mu_h(s,a),
  \]
  where $\bbQ(\xi_h=\xi)$ is the marginal probability of the event
  that $\xi_h=\xi$ in $M'$, which does not depend on the policy under consideration.

  For any
  $\pi'\in\Pi$ and $(s,\xi,a)\in\cS\times\Xi\times\cA$, we have
  \begin{align*}
    \frac{d^{\pi'}_h(x,a)}{\mu'_h(x,a)}
    =
    \frac{\bbQ^{\pi'}(s_h=s,\xi_h=\xi,a_h=a)}{\bbQ(\xi_h=\xi)\mu_h(s,a)}
    \leq     \frac{\max_{\pi'\in\Pi'}\bbQ^{\pi'}(s_h=s,\xi_h=\xi,a_h=a)}{\bbQ(\xi_h=\xi)\mu_h(s,a)}.
  \end{align*}
  From Propositions 3 and 4 of \citet{efroni2021provably}, we have
  $\max_{\pi'\in\Pi'}\bbQ^{\pi'}(s_h=s,\xi_h=\xi,a_h=a)=
  \bbQ(\xi_h=\xi)\cdot \max_{\pi'\in\Pi'}\bbQ^{\pi'}(s_h=s,a_h=a)=
  \bbQ(\xi_h=\xi)\cdot \max_{\pi\in\Pi}\bbP^{\pi}(s_h=s,a_h=a)$, so
  that
  \begin{align*}
    \frac{\max_{\pi'\in\Pi'}\bbQ^{\pi'}(s_h=s,\xi_h=\xi,a_h=a)}{\bbQ(\xi_h=\xi)\mu_h(s,a)}
    &=
    \frac{\bbQ(\xi_h=\xi)\max_{\pi\in\Pi}\bbP^{\pi}(s_h=s,a_h=a)}{\bbQ(\xi_h=\xi)\mu_h(s,a)} \\
& = \frac{\max_{\pi\in\Pi}\bbP^{\pi}(s_h=s,a_h=a)}{\mu_h(s,a)} \leq \Ccov(M).
  \end{align*}
  
\end{proof}

 }

\arxiv{\section{Proofs from \creftitle{sec:basic_cover}}}
\iclr{\section{Proofs and Additional Details from \creftitle{sec:basic_cover}}}
\label{app:main}
\iclr{
\subsection{\golf Algorithm and Proofs from \creftitle{sec:basic_cover}}
\begin{algorithm}[th]
\caption{\golf \citep{jin2021bellman}}
\label{alg:golf}
{\bfseries input:} Function class $\Fcal$, confidence width $\beta>0$. \\
{\bfseries initialize:} $\Fcal^\iter{0} \leftarrow \Fcal$, $\Dcal_{h}^\iter{0} \leftarrow \emptyset\;\;\forall h \in [H]$. 
\begin{algorithmic}[1]
\For{episode $t = 1,2,\dotsc,T$}
    \State Select policy $\pi^\iter{t} \leftarrow \pi_{f^\iter{t}}$, where $f^\iter{t} \ldef{} \argmax_{f \in \Fcal^\iter{t-1}}f(x_1,\pi_{f,1}(x_1))$. \label{step:glof_optimism}
    \State Execute $\pi^\iter{t}$ for one episode and obtain trajectory $(x_1^\iter{t},a_1^\iter{t},r_1^\iter{t}),\ldots,(x_H^\iter{t},a_H^\iter{t},r_H^\iter{t})$. \label{step:glof_sampling}
    \State Update dataset: $\Dcal_{h}^\iter{t} \leftarrow \Dcal_{h}^\iter{t-1} \cup \crl[\big]{\prn[\big]{x_h^\iter{t},a_h^\iter{t},x_{h+1}^\iter{t}}}\;\;\forall h \in [H]$.
    \State Compute confidence set:
    \begin{gather*}
    \Fcal^\iter{t} \leftarrow \crl[\bigg]{ f \in \Fcal: \Lcal_{h}^\iter{t}(f_h,f_{h+1}) - \min_{f'_h \in \Fcal_h} \Lcal_{h}^\iter{t}(f'_h,f_{h+1}) \leq \beta\;\;\forall h \in [H] },
    \\
    \nonumber
    \text{where \quad } \Lcal_{h}^\iter{t}(f,f') \coloneqq \sum_{(x,a,r,x') \in \Dcal_{h}^\iter{t}}\prn[\Big]{ f(x,a) - r - \max_{a' \in \Acal} f'(x',a') }^2 ,~\forall f,f' \in \Fcal.
    \end{gather*}
\EndFor
\State Output $\pibar = \unif(\pi^\iter{1:T})$. \algcommentlight{For PAC guarantee only.}
\end{algorithmic}
\end{algorithm} }
\begin{lemma}[{\citet[Lemmas 39 and 40]{jin2021bellman}}]
  \label{lem:golf_concentration}
  Suppose \pref{asm:completeness} holds. Then if $\beta>0$ is selected as in \pref{thm:golf_guarantee_basic}, then
  with probability at least $1-\delta$, for all $t\in\brk{T}$,
  \pref{alg:golf} satisfies
  \begin{enumerate}
  \item $\Qstar\in\cF\ind{t}$.
  \item
    $\sum_{i<t}\En_{(x,a)\sim{}d_h\ind{i}}\brk[\big]{\prn*{f_h(x,a)-\brk{\cT_hf_{h+1}}(x,a)}^2}\leq\bigoh(\beta)$
      for all $f\in\cF\ind{t}$.
  \end{enumerate}
\end{lemma}

\begin{lemma}[{\citet[Lemma 1]{jiang2017contextual}}]
  \label{lem:regret_optimistic}
  For any value function $f=(f_1,\ldots,f_H)$,
  \begin{align*}
    f_1(x_1,\pi_{f_1,1}(x_1)) -
    J(\pi_f)
    =
\sum_{h=1}^{H}\E_{(x,a)\sim{}d^{{\pi_f}}_h}\brk*{f_h(x,a)-(\cT_hf_{h+1})(x,a)}.
  \end{align*}
\end{lemma}

\iclr{
\begin{lemma}[Equivalence of \coverability and \reachability]
\label{lem:concen_eq_area}
The following definition is equivalent to
\pref{def:low_concentrability}:
\begin{align*}
\Ccov \ldef{} \max_{h\in\brk{H}}\sum_{(x,a) \in \Xcal \times \Acal}\sup_{\pi \in \Pi} d_h^\pi(x,a).
\end{align*}
\end{lemma}
}

\begin{proof}[\cpfname{lem:concen_eq_area}]
We relate \coverability and cumulative reachability for each choice for $h\in\brk{H}$.

{\em \Coverability bounds \reachability.}
It follows immediately from
the definition of \coverability that if
$\mu_h\in\Delta(\cX\times\cA)$ realizes the value of $\Ccov$, then
\begin{align*}
\sum_{(x,a) \in \Xcal \times \Acal} \max_{\pi \in \Pi} d_h^\pi(x,a) = &~ \sum_{(x,a) \in \Xcal \times \Acal} \frac{\max_{\pi \in \Pi} d_h^\pi(x,a)}{\mu_h(x,a)} \mu_h(x,a)
\\
\leq &~ \sum_{(x,a) \in \Xcal \times \Acal} C_\on \cdot \mu_h(x,a)
\tag{by \cref{def:low_concentrability}}
\\
= &~ C_\on.
\end{align*}

{\em Cumulative reachability bounds \coverability.}
Define $\mu_h(x,a) \propto \max_{\pi \in \Pi} d_h^\pi(x,a)$. Then for
any $\pi \in \Pi$ and any $(x,a) \in \Xcal \times \Acal$, we have
\begin{align*}
\frac{d_h^\pi(x,a)}{\mu_h(x,a)} = &~ \frac{d_h^\pi(x,a)}{\nicefrac{\max_{\pi'' \in \Pi} d_h^{\pi''}(x,a)}{\sum_{(x',a') \in \Xcal \times \Acal}\max_{\pi' \in \Pi} d_h^{\pi'}(x',a')}}
\\
\leq &~ \sum_{(x',a') \in \Xcal \times \Acal}\max_{\pi' \in \Pi} d_h^{\pi'}(x',a').
\end{align*}
This completes the proof.
\end{proof}

\iclr{
\begin{proof}[\pfref{thm:golf_guarantee_basic}]
\arxiv{Equipped with \pref{lem:concen_eq_area}, we proceed with the proof of \pref{thm:golf_guarantee_basic}.}
\iclr{Equipped with \pref{lem:concen_eq_area}, we prove \pref{thm:golf_guarantee_basic}.}

\paragraph{Preliminaries}
For each $t$, we define $\delta^\iter{t}_h(\cdot,\cdot) \coloneqq f_h^\iter{t}(\cdot,\cdot) - (\Tcal_h f_{h+1}^\iter{t})(\cdot,\cdot)$, which may be viewed as a ``test function'' at level $h$ induced by $f\ind{t} \in \Fcal$. We adopt the shorthand $d_h^\iter{t}\equiv{}d_h^{\pi^\iter{t}}$, and we define
\begin{align}
\label{eq:def_dbar}
  &\dtilde_h^\iter{t} (x,a) \coloneqq~ \sum_{i = 1}^{t - 1} d_h^\iter{i} (x,a),\mathand
  \mu^\star_h \coloneqq ~ \argmin_{\mu_h\in\Delta(\cX\times\cA)} \sup_{\pi\in\Pi}\, \nrm*{\frac{d_h^\pi}{\mu_h}}_{\infty}.
\end{align}
That is, $\dtilde_h^\iter{t}$ unnormalized average of all state visitations encountered prior to step $t$, and $\mu^\star_h$ is the distribution that attains the value of $\Ccov$ for layer $h$.\footnote{If the minimum in \pref{eq:def_dbar} is not obtained, we can repeat the argument that follows for each element of a limit sequence attaining the infimum.} Throughout the proof, we perform a slight abuse of notation and write $\E_{\dtilde_h^\iter{t}}[f] \coloneqq \sum_{i = 1}^{t - 1} \E_{d_h^\iter{i}}[f]$ for any function $f:\Xcal \times \Acal \to \RR$.

\paragraph{Regret decomposition}
As a consequence of completeness (\pref{asm:completeness}) and the construction of $\cF\ind{t}$, a standard concentration argument (\pref{lem:golf_concentration}\arxiv{ in \pref{app:main}}) guarantees that with probability at least $1-\delta$, for all $t\in\brk{T}$:
\begin{align}
\mathrm{(i)}\;\;\Qstar\in\cF\ind{t},\mathand\mathrm{(ii)}\;\;\sum_{x,a} \dtilde_h^\iter{t}(x,a) \left(\delta_h^\iter{t}(x,a)\right)^2\leq\bigoh(\beta).\label{eq:golf_concentration}
\end{align}
We condition on this event going forward. Since $\Qstar\in\cF\ind{t}$, we are guaranteed that $f\ind{t}$ is optimistic (i.e., $f_1\ind{t}(x_1,\pi_{f\ind{t},1}(x_1))\geq{}\Qstar_1(x_1,\pi_{\fstar,1}(x_1))$), and a regret decomposition for optimistic algorithms (\pref{lem:regret_optimistic}\arxiv{ in \pref{app:main}}) allows us to relate regret to the average Bellman error under the learner's sequence of policies:
\begin{align*}
  \Reg \leq{} \sum_{t = 1}^{T} \left( f_1^\iter{t}(x_1,\pi_{f\ind{t}_1,1}(x_1)) -J(\pi\ind{t}) \right) 
  = \sum_{t = 1}^{T} \sum_{h=1}^{H}\E_{(x,a)\sim{}d_h^\iter{t}}\big[\underbrace{f_h\ind{t}(x,a)-(\cT_hf_{h+1}\ind{t})(x,a)}_{\rdef\delta_h^\iter{t}(x,a)}\big].
\end{align*}
To proceed, we use a change of measure argument to relate the on-policy \emph{average} Bellman error $\E_{(x,a)\sim{}d_h^\iter{t}}[\delta_h^\iter{t}(x,a)]$ appearing above to the in-sample \emph{squared} Bellman error $\E_{(x,a)\sim{}\dtilde_h^\iter{t}}[\delta_h^\iter{t}(x,a)^2]$; the latter is small as a consequence of \pref{eq:golf_concentration}. Unfortunately, naive attempts at applying change-of-measure fail because during the initial rounds of exploration, the on-policy and in-sample visitation probabilities can be very different, making it impossible to relate the two quantities (i.e., any natural notion of extrapolation error will be arbitrarily large).

To address this issue, we introduce the notion of a ``burn-in'' phase for each state-action pair $(x,a)\in\cX\times\cA$ by defining
\[
\tau_h(x,a) = \min\left\{t \mid \dtilde_h^\iter{t}(x,a) \geq \Ccov\cdot\mu^{\star}_h(x,a)\right\},
\] %
which captures the earliest time at which $(x,a)$ has been explored sufficiently; we refer to $t<\tau_h(x,a)$ as the burn-in phase for $(x,a)$.

Going forward, let $h\in\brk{H}$ be fixed. We decompose regret into contributions from the burn-in phase for each state-action pair, and contributions from pairs which have been explored sufficiently and reached a stable phase ``stable phase''.
\begin{align*}
\underbrace{\sum_{t=1}^{T}\E_{(x,a)\sim{}d_h^\iter{t}}\left[\delta_h^\iter{t}(x,a)\right]}_{\text{on-policy average Bellman error}} = \underbrace{\sum_{t=1}^{T}\E_{(x,a)\sim{}d_h^\iter{t}}\left[\delta_h^\iter{t}(x,a)\1[t < \tau_h(x,a)]\right]}_{\text{burn-in phase}} +
  \underbrace{\sum_{t=1}^{T}\E_{(x,a)\sim{}d_h^\iter{t}}\left[\delta_h^\iter{t}(x,a)\1[t \geq \tau_h(x,a)]\right]}_{\text{stable phase}}.
\end{align*}
We will not show that every state-action pair leaves the burn-in phase. Instead, we use \coverability to argue that the contribution from pairs that have not left this phase is small on average.
In particular, we use that $\abs{\delta_h\ind{t}}\leq{}1$ to bound
\begin{align*}
  \sum_{t = 1}^{T} \E_{(x,a)\sim{}d_h^\iter{t}}\left[\delta_h^\iter{t}(x,a)\1[t < \tau_h(x,a)]\right] \leq
  \sum_{x,a}\sum_{t<\tau_h(x,a)} d_h^\iter{t}(x,a)=
  \sum_{x,a}\dtil_h\ind{\tau_h(x,a)}(x,a)
  \leq{}2\Ccov\sum_{x,a}\mustar_h(x,a)=2\Ccov,   
\end{align*}
where the last inequality holds because \[\dtil \ind{\tau_h(x,a)}_h(x,a) = \dtilde_h^\iter{\tau_h(x,a) - 1}(x,a) + d_h^\iter{\tau_h(x,a) - 1}(x,a) \leq 2C_\on\cdot\mu_h^\star(x,a),\] which follows from \cref{eq:def_dbar} and the definition of $\tau_h$.%

For the stable phase, we apply change-of-measure as follows:
\begin{align}
\nonumber
&~ \sum_{t = 1}^{T} \E_{(x,a)\sim{}d_h^\iter{t}}\left[\delta_h^\iter{t}(x,a)\1[t \geq \tau_h(x,a)]\right]
\\
\nonumber
&= ~ \sum_{t = 1}^{T} \sum_{x,a} d_h^\iter{t}(x,a) \left( \frac{\dtilde_h^\iter{t}(x,a)}{\dtilde_h^\iter{t}(x,a)} \right)^{\nicefrac{1}{2}} \delta_h^\iter{t}(x,a) \1[t \geq \tau_h(x,a)]
\\
\label{eq:reg_CS}
&\leq~ \underbrace{\sqrt{\sum_{t = 1}^{T} \sum_{x,a} \frac{\left( \1[t \geq \tau_h(x,a)] d_h^\iter{t}(x,a) \right)^2}{\dtilde_h^\iter{t}(x,a)} }}_{\text{{\tt (I)}: extrapolation error}} \cdot  \underbrace{\sqrt{\sum_{t = 1}^{T} \sum_{x,a} \dtilde_h^\iter{t}(x,a) \left(\delta_h^\iter{t}(x,a)\right)^2}}_{\text{{\tt (II)}: in-sample \emph{squared} Bellman error}},
\end{align}
where the last inequality is an application of Cauchy-Schwarz. Using part \texttt{(II)} of \pref{eq:golf_concentration}, we bound the in-sample error above by
\begin{align}
\label{eq:reg_CS_term2_main}
\texttt{(II)} \leq O\prn[\big]{\sqrt{\beta T}}.
\end{align}

\paragraph{Bounding the extrapolation error using \coverability}
To proceed, we show that the extrapolation error \texttt{(I)} is controlled by \coverability. 
We begin with a scalar variant of the standard elliptic potential lemma \citep{lattimore2020bandit}; \arxiv{this result is proven in \cref{app:main} for completeness.}\iclr{this result is proven in the sequel.}
\begin{lemma}[Per-state-action elliptic potential lemma]
\label{lem:per_sa_ep}
Let $d^\iter{1}, d^\iter{2}, \dotsc, d^\iter{T}$ be an arbitrary sequence of distributions over a set $\Zcal$ (e.g., $\Zcal = \Xcal \times \Acal$), and let $\mu\in\Delta(\Zcal)$ be a distribution such that $d^\iter{t}(z) / \mu(z) \leq C$ for all $(z,t) \in \Zcal \times [T]$. Then for all $z \in \Zcal$, we have
\begin{align*}
\sum_{t = 1}^{T} \frac{d^\iter{t}(z)}{\sum_{i < t} d^\iter{i}(z) + C \cdot \mu(z)} \leq O \left(\log\left( T \right) \right).
\end{align*}
\end{lemma}
We bound the extrapolation error \texttt{(I)} by applying \pref{lem:per_sa_ep} on a \emph{per-state basis}, then using \coverability (and the equivalence to \reachability) to argue that the potentials from different state-action pairs average out. Observe that by the definition of $\tau_h$, we have that for all $t \geq \tau_h(s,a)$, $\dtilde_h^\iter{t}(x,a) \geq C_\on \mu_h^\star(x,a) \Rightarrow \dtilde_h^\iter{t}(x,a)\geq \frac{1}{2}(\dtilde_h^\iter{t}(x,a) + C_\on \mu_h^\star(x,a))$, which allows us to bound term \texttt{(I)} of extrapolation error by
\begin{align}
\nonumber
\sum_{t = 1}^{T} \sum_{x,a} \frac{\left( \1[t \geq \tau_h(x,a)] d_h^\iter{t}(x,a) \right)^2}{\dtilde_h^\iter{t}(x,a)} \leq &~ 2 \sum_{t = 1}^{T} \sum_{x,a} \frac{ d_h^\iter{t}(x,a) \cdot d_h^\iter{t}(x,a)}{\dtilde_h^\iter{t}(x,a) + C_\on \cdot \mu_h^\star(x,a)}
\\
\leq &~ 2\sum_{t = 1}^{T} \sum_{x,a} \max_{t'\in\brk{T}}d_h\ind{t'}(x,a)\cdot\frac{ d_h^\iter{t}(x,a) }{\dtilde_h^\iter{t}(x,a) + C_\on \cdot \mu_h^\star(x,a)}
       \nonumber
\\
\nonumber
\leq &~ 2\underbrace{\left( \max_{(s,a) \in \Scal \times \Acal} \sum_{t = 1}^{T} \frac{ d_h^\iter{t}(x,a)}{\dtilde_h^\iter{t}(x,a) + C_\on \cdot \mu_h^\star(x,a)} \right)}_{\leq \bigoh(\log(T)) \text{ by \cref{lem:per_sa_ep}}} \cdot \underbrace{\left( \sum_{x,a}\max_{t\in\brk{T}} d\ind{t}_h(x,a) \right)}_{\leq C_\on \text{ by \pref{lem:concen_eq_area}}}
\\
\label{eq:reg_CS_term1_main}
\leq &~ O \left( C_\on \log\left( T \right) \right).
\end{align}
To conclude, we substitute \cref{eq:reg_CS_term2_main,eq:reg_CS_term1_main} into \cref{eq:reg_CS}, which gives
\begin{align*}
\Reg \leq \sum_{h = 1}^{H} \E_{(x,a)\sim{}d_h^\iter{t}}\left[\delta_h^\iter{t}(x,a)\right] \leq O \left( H \sqrt{\Ccov\cdot\beta T \log(T) } \right).
\end{align*}
\arxiv{
\qed

  To obtain the expression in \pref{eq:reg_CS} (term \texttt{(I)}), our proof critically uses that the confidence set construction provides a bound on the \emph{squared Bellman error} $\E_{(x,a)\sim{}\dtilde_h^\iter{t}}[\delta_h^\iter{t}(x,a)^2]$ in the change of measure argument. This contrasts with existing works on online RL with general function approximation~\citep[e.g.,][]{jiang2017contextual,jin2021bellman,du2021bilinear}, which typically move from average Bellman error to squared Bellman error as a lossy step, and only work with squared Bellman error because it permits simpler construction of confidence sets. For the argument in \pref{eq:reg_CS}, confidence sets based on average Bellman error will lead to a larger notion of extrapolation error which cannot be controlled using \coverability (cf. \pref{sec:gen_bedim}).
}
 \end{proof}
}

\begin{proof}[\cpfname{lem:per_sa_ep}]
Using the fact for any $u \in [0,1]$, $u \leq 2 \log(1 + u)$, we have
\begin{align*}
\sum_{t = 1}^{T} \frac{d^\iter{t}(z)}{\sum_{i < t} d^\iter{i}(z) + C \cdot \mu(x,a)} \leq &~ 2 \sum_{t = 1}^{T} \log\left(1 + \frac{d^\iter{t}(x,a)}{\sum_{i < t} d^\iter{i}(z) + C \cdot \mu(x,a)}\right)
\tag{since $d^\iter{t}(x,a) / \mu(x,a) \leq C\;\;\forall t \in [T]$}
\\
= &~ 2 \sum_{t = 1}^{T} \log\left(\frac{\sum_{i < t+1} d^\iter{i}(z) + C \cdot \mu(x,a)}{\sum_{i < t} d^\iter{i}(z) + C \cdot \mu(x,a)}\right)
\\
= &~ 2 \log\left(\prod_{t = 1}^{T} \frac{\sum_{i < t+1} d^\iter{i}(z) + C \cdot \mu(x,a)}{\sum_{i < t} d^\iter{i}(z) + C \cdot \mu(x,a)}\right)
\\
= &~ 2 \log\left(\frac{\sum_{i = 1}^{T} d^\iter{i}(z) + C \cdot \mu(x,a)}{C \cdot \mu(x,a)}\right)
\\
\leq &~ 2 \log(T + 1).
\tag{since $d\iter{t}(x,a) / \mu(x,a) \leq C\;\;\forall t \in [T]$}
\end{align*}
This completes the proof.
\end{proof}

\arxiv{

\begin{proof}[\pfref{prop:ex_bmdp}]
  Let $h\in\brk{H}$ be fixed. Let $z_h\ldef{}(s_h,\xi_h)$. For each $z=(s,\xi)\in\cS\times\Xi$, let
  $d^{\pi}_h(z)\ldef{}\bbP^{\pi}(z_h=z)$. Proposition 4 of
  \citet{efroni2021provably} shows that for all $z=(s,\xi)$, if we
  define $\pi_{s}=\argmax_{\pi\in\Pi}\bbP^{\pi}(s_h=s)$, then
  \begin{equation}
    \label{eq:exbmdp_policy_cover}
    \max_{\pi\in\Pi}d^{\pi}_h(z)= d_h^{\pi_s}(z).
  \end{equation}
  That is, $\pi_s$ maximizes $\bbP^{\pi}(z_h=(s,\xi))$ for all
  $\xi\in\Xi$ simultaneously. With this in mind, let us define
  \[
    \mu_h(x,a) = \frac{1}{\abs{\cS}\abs{\cA}}\sum_{s\in\cS}d_h^{\pi_s}(x).
  \]
  We proceed to bound the concentrability coefficient for $\mu$. Fix $\pi\in\Pi$ and $x\in\cX$, and let $z=(s,\xi)\in\cS\times\Xi$ be the unique latent
  state such that $x\in\supp(q_h(s,\xi))$. We first observe that
  \begin{align*}
    \frac{d^{\pi}_h(x,a)}{\mu_h(x,a)}
    \leq{}     \abs{\cS}\abs{\cA}\cdot{}\frac{d^{\pi}_h(x)}{d^{\pi_s}_h(x)}.
  \end{align*}
  Next, since $x_h\sim{}q_h(z_h)$, we have
  \begin{align*}
    \frac{d^{\pi}_h(x)}{d^{\pi_s}_h(x)}
    = \frac{q_h(x\mid{}z)d^{\pi}_h(z)}{q_h(x\mid{}z)d^{\pi_s}_h(z)} = \frac{d^{\pi}_h(z)}{d^{\pi_s}_h(z)}.
  \end{align*}
  Finally, by \pref{eq:exbmdp_policy_cover}, we have
  \begin{align*}
    \frac{d^{\pi}_h(z)}{d^{\pi_s}_h(z)}
    \leq{} \frac{\max_{\pi}d^{\pi}_h(z)}{d^{\pi_s}_h(z)}
    = \frac{d^{\pi_s}_h(z)}{d^{\pi_s}_h(z)} = 1.
  \end{align*}
Since this holds for all $x\in\cX$ simultaneously, this choice for
$\mu_h$ certifies that 
that $\Ccov\leq{}\abs{\cS}\abs{\cA}$.
\end{proof}

\begin{proof}[\pfref{prop:rich_obs}]
  Let $\Pi$ denote the space of all randomized policies acting on the
  latent state space $\cS$, and let $\Pi'$ denote the space of all
  randomized policies acting on the observed state space $\cX$. Let
  $\bbP^{\pi}$ denote distribution over trajectories in $M$ induced by
  $\pi\in\Pi$, and let $\bbQ^{\pi'}$ denote the distribution over
  trajectories in $M$ induced by $\pi'\in\Pi'$.

  Fix $h\in\brk{H}$, and let $\mu_h\in\Delta(\cS\times\cA)$ witness
  the \coverability coefficient for $M$. Define
  \[
    \mu'_h(x,a)=q_h(x\mid{}\phistar(x))\mu_h(\phistar(x),a),
  \]
  where $\phistar_h:\cX\to\cS$ is the decoder that maps $x\in\cX$ to
  the   unique state $s\in\cS$ such that $x\in\supp(q_h(s))$. For any
  $\pi'\in\Pi'$ and $(x,a)\in\cX\times\cA$, letting $s=\phistar_h(x)$, we have
  \begin{align*}
    \frac{d^{\pi'}_h(x,a)}{\mu'_h(x,a)}
    =     \frac{q_h(x\mid{}s)\bbQ^{\pi'}(s_h=s,a_h=a)}{q_h(x\mid{}s)\mu_h(s,a)}
    =     \frac{\bbQ^{\pi'}(s_h=s,a_h=a)}{\mu_h(s,a)}
    \leq{} \frac{\max_{\pi'\in\Pi'}\bbQ^{\pi'}(s_h=s,a_h=a)}{\mu_h(s,a)}.
  \end{align*}
  Finally, because the observation process is decodable, any Markov
  policy acting on $x_h$ can be viewed as a randomized Markov policy
  acting on $s_h$. As a result, we have
$\max_{\pi'\in\Pi'}\bbQ^{\pi'}(s_h=s,a_h=a)=\max_{\pi\in\Pi}\bbP^{\pi}(s_h=s,a_h=a)$,
and
\[
  \frac{\max_{\pi\in\Pi}\bbP^{\pi}(s_h=s,a_h=a)}
  {\mu_h(s,a)} \leq{} \Ccov(M).
\]

\end{proof}

\begin{proof}[\pfref{prop:exo}]
    Let $\Pi$ denote the space of all randomized policies acting on the
  latent state space $\cS$, and let $\Pi'$ denote the space of all
  randomized policies acting on the observed state space $\cX$. Let
  $\bbP^{\pi}$ denote distribution over trajectories in $M$ induced by
  $\pi\in\Pi$, and let $\bbQ^{\pi'}$ denote the distribution over
  trajectories in $M$ induced by $\pi'\in\Pi'$.

    Fix $h\in\brk{H}$, and let $\mu_h\in\Delta(\cS\times\cA)$ witness
  the \coverability coefficient for $M$. For
  $x=(s,\xi)\in\cS\times\Xi$, let 
  \[
    \mu'_h(x,a)=\bbQ(\xi_h=\xi)\mu_h(s,a),
  \]
  where $\bbQ(\xi_h=\xi)$ is the marginal probability of the event
  that $\xi_h=\xi$ in $M'$, which does not depend on the policy under consideration.

  For any
  $\pi'\in\Pi'$ and $(s,\xi,a)\in\cS\times\Xi\times\cA$, we have
  \begin{align*}
    \frac{d^{\pi'}_h(x,a)}{\mu'_h(x,a)}
    =
    \frac{\bbQ^{\pi'}(s_h=s,\xi_h=\xi,a_h=a)}{\bbQ(\xi_h=\xi)\mu_h(s,a)}
    \leq     \frac{\max_{\pi'\in\Pi'}\bbQ^{\pi'}(s_h=s,\xi_h=\xi,a_h=a)}{\bbQ(\xi_h=\xi)\mu_h(s,a)}.
  \end{align*}
  From Propositions 3 and 4 of \citet{efroni2021provably}, we have
  $\max_{\pi'\in\Pi'}\bbQ^{\pi'}(s_h=s,\xi_h=\xi,a_h=a)=
  \bbQ(\xi_h=\xi)\cdot \max_{\pi'\in\Pi'}\bbQ^{\pi'}(s_h=s,a_h=a)=
  \bbQ(\xi_h=\xi)\cdot \max_{\pi\in\Pi}\bbP^{\pi}(s_h=s,a_h=a)$, so
  that
  \begin{align*}
    \frac{\max_{\pi'\in\Pi'}\bbQ^{\pi'}(s_h=s,\xi_h=\xi,a_h=a)}{\bbQ(\xi_h=\xi)\mu_h(s,a)}
    &=
    \frac{\bbQ(\xi_h=\xi)\max_{\pi\in\Pi}\bbP^{\pi}(s_h=s,a_h=a)}{\bbQ(\xi_h=\xi)\mu_h(s,a)} \\
& = \frac{\max_{\pi\in\Pi}\bbP^{\pi}(s_h=s,a_h=a)}{\mu_h(s,a)} \leq \Ccov(M).
  \end{align*}
  
\end{proof}

}

\section{Proofs and Additional Details from \creftitle{sec:C_gen}}
\label{app:lower}

\subsection{Additional Details: Offline RL}

\begin{proposition}[Generalized concentrability is sufficient for
  offline RL]
  \label{prop:generalized_offline}
Given access to an offline data distribution $\mu$ satisfying
generalized concentrability (\cref{def:C_gen_offline}), if $\cF$
satisfies \pref{asm:completeness}, one can find an $\veps$-optimal
policy using $\poly(\Cgen_\off(\mu,\cF),H,\log\abs{\cF},\veps^{-1})$ samples.
\end{proposition}

\begin{proof}[\pfref{prop:generalized_offline}]
Given an offline dataset $\Dcal=\crl{\cD_h}_{h=1}^{H}$ with $n$ samples for each layer
$h\in\brk{H}$ under the distribution $\mu_h$, the MSBO
algorithm~\citep[e.g.,][]{xie2020q} produces a value function
$\fhat\in\cF$ of the form
\begin{gather*}
\fhat \leftarrow \argmin_{f \in \Fcal} \sum_{h=1}^{H} \left( \Lcal_{h}(f_h,f_{h+1}) - \min_{f'_h \in \Fcal_h} \Lcal_{h}(f'_h,f_{h+1}) \right),
\\
\nonumber
\text{where \quad } \Lcal_{h}(f,f') \coloneqq \sum_{(x,a,r,x') \in \Dcal_h}\prn[\Big]{ f(x,a) - r - \max_{a' \in \Acal} f'(x',a') }^2 ,~\forall f,f' \in \Fcal.
\end{gather*}
By adapting the proof of Theorem 5 of \citet{xie2020q} (or \cref{lem:golf_concentration}), one can show that under \pref{asm:completeness}, with probability at least $1-\delta$, $\fhat$ satisfies
\begin{align*}
\sum_{h = 1}^{H}\E_{(x,a) \sim \mu_h} \left[ \left((\fhat_h(x,a) - \Tcal_h
  \fhat_{h+1})(x,a)\right)^2 \right]\leq H \cdot\frac{\log(\nicefrac{\abs{\cF}}{\delta})}{n}.
\end{align*}
  The result now follows by applying an adaptation of \citet[Corollary
  4]{xie2020q}, which shows that for any $f\in\cF$,
  \begin{align*}
    J(\pi^\star) - J(\pi_f) \leq &~ 2 \max_{\pi \in \Pi} \sum_{h = 1}^{H} \E_{(x,a) \sim d^\pi_h} \left[\left|f_h(x,a) - (\Tcal_h f_{h+1})(x,a)\right|\right]
    \\
    \leq &~ 2 \sqrt{H \max_{\pi \in \Pi} \sum_{h = 1}^{H} \E_{(x,a) \sim \mu_h} \left[(f_h(x,a) - (\Tcal_h f_{h+1})(x,a))^2\right]}
    \\
    \leq &~ 2 \sqrt{H \Cgen_\off(\mu,\cF) \sum_{h = 1}^{H} \E_{(x,a) \sim \mu_h} \left[(f_h(x,a) - (\Tcal_h f_{h+1})(x,a))^2\right]}
           \tag{by \cref{def:C_gen_offline}}
    \\
    \leq &~  2 H \sqrt{\frac{\Cgen_\off(\mu,\cF) \log(\nicefrac{\abs{\cF}}{\delta})}{n} }.
  \end{align*}
\end{proof}

\subsection{Proofs from \creftitle{sec:C_gen}}

\begin{proof}[\pfref{thm:generalized_lower}]
  Assume
  without loss of generality that $H\leq\min\crl{\log_2(X),C}$;
  if this does not hold, the result is obtained by applying the
  argument that follows with $H'=\min\crl{H,\floor{\log_2(X)},C}$.

  We consider a family of deterministic MDPs with horizon $H$. We
  use a layered state space $\cX=\cX_1\cup\cdots\cup\cX_H$, where
  only states in $\cX_h$ are reachable at layer $h$. The state space
  is a binary tree of depth $H-1$, which has
  $\sum_{h=0}^{\log_2(X)-1}2^{h}=X-1$ states. The are two actions,
  $\mathsf{left}$ and $\mathsf{right}$, which determine whether the
  next state is the left or right successor in the tree.

For each MDP in the family, we allow a single action at a single leaf at
$h = H$ to have reward $r_H = 1$, give reward $0$ to all actions in
all other states. For each such MDP, we use $(x_H^\star,a_H^\star)$ to
denote the single state-action pair with $r = 1$. We also use
$(x_h^\star,a_h^\star)$ for $h \in [H]$ to denote the unique path from
$x_1$ to $(x_H^\star,a_H^\star)$. Note that the optimal policy is to
follow this path, i.e.
  \begin{align*}
    d^{\pi^\star}_h(x,a) = \1[(x,a) = (x_h^\star,a_h^\star)].
  \end{align*}

  We choose $\Fcal_h$ to be the set of all possible indicator functions for a
  single state-action pair:
  \begin{align*}
    \Fcal_h \coloneqq \left\{ f_h(x',a') = \1(x' = x, a' = a)\mid\; \forall (x,a) \in \Xcal_h \times \Acal \right\}.
  \end{align*}
  We define $\Fcal = \Fcal_1 \times \cdots \times \Fcal_H$. Note that for each $h \in [H]$,
  \begin{align*}
    Q_h^\star(x_h,a_h) = \1(x_h = x_h^\star, a_h = a_h^\star) \in \Fcal_h.
  \end{align*}
  In addition, we have $\log\abs{\cF}\leq{}H\log(2X)$.

  \paragraph{Completeness}
We first verify that the construction satisfies completeness. Fix $f_h
\in \Fcal_h$, and let
$f_h(x,a) = \1(x = x_{f,h}, a = a_{f,h})$ for some
$(x_{f,h},a_{f,h}) \in \Xcal_h \times \Acal$. Then for any
$(x_{h-1},a_{h-1})\in\cX_{h-1}\times\cA$, we consider two
cases. First, if $x_{f,h}$ is not the unique successor of
$(x_{h-1},a_{h-1})$, then $(\Tcal_{h-1} f_h)(x_{h-1},a_{h-1})=0$. Otherwise,
\begin{align*}
  (\Tcal_{h-1} f_h)(x_{h-1},a_{h-1}) = &~ \sum_{x_h} \Pr(x_h|x_{h-1},a_{h-1})\, \max_{a_h} f(x_h,a_h)
  \\
  = &~ \Pr(x_{f,h}|x_{h-1},a_{h-1}).
      \tag{as $\max_{a_h} f(x_h,a_h) = \1(x_h = x_{f,h})$}\\
  = &~ 1.
\end{align*}
This means $\Tcal_{h-1} f_h \in \Fcal_{h-1}$, because there exists
a single $(x_{h-1},a_{h-1})$ pair in $\Xcal_{h-1} \times \Acal$ such
that $(\Tcal_{h-1} f_h)(x_{h-1},a_{h-1})\neq{}0$.

\paragraph{Generalized coverability}
We now show that the construction satisfies generalized
coverability. Fix an MDP in the family with optimal path
$\crl*{(x^{\star}_h,a^{\star}_h)}_{h=1}^{H}$. We will show that for all $f=f_{1:H}\in\cF$, if $f_{1:H} \neq
Q^\star_{1:H}$, then there exists $h' \in [H]$, such that
\begin{align}
  \label{eq:c_gen_inconsis}
  \E_{d^{\pi^\star}_{h'}}\left[\left( f_h(x_{h'},a_{h'}) - (\Tcal_{h'} f_{h'+1})(x_{h'},a_{h'})\right)^2\right] = \left( f_{h'}(x_{h'}^\star,a_{h'}^\star) - (\Tcal_{h'} f_{h'+1})(x_{h'}^\star,a_{h'}^\star)\right)^2 = 1.
\end{align}
From here, the result will follow by choosing
$\mu_h=d^{\pistar}_h\;\forall{}h\in\brk{H}$. Indeed, using the boundedness of $f_{1:H} \in \Fcal$,
we have
\begin{align*}
  \sum_{h=1}^{H} \E_{d_h^\pi}\left[ \left( f_h(x_h,a_h) - (\Tcal_h f_{h+1})(x_h,a_h)\right)^2 \right] \leq H,
\end{align*}
for all $\pi \in \Pi$, meaning that \pref{eq:c_gen_inconsis} implies
that $\Cgen_\on(\mu,\cF) \leq H\leq{}C$ in this problem instance.

We proceed to prove \pref{eq:c_gen_inconsis}. Based on the definition
of $\Fcal$, we know that
$\left( f_h(x_h^\star,a_h^\star) - (\Tcal_h
  f_{h+1})(x_h^\star,a_h^\star)\right)^2 \in \{0,1\}$ for all
$h \in [H]$. Therefore, if we assume by contradiction that
$f_{1:H}\neq\Qstar_{1:H}$ and there does not exist an $h' \in [H]$ that
satisfies \cref{eq:c_gen_inconsis}, we must have
\begin{align}
  \label{eq:c_gen_consis}
  f_{h}(x_{h}^\star,a_{h}^\star) = (\Tcal_{h} f_{h+1})(x_{h}^\star,a_{h}^\star), \quad \forall h \in [H].
\end{align}
By the condition \cref{eq:c_gen_consis}, we have
$(\cT_H f_{H+1})(x_{h}^\star,a_{H}^\star) = R_H(x_{H}^\star,a_{H}^\star) =
1$, which implies that
$f_{h}(x_{h}^\star,a_{h}^\star) = 1$ for all $h \in [H]$. From the construction of $\Fcal$, we know $Q^\star_{1:H}$ is the only function
with $Q^\star_{h}(x_{h}^\star,a_{h}^\star) = 1$ for all
$h \in [H]$, which gives the desired contraction, and proves that such
$h' \in [H]$ must exist, establishing \pref{eq:c_gen_inconsis}.
\paragraph{Lower bound on sample complexity}
A lower bound of $2^{\bigom(H)}$ samples to learn a $0.1$-optimal with
probability $0.9$ follows from standard lower bounds for binary
tree-structured MDPs \citep{krishnamurthy2016pac,jiang2017contextual} (recall that since there are $2^{H/2}$ leaves at layer $H$, and only one has 
non-zero reward, finding a policy with non-trivial regret is no easier
than solving a multi-armed bandit
problem with $2^{H/2}$ actions and binary rewards).
\end{proof}

\section{Proofs and Additional Results from \creftitle{sec:gen_bedim}}
\label{app:gen_bedim}

\iclr{
  \subsection{Additional Details: \CompMeasure versus \betext}
  \label{app:sndim_bedim}

The discussion in \pref{sec:gen_bedim} (in particular,
\pref{prop:average_bellman_lower_bound} shows that \betext and Bellman
rank fail to capture \coverability as a result of only considering
average Bellman error rather than squared Bellman error. In light of this observation, a seemingly reasonable fix is to adapt the
\betext to consider squared Bellman error rather than average Bellman
error. Consider the following variant.
\begin{definition}[\betextsq]
\label{def:be_dim_sq}
  We define the \betextsq $\bedimsq(\Fcal, \Pi, \varepsilon,h)$ for layer $h$ is the
  largest $d\in\bbN$ such that there exist sequences
  $\{d_h^\iter{1},d_h^\iter{2},\dotsc,d_h^\iter{d}\} \subseteq
  \Dset^\Pi_h$ and
  $\{\delta_h\ind{1},\ldots,\delta_h^\iter{d}\} \subseteq \Fcal_h - \Tcal_h
  \Fcal_{h+1}$ such that for all $t\in\brk{d}$,
  \begin{align}
    \label{eq:bedimsq}
    \abs{\E_{d_h^\iter{t}}[\delta_h^\iter{t}]} >
    \varepsilon^\iter{t},\mathand \sqrt{\sum_{i = 1}^{t - 1} \E_{d_h^\iter{i}}\big[(\delta_h^\iter{t})^2\big]} \leq \varepsilon^\iter{t},
  \end{align}
  for $\varepsilon\ind{1},\ldots,\veps\ind{d} \geq \varepsilon$. We define $\bedimsq(\cF,\Pi,\veps)=\max_{h\in\brk{H}}\bedimsq(\cF,\Pi,\veps,h)$.
\end{definition}
This definition is identical to \pref{def:be_dim}, except that the
constraint $\sqrt{\sum_{i = 1}^{t - 1}
  \prn{\E_{d_h^\iter{i}}[\delta_h^\iter{t}]}^2} \leq
\varepsilon^\iter{t}$ in \arxiv{\pref{eq:bedim}}\iclr{\pref{def:be_dim}} has been replaced by the
constraint $\sqrt{\sum_{i = 1}^{t - 1}
  \E_{d_h^\iter{i}}\big[(\delta_h^\iter{t})^2\big]} \leq
\varepsilon^\iter{t}$, which uses squared Bellman error instead of
average Bellman error. By adapting the analysis of
\citet{jin2021bellman} it is possible to show that this definition
yields
$\Reg\leq\bigoht\prn[\big]{H\sqrt{\inf_{\veps>0}\crl{\veps^2T+\bedimsq(\cF,\Pi,\veps)}\cdot{}T\log\abs{\cF}}}$. If
one could show that
$\dimbesq(\cF,\Pi,\veps)\approxleq{}\Ccov\cdot\polylog(\veps^{-1})$, 
this would recover \pref{thm:golf_guarantee_basic}. Unfortunately, it
turns out that in general, one can have
$\dimbesq(\cF,\Pi,\veps)=\bigom(\Ccov/\veps)$, which leads to
suboptimal $T^{2/3}$-type regret using the result above. The following
result shows that this guarantee cannot be improved without changing
the complexity measure under consideration.
\begin{proposition}
\label{prop:bedim_lower_bound}
Fix $T\in\bbN$, and let $\veps_T\ldef{}T^{-1/3}$. There exist MDP class/policy class/value function class tuples
$(\cM_1,\Pi_1,\cF_1)$ and $(\cM_2,\Pi_2,\cF_2)$ with the following
properties.
\begin{enumerate}
\item All MDPs in $\cM_1$ (resp. $\cM_2$) satisfy
  \pref{asm:completeness} with respect to $\cF_1$ (resp. $\cF_2$). In
  addition, $\log\abs{\cF_1}=\log\abs{\cF_2}=\bigoht(1)$.
\item For all MDPs in $\cM_1$, we have
  $\dimbesq(\cF_1,\Pi_1,\veps_T)\propto{}1/\veps_T$, and any algorithm
  must have $\En\brk*{\Reg}\geq{}\bigom(T^{2/3})$ for some MDP in the class
\item For all MDPs in $\cM_2$, we also have
  $\dimbesq(\cF_2,\Pi_2,\veps_T)\propto{}1/\veps_T$, yet
  $\Ccov=\bigoh(1)$ and \golf attains $\En\brk*{\Reg}\leq\bigoht(\sqrt{T})$.
\end{enumerate}
\end{proposition}
This result shows that there are two classes for which the optimal
rate differs polynomially ($\bigom(T^{2/3})$ vs. $\bigoht(\sqrt{T})$),
yet the \betext has the same size, and implies that the \betext
cannot provide rates better
than $\bigom(T^{2/3})$ for classes with low coverability in general. Informally, the reason why \betext fails capture the optimal rates for
the problem instances in \pref{prop:bedim_lower_bound} is that the
definition in \pref{eq:bedimsq} only checks whether the average Bellman
error violates the threshold $\veps$, and does not consider how far
the error violates the threshold (e.g.,
$\abs{\E_{d_h^\iter{t}}[\delta_h^\iter{t}]}>\veps$ and
$\abs{\E_{d_h^\iter{t}}[\delta_h^\iter{t}]}>1$ are counted the same).

In spite of this counterexample, it is possible to show that the
\betext with squared Bellman error is always bounded by the
\CompMeasure up to a $\poly(\veps^{-1})$ factor, and hence can always
be bounded by \coverability, albeit suboptimally.

\begin{proposition}
  \label{prop:sqbedim_leq_snc}
  Let $\cF$ be a $\brk{0,1}$-valued function class. For all $T\in\bbN$ and $\veps>0$, we have $\min \{\bedimsq(\Fcal,\Pi,\veps), T\} \leq  \frac{\cdimrl(\Fcal,\Pi,T)}{\varepsilon^2}$. 
\end{proposition}

\subsubsection{Proofs from Additional Details}

\begin{proof}[\pfref{prop:bedim_lower_bound}]
  Let the time horizon $T\in\bbN$ be fixed. We first construct the class $\cM_1$ and verify that it satisfies the
properties in the statement of \pref{prop:bedim_lower_bound}, then do the same for $\cM_2$.  

\paragraph{Class $\cM_1$}
We choose $\cM_1$ to be a class of bandit problems with $H=1$. Let a
parameter $\veps_1\in\brk{0,1/2}$ be fixed, and let
$A\ldef{}\veps_1^{-1}$. We define
$\cM_1=\crl{M\ind{1},\ldots,M\ind{A}}$, where for each $M\ind{i}$:
  \begin{itemize}
  \item The action space is $\cA=\crl{1,\ldots,A}$.
  \item The reward distribution for action $a\in\cA$ in state $x_1$ is $\Ber(\nicefrac{1}{2}+\veps_1\indic\crl{a=i})$.
  \end{itemize}
For each $i\in\cM\ind{i}$, the mean reward function is
$f_1\ind{i}(x_1,\pi)=\nicefrac{1}{2}+\veps_1\indic\crl{a=i}$. We define
$\cF=\crl*{f\ind{i}}_{i=1}^{A}$ and $\Pi=\crl{\pi_f\mid{}f\in\cF}$. Note that since $H=1$,
completeness of $\cF$ is immediate.

\noindent\emph{Lower bounding the \betext.}~~Let $M\ind{A}$ be the underlying instance. We will lower bound the \betext
for layer $h=1$. Consider the sequence $d_1\ind{1},\ldots,d_1\ind{A-1}$,
where $d_1\ind{t}\ldef{}d_1^{\pi_{f\ind{t}}}$ and
$\delta_1\ind{1},\ldots,\delta_1\ind{A-1}$, where
$\delta_1\ind{t}\ldef{}f_1\ind{t}-\cT_1f_{2}\ind{t}=f_1\ind{t}-f_1\ind{A}$
(recall that we adopt the convention $f_{H+1}=0$). Observe that for
each $t\in\brk{A-1}$, we have
\[
  \abs{\En_{d_1\ind{t}}\brk{\delta_1\ind{t}}(x_1,a_1)}
  =\abs{f_1\ind{t}(x_1,t)-f_1\ind{A}(x_1,t)}=\veps_1,
\]
yet
\[
  \sum_{i<t}\En_{d_1\ind{i}}\brk*{(\delta_1\ind{t}(x_1,a_1))^2}
  = \veps_1\sum_{i<t}(f_1\ind{t}(x_1,i)-f_1\ind{A}(x_1,i))^2=0.
\]
This certifies that
$\dimbesq(\cF,\Pi,\veps)\geq{}A-1\geq\veps_1^{-1}/2$ for all $\veps<\veps_1$.

\noindent\emph{Lower bounding regret.}
A standard result \citep[e.g.,][]{lattimore2020bandit} is that for any family of multi-armed bandit
instances of the form $\crl{M\ind{1},\ldots,M\ind{A}}$, where
$M\ind{i}$ has Bernoulli rewards with mean
$\nicefrac{1}{2}+\Delta\indic\crl{a=i}$ for $\Delta\leq{}1/4$, any algorithm must have
regret
\[
\En\brk{\Reg} \geq \bigom(1)\cdot\min\crl*{\Delta{}T,\frac{A}{\Delta}}
\]
for some instance. We apply this result with the class $\cM_1$, which
has $\Delta=\veps_1$ and $A=\veps_1^{-1}$, which gives
\[
\En\brk{\Reg} \geq \bigom(1)\cdot\min\crl*{\veps_1T,\frac{1}{\veps_1^2}}.
\]
Choosing $\veps_1=\veps_T=T^{-1/3}$ yields
$\En\brk{\Reg}\geq{}\bigom(T^{2/3})$ whenever $T$ is greater than an
absolute constant.

\paragraph{Class $\cM_2$} Let a parameter $\veps_2\in\brk{0,1/2}$ be
fixed, and let $A\ldef{}\veps_2^{-1}$ (we assume without loss of
generality that $\veps_2^{-1}\in\bbN$). We define
$\cM_2=\crl{M\ind{1},\ldots,M\ind{A}}$, where each MDP $M\ind{i}$ is
as defined follows:
\begin{itemize}
\item We have $H=2$, and there is a layered state space $\cX=\cX_1\times\cX_2$, where
  $\cX_1=\crl{x_1}$ and $\cX_2=\crl{y,z}$.
\item The action space is $\cA=\crl{1,\ldots,A}$.
\item $x_1$ is the deterministic initial state. Regardless of the
  action, we transition to $z$ with probability $1-\veps_2$ and $y$
  with probability $\veps_2$.
\item For each MDP $M\ind{i}$ all actions have zero reward in states
  $x_1$ and $z$. For state $y$, action $i$ has reward $1$ and all
  other actions have reward $0$.
\end{itemize}
We let $f\ind{i}$ denote the optimal $Q$-function for $M\ind{i}$, which
has:
\begin{itemize}
\item $f\ind{i}_1(x_1,\cdot)=\veps_2$ and $f\ind{i}_2(z,\cdot)=0$.
\item $f\ind{i}_2(y,a)=\indic\crl{a=i}$.
\end{itemize}
We define $\cF=\crl*{f\ind{i}}_{i\in\brk{A}}$; it is clear that this
class satisfies completeness. We define
$\Pi=\crl{\pi_f\mid{}f\in\cF}$; for states where there are multiple
optimal actions
(i.e., $f_h(x,a)=f_h(x,a')$), we take $\pi_{f,h}(x)$ to be the optimal
with the least index, which implies that
$\pi_{f,1}(x_1)=\pi_{f,2}(z)=1$ for all $f\in\cF$.

\noindent\emph{Verifying coverability.}~~
We choose $\mu_1(x,a)=\indic\crl{x=x_1,a=1}$. We choose
$\mu_2(z,1)=\frac{1}{2}$ and $\mu_2(y,a)=\frac{1}{2A}$ for all
$a\in\cA$. It is immediate that coverability is satisfied with
constant $1$ for $h=1$. For $h=2$, we have that for all $\pi\in\Pi$,
\[
\frac{d^{\pi}_2(z,1)}{\mu_2(z,1)} = \frac{1-\veps_2}{\nicefrac{1}{2}}
\leq 2
\]
and
\[
\frac{d^{\pi}_2(y,a)}{\mu_2(y,a)} \leq \frac{\veps_2}{\mu_2(y,a)}\leq{}2A\veps_2\leq{}2.
\]
Hence, we have $\Ccov \leq{} 2$; note that this holds for any choice
of $\veps_2$.

\noindent\emph{Lower bounding the \betext.}~~Let $M\ind{A}$ be the underlying MDP. We will lower bound the \betext
for layer $h=2$. Consider the sequence $d_2\ind{1},\ldots,d_2\ind{A-1}$,
where $d_2\ind{t}\ldef{}d_2^{\pi_{f\ind{t}}}$ and
$\delta_2\ind{1},\ldots,\delta_2\ind{A-1}$, where
$\delta_2\ind{t}\ldef{}f_2\ind{t}-\cT_2f_{3}\ind{t}=f_2\ind{t}-f_2\ind{A}$
(recall that we adopt the convention $f_{H+1}=0$). Observe that for
each $t\in\brk{A-1}$, we have
\[
  \abs{\En_{d_2\ind{t}}\brk{\delta_2\ind{t}}(x_2,a_2)}
  =\veps_2\abs{f_2\ind{t}(y,t)-f_2\ind{A}(y,t)}=\veps_2,
\]
yet
\[
  \sum_{i<t}\En_{d_2\ind{i}}\brk*{(\delta_2\ind{t}(x_2,a_2))^2}
  = \veps_2\sum_{i<t}(f_2\ind{t}(y,i)-f_2\ind{A}(y,i))^2=0.
\]
This certifies that $\dimbesq(\cF,\Pi,\veps,2)\geq{}A-1\geq\veps_2^{-1}/2$
for all $\veps<\veps_2$.

\noindent\emph{Upper bound on regret.}~
To conclude, we set $\veps_2=\veps_T=1/T^{-1/3}$. With this choice, we
have $\bedimsq(\cF_2,\Pi_2,\veps_T)\geq{}\bigom(\veps_T^{-1})$. Since the construction satisfies completeness
(\pref{asm:completeness}) and has $\Ccov\leq{}2$ and $H=2$,
\pref{thm:golf_guarantee_basic} yields
\[
  \Reg\leq\bigoh(\sqrt{T\log(\abs{\cF}T/\delta)}) = \bigoh(\sqrt{T\log(T/(\veps_2\delta))})=\bigoht(\sqrt{T\log(1/\delta)}).
\]

\end{proof}

\begin{proof}[\cpfname{prop:sqbedim_leq_snc}]
Fix $h \in [H]$, and $n\in\bbN$, and consider sequences
$\{d_h^\iter{1},d_h^\iter{2},\dotsc,d_h^\iter{n}\}$ and
$\{\delta_h^\iter{1},\delta_h^\iter{2},\dotsc,\delta_h^\iter{n}\}$
that satisfy \pref{eq:bedimsq} (that is, the sequences witness the
value of $\bedim(\Fcal,\Pi,\veps,h)$). Then
\begin{align*}
\bedimsq(\Fcal,\Pi,\veps,h) \leq &~ \sum_{t = 1}^{n} \frac{\E_{d_h^\iter{t}}\left[\delta_h^\iter{t}\right]^2}{\left( \varepsilon^\iter{t} \right)^2}
\\
\leq &~ \sum_{t = 1}^{n} \left( 1 + \left( \varepsilon^\iter{t} \right)^2 \right) \cdot \frac{\E_{d_h^\iter{t}}\left[\delta_h^\iter{t}\right]^2}{\left( \varepsilon^\iter{t} \right)^2 \left(1 + \sum_{i = 1}^{t - 1} \E_{d_h^\iter{i}}[\delta_h^\iter{t}]^2\right)}
\tag{by $\sum_{i = 1}^{t - 1} \E_{d_h^\iter{i}}[\delta_h^\iter{t}]^2 \leq ( \varepsilon^\iter{t} )^2$}
\\
\leq &~ \sum_{t = 1}^{n} \frac{1 + \left( \varepsilon^\iter{t} \right)^2}{\left( \varepsilon^\iter{t} \right)^2} \frac{\E_{d_h^\iter{t}}\left[\delta_h^\iter{t}\right]^2}{1 + \sum_{i = 1}^{t - 1} \E_{d_h^\iter{i}}[\delta_h^\iter{t}]^2}
\\
\leq &~ \sum_{t = 1}^{n} \frac{2}{\left( \varepsilon^\iter{t} \right)^2} \frac{\E_{d_h^\iter{t}}\left[\delta_h^\iter{t}\right]^2}{1 + \sum_{i = 1}^{t - 1} \E_{d_h^\iter{i}}[\delta_h^\iter{t}]^2}
\tag{by $\varepsilon^\iter{t} < | \E_{d_h^\iter{t}}[\delta_h^\iter{t}] | \leq 1$}
\\
\leq &~ \frac{1}{\varepsilon^2} \sum_{t = 1}^{n} \frac{\E_{d_h^\iter{t}}\left[\delta_h^\iter{t}\right]^2}{1 \vee \sum_{i = 1}^{t - 1} \E_{d_h^\iter{i}}[\delta_h^\iter{t}]^2}
\tag{by $\varepsilon^\iter{t} \geq \varepsilon$}
\\
\leq &~ \frac{\cdimrl(\Fcal,\Pi,n)}{\varepsilon^2}.
\end{align*}
This implies for any $T > 0$,
\begin{align*}
\min \{\bedimsq(\Fcal,\Pi,\veps), T\} \leq  \frac{\cdimrl(\Fcal,\Pi,T)}{\varepsilon^2}.
\end{align*}
\end{proof}
    }

\subsection{Proofs from \creftitle{sec:gen_bedim}}
\label{sec:genbedim_proof}

\begin{proof}[\pfref{prop:average_bellman_lower_bound}]
  We present a counterexample for both $Q$-type and $V$-type
  \betext. We recall that the $V$-type \betext is defined by replacing
  $\Fcal_h - \Tcal_h \Fcal_{h+1}$ with $ V_{\Fcal_h - \Tcal_h
    \Fcal_{h+1}}$ and $\Dset_{h}^\Pi$ with $\Dset_{h,x}^\Pi$ in
  \cref{def:be_dim}, where $V_{\Fcal_h - \Tcal_h \Fcal_{h+1}}
  \coloneqq \{(f_h - \Tcal_h f_{h+1})(\cdot,\pi_{f,h}): f \in \Fcal\}
  \subset (\Xcal \to \RR)$ and $\Dset_{h,x}^\Pi \coloneqq \{d^\pi_h(\cdot): \pi \in \Pi\} \subset \Delta(\Xcal)$; see
\pref{app:v_type} or \citet{jin2021bellman} for more background on $V$-type \betext.

\paragraph{$V$-type \betext}
The hard instance for $V$-type \betext is based on the construction
of~\citet[Proposition B.1]{efroni2022sample}, which shows that for any
$d = 2^i$ ($i \in \NN$), there exists an exogenous MDP (ExoMDP) with
$|\Scal| = 3$ endogenous states, $|\Acal| = 2$, $H = 2$, and $d$
exogenous factors, with the following
properties:\footnote{Technically, the construction in
  \citet{efroni2022sample} has a stochastic initial state with known
  distribution. This can be embedded in our framework, which has a
  deterministic initial state, by lifting the horizon from $2$ to $3$.}
\begin{enumerate}
  \item There exists a function class $\Fcal$ such that $Q^\star \in \Fcal$
    and $|\Fcal| = d$. In addition for all $f\in\cF$ with
    $f\neq\Qstar$, $\pi_f$ is $1/8$-suboptimal.
  \item For all $f,f' \in \Fcal \setminus Q^\star$, we have (note that $H=2$)
\begin{align}
\label{eq:BEdim_lb_vtype}
\E_{x \sim d_2^{\pi_{f'}}, a \sim \pi_{f,2}}\left[f_2(x,a) - R_2(x,a) \right] = \frac{1}{2} \1\crl{f = f'}.
\end{align}
\item $\Ccov\leq{}6$; this is a consequence of \pref{prop:ex_bmdp} and
  the fact that the ExoMDP model in \citet{efroni2022sample} is a
  special case of the \exbmdp model in \pref{sec:exbmdp}.
\end{enumerate}

This means that if we take
$\{f^\iter{1},f^\iter{2},\dotsc,f^\iter{d-1}\}$ to be any ordering of
the set of functions in $\cF\setminus\crl{\Qstar}$, then set
$\delta_h^\iter{i} \coloneqq f_h^\iter{i} - \Tcal_h f_{h+1}^\iter{i}$
and $d_h\ind{t}\ldef d_h^{\pi_{f\ind{t}}}$, we have that for all $t\in\brk{d-1}$,
\begin{align*}
    \left|\E_{x \sim d_2^\iter{t}, a \sim \pi_{f^\iter{t},2}}[\delta_2^\iter{t}]\right| = \frac{1}{2},\mathand \sqrt{\sum_{i = 1}^{t - 1} \prn[\big]{\E_{x \sim d_2^\iter{i}, a \sim \pi_{f^\iter{t},2}}[\delta_2^\iter{t}]}^2} =0.
\end{align*}
This implies that the $V$-type \betext
$\bedimv(\Fcal,\Pi_\Fcal,\varepsilon)$ is at least $d-1$ for all $\varepsilon \leq \nicefrac{1}{2}$.
It is straightforward to verify that this construction in
\citet{efroni2022sample} satisfies \cref{asm:completeness}
(completeness), because functions in the class have $f_1 = \Tcal_2 f_2$ (that is, zero Bellman
error at $h=1$). As a result, since $H=2$, completeness for this construction is
implied by $\Qstar\in\cF$.

\paragraph{$Q$-type \betext}
The construction above immediately extends to $Q$-type. This is
because in the construction, the value of $R_2(x,\cdot)$ and
$f_2(x,\cdot)$ depends only on $x$ (i.e., is independent of the action) for all $f\in\cF$ (cf.~\citealp[Proposition B.1]{efroni2022sample}). Therefore, for any $f,g \in \Fcal$, we have,
\begin{align}
\label{eq:BEdim_lb_qtype}
\E_{x \sim d_2^{\pi_{f}}, a \sim \pi_{g,2}}[g_2(x,a) - R_2(x,a)] = \E_{(x,a) \sim d_2^{\pi_{f}}}[g_2(x,a) - R_2(x,a)].
\end{align}
This implies that the $Q$-type Bellman residual matrix
\[
  \crl*{\E_{(x,a) \sim d_2^{\pi_{f'}}}[f_2(x,a) - R_2(x,a)]}_{f,f'\in\cF\setminus\crl{\Qstar}}
\]
embeds the scaled identity matrix and, via the same argument as for $V$-type
above, immediately implies that
$\bedim(\cF,\Pi,\varepsilon)\geq{}d-1$ for all $\varepsilon \leq
\nicefrac{1}{2}$. As before, we have $\Ccov\leq{}6$, and $\cF$ is complete.
\end{proof}

\begin{proposition}
\label{prop:lb_olive}
For any $d\in\bbN$, there exists an MDP $M$ with $H=2$ and
$\abs{\cA}=2$, a policy class $\Pi$ with
  $\abs{\Pi}=d$, and a value function class $\Fcal$ with $\abs{\cF}=d$ satisfying completeness, such that $C_\on = \Ocal(1)$, yet \olive~\citep{jiang2017contextual} requires at least $\Omega(d)$ trajectories to return a $0.1$-optimal policy.
\end{proposition}
\begin{proof}[\cpfname{prop:lb_olive}]
We now show that that \olive, a canonical average-Bellman-error-based
hypothesis elimination algorithm, also suffers from the lower bound in
the construction from \cref{prop:average_bellman_lower_bound}. By
\cref{eq:BEdim_lb_vtype} (V-type \olive) and \pref{eq:BEdim_lb_qtype}
(Q-type \olive), we know that any sub-optimal hypothesis $f
\in \Fcal \setminus Q^\star$ cannot be eliminated until $\pi_f$ is
executed. On the other hand, the construction ensures $\E[\max_a
f(s_1,a)] = \nicefrac{7}{8}$ whereas $J(\pi^\star) =
\nicefrac{3}{4}$. This means \olive will enumerate over $\Fcal
\setminus Q^\star$ before finding a $0.1$-optimal policy for this
instance, and hence suffers from complexity of $\Omega(d)$ ($|\Fcal| =
d$).
\end{proof}

\arxiv{
\begin{proof}[\pfref{prop:bedim_lower_bound}]
  Let the time horizon $T\in\bbN$ be fixed. We first construct the class $\cM_1$ and verify that it satisfies the
properties in the statement of \pref{prop:bedim_lower_bound}, then do the same for $\cM_2$.  

\paragraph{Class $\cM_1$}
We choose $\cM_1$ to be a class of bandit problems with $H=1$. Let a
parameter $\veps_1\in\brk{0,1/2}$ be fixed, and let
$A\ldef{}\veps_1^{-1}$. We define
$\cM_1=\crl{M\ind{1},\ldots,M\ind{A}}$, where for each $M\ind{i}$:
  \begin{itemize}
  \item The action space is $\cA=\crl{1,\ldots,A}$.
  \item The reward distribution for action $a\in\cA$ in state $x_1$ is $\Ber(\nicefrac{1}{2}+\veps_1\indic\crl{a=i})$.
  \end{itemize}
For each $i\in\cM\ind{i}$, the mean reward function is
$f_1\ind{i}(x_1,\pi)=\nicefrac{1}{2}+\veps_1\indic\crl{a=i}$. We define
$\cF=\crl*{f\ind{i}}_{i=1}^{A}$ and $\Pi=\crl{\pi_f\mid{}f\in\cF}$. Note that since $H=1$,
completeness of $\cF$ is immediate.

\noindent\emph{Lower bounding the \betext.}~~Let $M\ind{A}$ be the underlying instance. We will lower bound the \betext
for layer $h=1$. Consider the sequence $d_1\ind{1},\ldots,d_1\ind{A-1}$,
where $d_1\ind{t}\ldef{}d_1^{\pi_{f\ind{t}}}$ and
$\delta_1\ind{1},\ldots,\delta_1\ind{A-1}$, where
$\delta_1\ind{t}\ldef{}f_1\ind{t}-\cT_1f_{2}\ind{t}=f_1\ind{t}-f_1\ind{A}$
(recall that we adopt the convention $f_{H+1}=0$). Observe that for
each $t\in\brk{A-1}$, we have
\[
  \abs{\En_{d_1\ind{t}}\brk{\delta_1\ind{t}}(x_1,a_1)}
  =\abs{f_1\ind{t}(x_1,t)-f_1\ind{A}(x_1,t)}=\veps_1,
\]
yet
\[
  \sum_{i<t}\En_{d_1\ind{i}}\brk*{(\delta_1\ind{t}(x_1,a_1))^2}
  = \veps_1\sum_{i<t}(f_1\ind{t}(x_1,i)-f_1\ind{A}(x_1,i))^2=0.
\]
This certifies that
$\dimbesq(\cF,\Pi,\veps)\geq{}A-1\geq\veps_1^{-1}/2$ for all $\veps<\veps_1$.

\noindent\emph{Lower bounding regret.}
A standard result (e.g.,
\citet{lattimore2020bandit}) is that for any family of multi-armed bandit
instances of the form $\crl{M\ind{1},\ldots,M\ind{A}}$, where
$M\ind{i}$ has Bernoulli rewards with mean
$\nicefrac{1}{2}+\Delta\indic\crl{a=i}$ for $\Delta\leq{}1/4$, any algorithm must have
regret
\[
\En\brk{\Reg} \geq \bigom(1)\cdot\min\crl*{\Delta{}T,\frac{A}{\Delta}}
\]
for some instance. We apply this result with the class $\cM_1$, which
has $\Delta=\veps_1$ and $A=\veps_1^{-1}$, which gives
\[
\En\brk{\Reg} \geq \bigom(1)\cdot\min\crl*{\veps_1T,\frac{1}{\veps_1^2}}.
\]
Choosing $\veps_1=\veps_T=T^{-1/3}$ yields
$\En\brk{\Reg}\geq{}\bigom(T^{2/3})$ whenever $T$ is greater than an
absolute constant.

\paragraph{Class $\cM_2$} Let a parameter $\veps_2\in\brk{0,1/2}$ be
fixed, and let $A\ldef{}\veps_2^{-1}$ (we assume without loss of
generality that $\veps_2^{-1}\in\bbN$). We define
$\cM_2=\crl{M\ind{1},\ldots,M\ind{A}}$, where each MDP $M\ind{i}$ is
as defined follows:
\begin{itemize}
\item We have $H=2$, and there is a layered state space $\cX=\cX_1\times\cX_2$, where
  $\cX_1=\crl{x_1}$ and $\cX_2=\crl{y,z}$.
\item The action space is $\cA=\crl{1,\ldots,A}$.
\item $x_1$ is the deterministic initial state. Regardless of the
  action, we transition to $z$ with probability $1-\veps_2$ and $y$
  with probability $\veps_2$.
\item For each MDP $M\ind{i}$ all actions have zero reward in states
  $x_1$ and $z$. For state $y$, action $i$ has reward $1$ and all
  other actions have reward $0$.
\end{itemize}
We let $f\ind{i}$ denote the optimal $Q$-function for $M\ind{i}$, which
has:
\begin{itemize}
\item $f\ind{i}_1(x_1,\cdot)=\veps_2$ and $f\ind{i}_2(z,\cdot)=0$.
\item $f\ind{i}_2(y,a)=\indic\crl{a=i}$.
\end{itemize}
We define $\cF=\crl*{f\ind{i}}_{i\in\brk{A}}$; it is clear that this
class satisfies completeness. We define
$\Pi=\crl{\pi_f\mid{}f\in\cF}$; for states where there are multiple
optimal actions
(i.e., $f_h(x,a)=f_h(x,a')$), we take $\pi_{f,h}(x)$ to be the optimal
with the least index, which implies that
$\pi_{f,1}(x_1)=\pi_{f,2}(z)=1$ for all $f\in\cF$.

\noindent\emph{Verifying coverability.}~~
We choose $\mu_1(x,a)=\indic\crl{x=x_1,a=1}$. We choose
$\mu_2(z,1)=\frac{1}{2}$ and $\mu_2(y,a)=\frac{1}{2A}$ for all
$a\in\cA$. It is immediate that coverability is satisfied with
constant $1$ for $h=1$. For $h=2$, we have that for all $\pi\in\Pi$,
\[
\frac{d^{\pi}_2(z,1)}{\mu_2(z,1)} = \frac{1-\veps_2}{\nicefrac{1}{2}}
\leq 2
\]
and
\[
\frac{d^{\pi}_2(y,a)}{\mu_2(y,a)} \leq \frac{\veps_2}{\mu_2(y,a)}\leq{}2A\veps_2\leq{}2.
\]
Hence, we have $\Ccov \leq{} 2$; note that this holds for any choice
of $\veps_2$.

\noindent\emph{Lower bounding the \betext.}~~Let $M\ind{A}$ be the underlying MDP. We will lower bound the \betext
for layer $h=2$. Consider the sequence $d_2\ind{1},\ldots,d_2\ind{A-1}$,
where $d_2\ind{t}\ldef{}d_2^{\pi_{f\ind{t}}}$ and
$\delta_2\ind{1},\ldots,\delta_2\ind{A-1}$, where
$\delta_2\ind{t}\ldef{}f_2\ind{t}-\cT_2f_{3}\ind{t}=f_2\ind{t}-f_2\ind{A}$
(recall that we adopt the convention $f_{H+1}=0$). Observe that for
each $t\in\brk{A-1}$, we have
\[
  \abs{\En_{d_2\ind{t}}\brk{\delta_2\ind{t}}(x_2,a_2)}
  =\veps_2\abs{f_2\ind{t}(y,t)-f_2\ind{A}(y,t)}=\veps_2,
\]
yet
\[
  \sum_{i<t}\En_{d_2\ind{i}}\brk*{(\delta_2\ind{t}(x_2,a_2))^2}
  = \veps_2\sum_{i<t}(f_2\ind{t}(y,i)-f_2\ind{A}(y,i))^2=0.
\]
This certifies that $\dimbesq(\cF,\Pi,\veps,2)\geq{}A-1\geq\veps_2^{-1}/2$
for all $\veps<\veps_2$.

\noindent\emph{Upper bound on regret.}~
To conclude, we set $\veps_2=\veps_T=1/T^{-1/3}$. With this choice, we
have $\bedimsq(\cF_2,\Pi_2,\veps_T)\geq{}\bigom(\veps_T^{-1})$. Since the construction satisfies completeness
(\pref{asm:completeness}) and has $\Ccov\leq{}2$ and $H=2$,
\pref{thm:golf_guarantee_basic} yields
\[
  \Reg\leq\bigoh(\sqrt{T\log(\abs{\cF}T/\delta)}) = \bigoh(\sqrt{T\log(T/(\veps_2\delta))})=\bigoht(\sqrt{T\log(1/\delta)}).
\]

\end{proof}
}

\begin{proof}[\pfref{thm:regret_online_c}]
  As in \pref{thm:golf_guarantee_basic}, as a consequence of
  completeness (\pref{asm:completeness}), the construction of
  $\cF\ind{t}$, and \pref{lem:golf_concentration}, we have that with probability at least $1-\delta$, for all $t\in\brk{T}$:
\begin{align*}
\mathrm{(i)}\;\;\Qstar\in\cF\ind{t},\mathand\mathrm{(ii)}\;\;\sum_{x,a} \dtilde_h^\iter{t}(x,a) \left(\delta_h^\iter{t}(x,a)\right)^2\leq\bigoh(\beta),%
\end{align*}
and whenever this event holds, 
\begin{align*}
  \Reg \leq{} \sum_{t = 1}^{T} \left( f_1^\iter{t}(x_1,\pi_{f\ind{t}_1,1}(x_1)) -J(\pi\ind{t}) \right) 
  = \sum_{t = 1}^{T} \sum_{h=1}^{H}\E_{(x,a)\sim{}d_h^\iter{t}}\big[\underbrace{f_h\ind{t}(x,a)-(\cT_hf_{h+1}\ind{t})(x,a)}_{\rdef\delta_h^\iter{t}(x,a)}\big].
\end{align*}
To proceed, we have that for all $h\in\brk{H}$, 
  \begin{align*}
    \sum_{t = 1}^{T} \E_{d_h^\iter{t}}\left[\delta_h^\iter{t}\right] = &~ \sum_{t = 1}^{T} \left( \E_{d_h^\iter{t}}\left[\delta_h^\iter{t}\right] \right) \left( \frac{1 \vee \sum_{i = 1}^{t - 1} \E_{d_h^\iter{i}}\big[(\delta_h^\iter{t})^2\big]}{1 \vee \sum_{i = 1}^{t - 1} \E_{d_h^\iter{i}}\left[(\delta_h^\iter{t})^2\right]} \right)^{\nicefrac{1}{2}}
    \\
    \leq &~ \sqrt{\sum_{t = 1}^{T} \frac{\E_{d_h^\iter{t}}\left[\delta_h^\iter{t}\right]^2}{1 \vee \sum_{i = 1}^{t - 1} \E_{d_h^\iter{i}}\left[(\delta_h^\iter{t})^2\right]}} \sqrt{\sum_{t = 1}^{T} \left( 1 \vee \sum_{i = 1}^{t - 1} \E_{d_h^\iter{i}}\left[(\delta_h^\iter{t})^2\right] \right)}
           \tag{by Cauchy-Schwarz inequality}
    \\
    \leq &~ \sqrt{\sum_{t = 1}^{T} \frac{\E_{d_h^\iter{t}}\left[\delta_h^\iter{t}\right]^2}{1 \vee \sum_{i = 1}^{t - 1} \E_{d_h^\iter{i}}\left[(\delta_h^\iter{t})^2\right]}} \sqrt{\beta{}T}
    \\
    \leq &~ \sqrt{\cdimrl(\Fcal,\Pi,T)\cdot{} \beta{}T}.
    \tag{by \cref{def:snc_qtype_rf}}
  \end{align*}
  Therefore, we obtain
  \begin{align*}
    \Reg \leq H \sqrt{\cdimrl(\Fcal,\Pi,T)\cdot \beta{}T}.
  \end{align*}
Plugging in the choice for $\beta$ completes the proof.
\end{proof}

\begin{proof}[\pfref{prop:sedim_coverability}]
We prove a more general result. Consider a set of
distributions $\Dset \subset \Delta(\Zcal)$, and a set of test
functions $\Psi \subset (\Zcal \to [0,1])$. We define a
generalized form of coverability with respect to $\Dset$ by
\begin{align*}
  \Ccov(\Dset) \ldef \inf_{\mu \in \Delta(\cZ)}  \sup_{d \in \Dset}\,
  \nrm*{\frac{d}{\mu}}_{\infty}.
\end{align*}
We will show that, for any $T > 0$,
\begin{align*}
    \cdim(\Psi, \Dcal, T) \lesssim C_\on(\Dset) \log(T),
\end{align*}
which is implies \cref{prop:sedim_coverability}.

Going forward, we fix an arbitrary sequence $\{d^\iter{1},d^\iter{2},\dotsc,d^\iter{T}\} \subset \Dset$ as well as an arbitrary sequence of $\{\psi^\iter{1},\psi^\iter{2},\dotsc,\psi^\iter{T}\} \subset \Psi$.
Following \cref{eq:def_dbar}, we define
\begin{align}
\label{eq:def_dbar_app}
  \mu^\star \coloneqq ~ \argmin_{\mu\in\Delta(\cZ)} \sup_{d\in\Dset}\, \nrm*{\frac{d}{\mu}}_{\infty}.
\end{align}
In addition, define $\dtil\ind{t}=\sum_{i<t}d\ind{t}$.

  For each $z \in \Zcal$, let
  \begin{align}
    \label{eq:def_tau_selfnorm}
    \tau(z) \ldef \min\left\{t \midmid \sum_{i = 1}^{t-1} d^\iter{i}(z) \geq \Ccov\mu^{\star}(z)\right\}.
  \end{align}
  We decompose $\E_{d^\iter{t}}[\psi\ind{t}]$ as
  \begin{align*}
    \E_{d^\iter{t}}\big[\psi^\iter{t}\big] = \E_{d^\iter{t}}\big[\psi^\iter{t}(z) \1[t < \tau(z)]\big] + \E_{d^\iter{t}}\big[\psi^\iter{t}(z) \1[t \geq \tau(z)]\big].
  \end{align*}
  Then,
  \begin{equation}
    \label{eq:snorm_C_terms}
    \begin{aligned}
      &~ \sum_{t = 1}^{T}
      \frac{\E_{d^\iter{t}}\big[\psi^\iter{t}\big]^2}{1 \vee \sum_{i =
          1}^{t - 1} \E_{d^\iter{i}}\big[(\psi^\iter{t})^2\big]}
      \\
      \lesssim &~ \underbrace{\sum_{t = 1}^{T} \frac{\E_{d^\iter{t}}\big[\psi^\iter{t}(z) \1[t < \tau(z)]\big]^2}{1 \vee \sum_{i = 1}^{t -
            1} \E_{d^\iter{i}}\big[(\psi^\iter{t})^2\big]}}_{\text{(I)}}
      + \underbrace{\sum_{t = 1}^{T} \frac{\E_{d^\iter{t}}\big[\psi^\iter{t}(z) \1[t \geq \tau(z)]\big]^2}{1 \vee \sum_{i = 1}^{t - 1}
          \E_{d^\iter{i}}\big[(\psi^\iter{t})^2\big]}}_{\text{(II)}},
    \end{aligned}
  \end{equation}
  where we use $a \lesssim b$ as shorthand for $a \leq \Ocal(b)$.
  
  We first bound the term (I),
  \begin{align}
    \nonumber
    \text{(I)} \leq &~ \sum_{t = 1}^{T} \E_{d^\iter{t}}\big[\psi^\iter{t}(z) \1[t < \tau(z)]\big]^2
    \\
    \leq &~ \sum_{t = 1}^{T} \E_{d^\iter{t}}\big[\1[t < \tau(z)]\big]^2
           \tag{by $\psi(\cdot) \in [0,1]$, $\forall \psi \in \Psi$}
    \\
    \leq &~ \sum_{t = 1}^{T} \E_{d^\iter{t}}\big[\1[t < \tau(z)]\big]
           \tag{by $\E_{d^\iter{t}}\big[\1[t < \tau(z)]\big] \leq 1$}
    \\
    \nonumber
    = &~ \sum_{z \in \Zcal} \sum_{t = 1}^{T} d_h^\iter{t}(z) \1[t < \tau(z)]
    \\
    \nonumber
    = &~ \sum_{z \in \Zcal} \left( \dtilde^\iter{\tau(z) - 1}(z) + d^\iter{\tau(z) - 1}(z)\right)
    \\
    \nonumber
    \overset{\text{(a)}}{\leq} &~ \sum_{z \in \Zcal} 2 \Ccov(\Dset)\mu^{\star}(z)
    \\
    \label{eq:snorm_C_term1}
    \leq &~ \Ccov(\Dset),
  \end{align}
  where (a) follows because
  $\dtilde^\iter{\tau(z) - 1}(z), d^\iter{\tau(z) -
    1}(z) \leq \Ccov(\Dset)\mu^{\star}(z)$, for all
  $z \in \Zcal$, as a consequence of 
  \cref{eq:def_tau_selfnorm,eq:def_dbar_app}.

  We now turn to the term (II). First, observe that
  \begin{align}
    \nonumber
    &~ \sum_{z \in \Zcal} \1[t \geq \tau(z)] d^\iter{t}(z) \psi^\iter{t}(z)
    \\
    \nonumber
    = &~ \sum_{z \in \Zcal} \1[t \geq \tau(z)] d^\iter{t}(z) \left(\frac{ \sum_{i = 1}^{t - 1} d^\iter{i}(z) }{ \sum_{i = 1}^{t - 1} d^\iter{i}(z) }\right)^{\nicefrac{1}{2}} \psi^\iter{t}(z)
    \\
    \nonumber
    \leq &~ \sqrt{\sum_{z \in \Zcal} \frac{\1[t \geq \tau(z)] \left( d^\iter{t}(z) \right)^2}{\sum_{i = 1}^{t - 1} d^\iter{i}(z)} } \sqrt{\sum_{i = 1}^{t - 1} \E_{d^\iter{i}}\big[(\psi^\iter{t})^2\big]}.
           \tag{by Cauchy-Schwarz inequality}
  \end{align}
  By rearranging this inequality, we have
  \begin{align}
    \nonumber
    \text{(II)} \leq &~ \sum_{t = 1}^{T} \sum_{z \in \Zcal} \frac{\1[t \geq \tau(z)] \left( d_h^\iter{t}(z) \right)^2}{\sum_{i = 1}^{t - 1} d^\iter{i}(z)}
                       \tag{defining $0/0 = 0$}
    \\
    \nonumber
    \leq &~ 2 \sum_{t = 1}^{T} \sum_{z \in \Zcal} \frac{\1[t \geq \tau(z)] \left( d^\iter{t}(z) \right)^2}{C_\on \cdot \mu(z) + \sum_{i = 1}^{t - 1} d^\iter{i}(z)}
           \tag{by \cref{eq:def_tau_selfnorm}}
    \\
    \nonumber
    \lesssim &~ \sum_{t = 1}^{T} \sum_{z \in \Zcal} \frac{ \left( d^\iter{t}(z) \right)^2}{C_\on \cdot \mu(z) + \sum_{i = 1}^{t - 1} d^\iter{i}(z)}
    \\
    \leq &~ \sum_{t = 1}^{T} \sum_{z \in \Zcal} \left( \max_{i \leq T} d^\iter{i}(z) \right) \frac{d^\iter{t}(z) }{\sum_{i = 1}^{t - 1} d^\iter{i}(z) + C_\on \cdot \mu^\star(z)}
           \tag{by \cref{def:low_concentrability}}
    \\
    \leq &~ C_\on(\Dset_h) \sum_{z \in \Zcal} \mu^\star(z) \sum_{t = 1}^{T} \frac{ d^\iter{t}(z) }{\sum_{i < t}d^\iter{i}(z) + C_\on \cdot \mu^\star(z)}
    \tag{by \cref{lem:concen_eq_area}}
    \\
    \lesssim &~ C_\on(\Dset) \sum_{z \in \Zcal} \mu^\star(z) \log(T)
               \tag{by \cref{lem:per_sa_ep}}
    \\
    \label{eq:snorm_C_term2}
    = &~ C_\on(\Dset) \log(T).
  \end{align}

  Substituting \cref{eq:snorm_C_term1,eq:snorm_C_term2} into
  \cref{eq:snorm_C_terms}, we obtain
  \begin{align*}
    \cdim(\Psi, \Dcal, T) \lesssim C_\on(\Dset) \log(T).
  \end{align*}
\end{proof}

\begin{proof}[\pfref{prop:sedim_bedim}]
This proof provides a slightly more general result.
Consider a set of distributions $\Dset \subset \Delta(\Zcal)$ and a set
of test functions $\Psi \subset (\Zcal \to [0,1])$ be given. We
consider an abstract version of the \betext with respect to $\Dset$
and $\Psi$. We define $\bedim(\Psi, \Dset, \varepsilon)$
is the largest $d\in\bbN$ such that there exist sequences
$\{d^\iter{1},d^\iter{2},\dotsc,d^\iter{d}\} \subset \Dset$
and $\{\psi\ind{1},\psi\ind{2},\ldots,\psi^\iter{d}\} \subset \Psi$ such that
for all $t\in\brk{d}$, \footnote{This definition coincides with
  distributional Eluder dimension~\citep[see,
  e.g.,][]{jin2021bellman}, which only differs from \betext on the
  notation of test function. We overload the notation for $\bedim$ over this proof for simplicity.}
  \begin{align}
    \label{eq:bedim_app}
    \abs{\E_{d^\iter{t}}[\psi^\iter{t}]} >
    \varepsilon^\iter{t},\mathand \sqrt{\sum_{i = 1}^{t - 1} \prn[\big]{\E_{d^\iter{i}}[\psi^\iter{t}]}^2} \leq \varepsilon^\iter{t},
  \end{align}
for $\varepsilon\ind{1},\ldots,\veps\ind{d} \geq \varepsilon$. We will
show that, for any all $T\in\bbN$,
\begin{align*}
\cdim(\Psi, \Dcal, T) \lesssim &~ \inf_{\veps>0}\crl*{\veps^2T + \bedim\left(\Psi, \Dset, \varepsilon\right)}\cdot\log(T),
\end{align*}
which immediately implies \cref{prop:sedim_bedim}.

\paragraph{A generalized definition of $\varepsilon$-dependent sequence}

In what follows, we rely on a slightly different notion of an {\em
  $\varepsilon$-(in)dependent sequence} from the one given
in~\citet[Definition 6]{jin2021bellman} and
\citet{russo2013eluder}. We provide background on
both definitions below.

\textbf{$\varepsilon$-(in)dependent sequence~\citep[e.g.,][Definition 6]{jin2021bellman}.}
A distribution $\nu \in \Dset$ is $\veps$-dependent on a sequence $\{\nu^\iter{1}, \dotsc,
\nu^\iter{k}\} \subseteq \Dset$ if: When $|\E_{\nu}[\psi]| > \varepsilon$ for some $\psi
\in \Psi$, we also have $\sum_{i=1}^k (\E_{\nu^\iter{i}}[\psi])^2 > \varepsilon^2$. Otherwise, $\nu$ is $\veps$-independent if this does not hold.

\textbf{Generalized~$\varepsilon$-(in)dependent sequence.}
A distribution $\nu \in \Dset$ is (generalized) $\veps$-dependent on a sequence $\{\nu^\iter{1}, \dotsc,\nu^\iter{k}\} \subseteq \Dset$ if: \underline{for all $\varepsilon' \geq
\varepsilon$}, if $|\E_{\nu}[\psi]| > \varepsilon'$ for some $\psi
\in \Psi$, we also have $\sum_{i=1}^k
(\E_{\nu^\iter{i}}[\psi])^2 > \varepsilon'^2$. We say that $\nu$
is (generalized) $\veps$-independent if this does not hold, i.e., for some $\varepsilon' \geq \varepsilon$, it has $|\E_{\nu}[\psi]| > \varepsilon'$ but $\sum_{i=1}^k
(\E_{\nu^\iter{i}}[\psi])^2 \leq \varepsilon'^2$.

The generalized definition above naturally induces a new implication (which the original definition may not have): \emph{If $\varepsilon' \geq \varepsilon$, then $\varepsilon$-dependent sequence $\Rightarrow$ $\varepsilon'$-dependent sequence, or in other words, $\varepsilon'$-independent sequence $\Rightarrow$ $\varepsilon$-independent sequence.}

The definition of the \emph{distributional Eluder dimension} (see
\cref{eq:bedim_app}) can be written in two equivalent ways using
original and generalized definition for a $\varepsilon$-independent sequence: $\bedim(\Psi, \Dset, \varepsilon)$
is the largest $d\in\bbN$ such that there exists a sequence
$\{d^\iter{1},d^\iter{2},\dotsc,d^\iter{d}\} \subset \Dset$
such that
for all $t\in\brk{d}$:
\begin{enumerate}[(i)]
\item $d^\iter{t}$ is  $\varepsilon'$-independent of $\{ d^\iter{1},d^\iter{2},\dotsc,d^\iter{t-1} \}$ for some $\varepsilon' \geq \varepsilon$.
\item $d^\iter{t}$ is (\emph{generalized}) $\varepsilon$-independent of $\{ d^\iter{1},d^\iter{2},\dotsc,d^\iter{t-1} \}$ $\Longleftarrow$[by the implication above]$\Longrightarrow$ $d^\iter{t}$ is (\emph{generalized}) $\varepsilon'$-independent of $\{ d^\iter{1},d^\iter{2},\dotsc,d^\iter{t-1} \}$ for some $\varepsilon' \geq \varepsilon$.
\end{enumerate}
This indicates that the distributional Eluder dimension can be
equivalently written in terms of generalized independent
sequences. Going forward, we only use the generalized $\varepsilon$-(in)dependent definition, and omit the word generalized.

\paragraph{Setup}
Let us use $\bedim(\varepsilon)$ as shorthand for
$\bedim(\Psi, \Dset, \varepsilon)$. By \cref{eq:bedim_app}, we know $\bedim(\varepsilon)$ also upper bounds the length of sequences $\{d^\iter{1},d^\iter{2},\dotsc,d^\iter{d}\} \subset \Dset$ and $\{\psi\ind{1},\psi\ind{2},\ldots,\psi^\iter{d}\} \subset \Psi$ such that for all $t\in\brk{d}$,
  \begin{align*}
    \abs{\E_{d^\iter{t}}[\delta^\iter{t}]} >
    \varepsilon^\iter{t},\mathand \sqrt{\sum_{i = 1}^{t - 1} \E_{d^\iter{i}}\left[(\psi^\iter{t})^2\right]} \leq \varepsilon^\iter{t},
  \end{align*}
for $\varepsilon\ind{1},\ldots,\veps\ind{d} \geq \varepsilon$ (note that the square is inside the expectation which is different from \cref{eq:bedim_app}).

Now, for any $\{d^\iter{1},d^\iter{2},\dotsc,d^\iter{T}\} \subset \Dset$ and $\{\psi\ind{1},\psi\ind{2},\ldots,\psi^\iter{T}\} \subset \Psi$, we define
  $\beta^\iter{t} \coloneqq \sum_{i = 1}^{t - 1}
  \E_{d^\iter{i}}[(\psi^\iter{t})^2]$. We will study the sequence
  \begin{align}
    \label{eq:def_be_sequence}
    \left\{ \frac{\E_{d^\iter{1}}[\psi^\iter{1}]^2}{1 \vee \beta^\iter{1}}, \frac{\E_{d^\iter{2}}[\psi^\iter{2}]^2}{1 \vee \beta^\iter{3}}, \dotsc, \frac{\E_{d^\iter{T}}[\psi^\iter{T}]^2}{1 \vee \beta^\iter{T}} \right\}.
  \end{align}

Fix a parameter $\alpha>0$, whose value will be specified later. For
the remainder of the proof, we use $L^\iter{t}$ to denote the number of
disjoint $\alpha\sqrt{1 \vee \beta_h^\iter{t}}$-dependent subsequences of $d^\iter{t}$ in $\{d^\iter{1},d^\iter{2},\dotsc,d^\iter{t-1}\}$, for each $t \in [T]$. 

\paragraph{Step 1}
Suppose the $t$-th term of \cref{eq:def_be_sequence} is greater than
$\alpha^2$, so that $|\E_{d^\iter{t}}[\psi^\iter{t}]| >
\alpha\sqrt{1 \vee \beta^\iter{t}}$. From the definition of $L^\iter{t}$, we know there have at least $L^\iter{t}$ disjoint subsequences of $\{d^\iter{1},\dotsc,d^\iter{t-1}\}$ (denoted by $\Seq^\iter{1}, \dotsc, \Seq^\iter{L^\iter{t}}$), such that
\begin{align}
\label{eq:bedim_snc_step1_1}
\sum_{i = 1}^{L^\iter{t}}\sum_{\nu \in \Seq^\iter{i}} (\E_{\nu}[\psi^\iter{t}])^2 \geq (1 \vee \beta^\iter{t})\alpha^2.
\end{align}
On the other hand, by the definition of $\beta_h^\iter{t}$, we have
\begin{align}
\label{eq:bedim_snc_step1_2}
\sum_{i = 1}^{L^\iter{t}}\sum_{\nu \in \Seq^\iter{i}} (\E_{\nu}[\psi^\iter{t}])^2 \leq \sum_{i = 1}^{t - 1} \E_{d^\iter{i}}[(\psi^\iter{t})^2] \leq \beta^\iter{t}.
\end{align}
Therefore, combining \cref{eq:bedim_snc_step1_1,eq:bedim_snc_step1_2} we obtain that, if $|\E_{d^\iter{t}}[\psi^\iter{t}]| > \alpha\sqrt{1 \vee \beta^\iter{t}}$ for some $t \in [T]$, 
\begin{align}
\label{eq:be_dim_prop43_1}
\beta^\iter{t} \geq L^\iter{t} (1 \vee \beta^\iter{t})\alpha^2 \Longrightarrow L^\iter{t} \leq \frac{1}{\alpha^2}.
\end{align}

\paragraph{Step 2}

  On the other hand, let $\{i_1,i_2,\dotsc,i_\kappa\}$ be the longest subsequence of $[T]$, where
  \begin{align*}
    \frac{\E_{d^\iter{i_j}}[\delta^\iter{i_j}]^2}{1 \vee \beta^\iter{i_j}} > \alpha^2,~\forall j \in [\kappa].
  \end{align*}
For compactness, we use $\{\nu^\iter{1},\nu^\iter{2},\dotsc,\nu^\iter{\kappa}\}$ abbreviate $\{d^\iter{i_1},d^\iter{i_2},\dotsc,d^\iter{i_{\kappa}}\}$.
  We now argue that there exists $j^\star \in [\kappa]$, such that for $\nu^\iter{j^\star}$, there must exist at least
  \newcommand{\Lk}{\underline L^\star}
  \begin{align}
    \label{eq:be_dim_prop43_2}
    \Lk \geq \left\lfloor\frac{\kappa}{\bedim\left(\alpha\right) + 1}\right\rfloor \geq \frac{\kappa}{\bedim\left(\alpha\right) + 1} - 1
  \end{align}
  $\alpha$-dependent disjoint subsequences in
  $\{\nu^\iter{1},\nu^\iter{2},\dotsc,\nu^\iter{j^\star-1}\}$ (the actual number of disjoint subsequences is denoted by $\Lk$). This is because we can construct such disjoint subsequences by the following procedure:
  \begin{enumerate}[$\langle$1$\rangle$]
      \item For $j \in [\Lk]$, $\Seq^\iter{j} \leftarrow \{ \nu^\iter{j} \}$. Then, set $j \leftarrow \Lk + 1$.
      \item If $\nu^\iter{j}$ is $\alpha$-dependent on $\Seq^\iter{1},\ldots,\Seq^\iter{\Lk}$, terminate the procedure (goal achieved). 
      \item Otherwise, we know $\nu^\iter{j}$ is $\alpha$-independent
        on at least one of $\Seq^\iter{1},\ldots,\Seq^\iter{\Lk}$ (denoted by $\Seq^\star$). Update $\Seq^\star \leftarrow \Seq^\star \bigcup \{ \nu^\iter{j} \}$, $j \leftarrow j+1$, and go to $\langle$2$\rangle$.
  \end{enumerate}
  From the definition of $\bedim\left(\alpha\right)$, we know if $|\Seq^\iter{i}| \geq \bedim\left(\alpha\right) + 1$, any $\nu \in \Dset_h$ must be $\alpha$-dependent on $\Seq^\iter{i}$ (for each $i \in [\Lk]$). Therefore, such a procedure must terminate before or on $j^\iter{\max} = \Lk \bedim\left(\alpha\right) + \Lk$. Thus, if $j^\iter{\max} \leq \kappa$, termination in $\langle$2$\rangle$ must happen.
  This only requires $\Lk$ to satisfy
  \begin{align*}
  \Lk \bedim\left(\alpha\right) + \Lk \leq \kappa \quad\Longrightarrow\quad \Lk \leq \frac{\kappa}{\bedim\left(\alpha\right) + 1}.
  \end{align*}
  That is, as long as $\Lk \leq
  \left\lfloor\frac{\kappa}{\bedim\left(\alpha\right) +
      1}\right\rfloor$, the termination in $\langle$2$\rangle$ must happen for some $j^\star \leq \kappa$.

\paragraph{Step 3}
As we discussed at the beginning, $\alpha$-dependence implies $\alpha'$-dependence for all $\alpha' \geq \alpha$. This means the $\Lk$ in Step 2 lower bounds $\max_{t \in [T]} L^\iter{t}$ in Step 1, because $\{d^\iter{i_1},d^\iter{i_2},\dotsc,d^\iter{i_\kappa}\}$ is a subset of $\{d^\iter{1},d^\iter{2},\dotsc,d^\iter{i_\kappa}\}$.
Thus, combining \cref{eq:be_dim_prop43_1,eq:be_dim_prop43_2}, we can obtain that,
  \begin{align*}
    \frac{1}{\alpha^2} \geq\max_{t \in [T]} L^\iter{t} \geq \Lk \geq \frac{\kappa}{\bedim\left(\alpha\right) + 1} - 1.
  \end{align*}
This implies that
  \begin{align*}
    \kappa \leq \left( 1 + \frac{1}{\alpha^2} \right) (\bedim\left(\alpha\right) + 1) \leq \frac{3 \bedim\left(\alpha\right)}{\alpha^2} + 1.
    \tag{suppose $\alpha \leq 1$}
  \end{align*}
  As a consequence, for any $\varepsilon \in (0,1]$, by setting
  $\alpha = \sqrt{\varepsilon}$,
  \begin{align}
    \label{eq:be_dim_count_bound}
    \sum_{t = 1}^{T} \1\left( \frac{\E_{d^\iter{t}}[\psi^\iter{t}]^2}{1 \vee \beta^\iter{t}}  > \varepsilon\right) \leq \frac{3\bedim(\sqrt{\varepsilon})}{\varepsilon} + 1.
  \end{align}

\paragraph{Step 4} %
Let $e^\iter{1} \geq e^\iter{2} \geq \cdots \geq e^\iter{T}$ denote the sequence in
\cref{eq:def_be_sequence} reordered in a decreasing fashion. For any
parameter $w \in (0,1]$ to be specified later, we have
  \begin{gather*}
    \sum_{t = 1}^{T} \frac{\E_{d^\iter{t}}[\psi^\iter{t}]^2}{1 \vee \sum_{i = 1}^{t - 1} \E_{d^\iter{i}}\left[(\psi^\iter{t})^2\right]} = \sum_{t = 1}^{T} e^\iter{t}
    \\
    \leq T w + \sum_{t = 1}^{T} e^\iter{t} \1(e^\iter{t} > w).
  \end{gather*}

Observe that for any $t \in [T]$ such that $e^\iter{t} > w$, if
  $2 \eta \geq e^\iter{t} > \eta \geq w$, we have
  \begin{align}
    \nonumber
    t \leq &~ \sum_{i = 1}^{T} \1(e_i > \eta)
    \\
    \leq &~ \frac{3}{\eta} \bedim\left(\sqrt{\eta}\right) + 1 \tag{by \cref{eq:be_dim_count_bound}}
    \\
    \nonumber
    \leq &~ \frac{3}{\eta} \bedim\left(\sqrt{w}\right) + 1
    \\
    \Longrightarrow \eta \leq &~ \frac{3d}{t - 1} \tag{define $d \coloneqq \bedim\left(\sqrt{w}\right)$}
    \\
    \Longrightarrow e^\iter{t} \leq &~ \min\left( \frac{6 d }{t - 1}, 1 \right).
    \tag{$2 \eta \geq e^\iter{t} > \eta$}
  \end{align}
  Therefore,
  \begin{align*}
    \sum_{t = 1}^{T} e^\iter{t} \1(e^\iter{t} > w) \leq &~ d + \sum_{t = d + 1}^{T} \frac{6 d}{t - 1}
    \\
    \leq &~ d + 6 d \log(T).
    \\
    \Longrightarrow \sum_{t = 1}^{T} e^\iter{t} \leq &~ Tw + \bedim\left(\sqrt{w}\right) + 6 \bedim\left(\sqrt{w}\right) \log(T).
  \end{align*}
  Selecting $w = \varepsilon^2$ implies
  \begin{align*}
    \cdim(T) \lesssim &~ \inf_{\veps>0}\crl*{\veps^2T + \bedim\left(\veps\right)}\cdot\log(T).
  \end{align*}
  This completes the proof.
\end{proof}

\iclr{\subsection{Discussion: Relationship to Additional Complexity Measures}}
\iclr{\label{app:complexity_additional}}

\arxiv{\subsection{\CompMeasure: \texorpdfstring{$Q$}{Q}-type versus \texorpdfstring{$V$}{V}-type}}
\iclr{\subsubsection{\CompMeasure: $Q$-type versus $V$-type}}
\label{app:v_type}

The \CompMeasure, as defined in \cref{def:snc_qtype}), can be thought
of as a generalization of $Q$-type \betext \citep{jin2021bellman}. In
this section we sketch how one can adapt \CompMeasure so as to
generalize $V$-type \betext instead. Note that $V$-type \betext subsumes
the original notion of Bellman rank from \citet{jiang2017contextual}.

We define the $V$-type \CompMeasure for RL as follows.
\begin{definition}[\CompMeasure for RL, $V$-type]
\label{def:snc_vtype}
For each $h \in [H]$, let $\Dset_{h,x}^\Pi \coloneqq \{d^\pi_h(\cdot): \pi \in \Pi\} \subset \Delta(\Xcal)$ and $V_{\Fcal_h - \Tcal_h \Fcal_{h+1}} \coloneqq \{(f_h - \Tcal_h f_{h+1})(\cdot,\pi_{f,h}): f \in \Fcal\} \subset (\Xcal \to \RR)$. Then we define,
\begin{align*}
\qquad \cdimrlv(\Fcal,\Pi,T) \coloneqq \max_{h \in [H]} \cdim(V_{\Fcal_h - \Tcal_h \Fcal_{h+1}},\Dset_{h,x}^\Pi,T).
\end{align*}
\end{definition}
We recall that the $V$-type \betext $\bedimv(\cF,\Pi,\veps)$ is defined
analogously, by replacing $\Fcal_h - \Tcal_h \Fcal_{h+1} \to
V_{\Fcal_h - \Tcal_h \Fcal_{h+1}}$ and $\Dset_{h}^\Pi \to
\Dset_{h,x}^\Pi$ in \cref{def:be_dim}.

Lastly, we give a $V$-type generalization of
\cref{def:low_concentrability} (i.e., coverability w.r.t.~state only),
for a policy class $\Pi$ as follows:
\begin{align}
\label{eq:def_ctype_con}
  \Conv \ldef \inf_{\mu_1,\ldots,\mu_H\in\Delta(\cX)}  \sup_{\pi \in \Pi,h\in\brk{H}}\,
  \nrm*{\frac{d_h^\pi}{\mu_h}}_{\infty}.
\end{align}
As a simple implication, we have $\Conv \leq C_\on \leq \Conv \cdot |\Acal|$.

Note that the $V$-type variants of \compmeasure, \betext, and
coverability differ from their $Q$-type counterparts only in the choices
for the distribution and test function sets. Since our proofs for \cref{prop:sedim_coverability,prop:sedim_bedim} hold for arbitrary distributions and test function sets, we immediately obtain the following $V$-type extensions of \cref{prop:sedim_coverability,prop:sedim_bedim}.

\begin{proposition}[Coverability $\Longrightarrow$ $\SNC$, $V$-type]
  \label{prop:sedim_coverability_v}
  Let $\Conv$ be the $V$-type coverability coefficient
  (\cref{eq:def_ctype_con}) with policy class $\Pi$. Then for
  any value function class $\cF$, $\cdimrlv(\Fcal,\Pi,T) \leq O \left( \Conv\cdot\log(T) \right)$.
\end{proposition}

\begin{proposition}[\betext $\Longrightarrow$ $\SNC$, $V$-type]
  \label{prop:sedim_bedim_v}
  Suppose $\bedimv(\Fcal,\Pi,\veps)$ be the $V$-type \betext with function class $\Fcal$ and policy $\Pi$, then
  \begin{align*}
      \cdimrlv(\Fcal,\Pi,T) \leq O\left(\inf_{\veps>0}\crl*{\veps^2T + \bedimv(\Fcal,\Pi,\veps)}\cdot\log(T)\right).
  \end{align*}
\end{proposition}
As shown in~\citet{jin2021bellman}, \golf (\cref{alg:golf}) can be
extended to $V$-type by simply replacing \cref{step:glof_sampling} in
\cref{alg:golf} with sampling $(s_h,a_h,r_h,s_{h+1}) \sim d_h^\iter{t}
\times \pi_{\unif}$ ($s_h \sim d_h^\iter{t}$ and $a_h \sim
\unif(\Acal)$) each $h \in [H]$. By slightly modifying the proof of
\pref{thm:regret_online_c} one can obtain similar sample complexity
guarantees based on the $V$-type \CompMeasure. We omit the details here,
since the only differences are 1) a $V$-type analog of
\cref{lem:golf_concentration}~\citep[provided by][Lemma
44]{jin2021bellman}; and 2) trivially upper bounding the quantity $\E_{d_h^\iter{i}
  \times \pi_h^\iter{t}}[(\delta_h^\iter{t})^2]$ (used in $\cdimrlv$)
by $|\Acal|\cdot\E_{d_h^\iter{i} \times \pi_\unif}[(\delta_h^\iter{t})^2]$
(controlled by in-sample error). Note, however, that due to the uniform
exploration, this algorithm leads to a sample complexity guarantee of
the form
\begin{align*}
J(\pi^\star) - J(\pibar) \leq O \left( H\sqrt{\frac{\cdimrlv(\Fcal,\Pi,T) |\Acal| \log(\nicefrac{TH |\Fcal|}{\delta})}{T}}\right),
\end{align*}
but not a regret bound.

\arxiv{\subsection{Connection to Bilinear Classes}}
\iclr{\subsubsection{Connection to Bilinear Classes}}
\label{app:bilinear}
The Bilinear
class framework \citep{du2021bilinear} generalizes the notion of Bellman rank
\citep{jiang2017contextual}, which captures various more structural
conditions via an additional class of \emph{discrepancy functions}. In this section we sketch how one can generalize the \compmeasure ($\SNC$)
further by allowing for the use of general discrepancy functions to form
confidence sets and estimate Bellman residuals, in the vein of
Bilinear classes. 
\begin{definition}[\BiCompMeasure]
\label{def:online_c_bi}
Let $\cZ$ be an abstract set. Let $\Psi\subset(\cZ\to\bbR)$ be a
\emph{function class}, and let $\Dset_\Psi (\coloneqq \{d_\psi: \psi
\in \Psi\}), \Pset_\Psi (\coloneqq \{p_\psi: \psi
\in \Psi\}) \subset\Delta(\cZ)$ be two corresponding distribution classes, and
$\Lset_{\Psi} (\coloneqq \{\ell_\psi: \psi \in \Psi\})
\subset(\cZ\to\bbR)$ be a corresponding discrepancy function class.
The \BiCompMeasure for length $T$ is given by
\begin{align}
\cdimbi(\Psi,\Dset_\Psi,\Pset_\Psi,\Lset_{\Psi},T) \ldef{}
  \sup_{\psi^\iter{1},\ldots,\psi^\iter{T} \in \Psi}\crl*{\sum_{t =
  1}^{T} \frac{\E_{d_{\psi^\iter{t}}}[\psi^\iter{t}]^2 }{1 \vee
  \sum_{i = 1}^{t - 1}
  \E_{p_{\psi^\iter{i}}}[\ell_{\psi^\iter{t}}^2]}}.
  \label{eq:snc_gen}
\end{align}
\end{definition}
To apply the generalized $\SNC$ to reinforcement learning, one can set (for each level $h$) $\Psi = \Fcal_h - \Tcal_h \Fcal_{h+1}$, $\Dset_\Psi = \{d^\pi_h(\cdot,\cdot): \pi \in \Pi_\Fcal\}$, and $\Pset_\Psi = \{(d^\pi_h \times \pi_{\est,\psi_h})(\cdot,\cdot): \pi \in \Pi_\Fcal\}$, where $(d \times \pi)(x,a) \coloneqq d(x) \pi(a|x)$ (for any $d \in \Delta(\Xcal)$, $\pi \in (\Xcal \to \Delta(\Acal))$ and $(x,a) \in \Xcal \times \Acal$), and $\pi_{\est,\psi_h}$ denotes the estimation policy depending on $\psi_h$ (e.g., greedy policy w.r.t.~$\psi_h$ or uniformly random policy over $\Acal$).
The discrepancy function class $\Lset_\Psi$ can be selected according to the original Bilinear rank for covering various structural conditions, and setting $\Lset_\Psi = \Psi$ recovers the original $\SNC$.

By combining \golf and \cref{thm:regret_online_c} with the approach from
\citet{du2021bilinear}, one can provide sample complexity guarantees
that scale with the \BiCompMeasure. We omit the details, but the basic
idea is to form the confidence set using the discrepancy function
class $\Lset_\Psi$ rather than working with squared Bellman error.

\paragraph{Bounding the generalized $\SNC$ by bilinear rank}
In what follows, we show that the abstract version of the Generalized
$\SNC$ in \pref{eq:snc_gen} can be bounded by an abstract
generalization of the notion of Bilinear rank from \citet{du2021bilinear}.
\begin{definition}[\bitext, finite dimension~\citep{du2021bilinear}]
  \label{def:bilinear}
  Let $\cZ$ be an abstract set. Let $\Psi\subset(\cZ\to\bbR)$ be a
\emph{function class}, and let $\Dset_\Psi (\coloneqq \{d_\psi: \psi
\in \Psi\}), \Pset_\Psi (\coloneqq \{p_\psi: \psi
\in \Psi\}) \subset\Delta(\cZ)$ be two corresponding distribution classes, and
$\Lset_{\Psi} (\coloneqq \{\ell_\psi: \psi \in \Psi\})
\subset(\cZ\to\bbR)$ be a corresponding discrepancy function
class. The class $\Psi$ is said to have Bilinear rank $d$ if there
exists $\psi^\star \in \Psi$ and functions $X,W \subset( \Psi \to
\RR^{d} )$ such that 1) $\sum_{\psi \in \Psi} \| X(\psi) \|_2 \leq 1$
and $\sum_{\psi \in \Psi} \| W(\psi) \|_2 \leq B_W$, and 2)
\begin{align*}
\E_{d_{\psi}}[\psi] \leq &~ \left| \left\langle W(\psi) - W(\psi^\star), X(\psi) \right\rangle \right| \quad \forall \psi \in \Psi,
\\
\E_{p_{\psi}}[\ell_{\psi'}] = &~ \left| \left\langle W(\psi') - W(\psi^\star), X(\psi) \right\rangle \right| \quad \forall \psi,\psi' \in \Psi.
\end{align*}
We define $\bidim(\Psi,\Dset_\Psi,\Pset_\Psi,\Lset_{\Psi})$ as the least
dimension $d$ for which this property holds.
\end{definition}

\begin{proposition}[\bitext $\Longrightarrow$ \BiCompMeasure]
\label{prop:bilinear_snc}
Suppose $\cdimbi(\Psi,\Dset_\Psi,\Pset_\Psi,\Lset_{\Psi},T)$ and $\bidim(\Psi,\Dset_\Psi,\Pset_\Psi,\Lset_{\Psi})$ be the \bicompmeasure and \bitext defined in \cref{def:online_c_bi,def:bilinear} with respect to function class $\Psi$, distribution classes $\Dset_\Psi$ and $\Dset_\Psi$, and discrepancy function class $\Lset_{\Psi}$. Then we have,
\begin{align*}
\cdimbi(\Psi,\Dset_\Psi,\Pset_\Psi,\Lset_{\Psi},T) \lesssim \bidim(\Psi,\Dset_\Psi,\Pset_\Psi,\Lset_{\Psi}) \log\left( 1 + \frac{4 B_W^2 T }{d} \right).
\end{align*}
\end{proposition}
\begin{proof}[\cpfname{prop:bilinear_snc}]
Throughout the proof, we use $d^\iter{t}$, $p^\iter{t}$ and $\ell^\iter{t}$ as the shorthands of $d_{\psi^\iter{t}}$, $p_{\psi^\iter{t}}$ and $\ell_{\psi^\iter{t}}$.
  We study the quantity,
  \begin{align*}
    \sum_{t = 1}^{T} \frac{\E_{d^\iter{t}}[\psi^\iter{t}]^2 }{1 \vee \sum_{i = 1}^{t - 1} \E_{p^\iter{i}}[(\ell^\iter{t})^2]} \leq 2 \sum_{t = 1}^{T} \frac{\E_{d^\iter{t}}[\psi^\iter{t}]^2 }{1 + \sum_{i = 1}^{t - 1} \E_{p^\iter{i}}[\ell^\iter{t}]^2}.
  \end{align*}
  By \cref{def:bilinear}, we have
  \begin{align*}
    \E_{d^\iter{t}}[\psi^\iter{t}]^2 \leq &~ \left| \left\langle W(\psi^\iter{t}) - W(\psi^\star), X(\psi^\iter{t}) \right\rangle \right|^2,
\intertext{and}
    1 + \sum_{i = 1}^{t - 1} \E_{p^\iter{i}}[\ell^\iter{t}]^2 = &~ 1 + \sum_{i = 1}^{t - 1} \left| \left\langle W(\psi^\iter{t}) - W(\psi^\star), X(\psi^\iter{i}) \right\rangle \right|^2
    \\
    \geq &~ \left( W(\psi^\iter{t}) - W(\psi^\star) \right)^\T \Sigma_t \left( W(\psi^\iter{t}) - W(\psi^\star) \right)
    \\
    = &~ \left\| W(\psi^\iter{t}) - W(\psi^\star) \right\|_{\Sigma_t}^2,
  \end{align*}
  where
  $\Sigma_t \ldef \frac{1}{4B_W^2}\bfI + \sum_{i = 1}^{t - 1} X(\psi^\iter{i})
  X(\psi^\iter{i})^\T$.

We bound
  \begin{align*}
    \E_{d^\iter{t}}[\psi^\iter{t}]^2 \leq &~ \left| \left\langle W(\psi^\iter{t}) - W(\psi^\star), X(\psi^\iter{t}) \right\rangle \right|^2
    \\
    \leq &~ \left\| W(\psi^\iter{t}) - W(\psi^\star) \right\|_{\Sigma_t}^2 \cdot \left\| X(\psi^\iter{t}) \right\|_{\Sigma_t^{-1}}^2,
  \end{align*}
which implies
  \begin{align*}
    \sum_{t = 1}^{T} \frac{\E_{d^\iter{t}}[\psi^\iter{t}]^2 }{1 + \sum_{i = 1}^{t - 1} \E_{p^\iter{i}}[\ell^\iter{t}]^2} \leq &~ \sum_{t = 1}^{T} 1 \wedge \left\| X(\psi^\iter{t}) \right\|_{\Sigma_t^{-1}}^2
    \\
    \leq &~ 2 \log\left( \frac{\det(\Sigma_T)}{\det(\Sigma_1)} \right)
    \\
    \leq &~ 2 \bidim(\Psi,\Dset_\Psi,\Pset_\Psi,\Lset_{\Psi}) \log\left( 1 + \frac{4 B_W^2 T }{d} \right),
  \end{align*}
  where the last two inequalities follow from the elliptical potential
  lemma~\citep[Lemma 19.4]{lattimore2020bandit}.  Putting everything
  together, we obtain
  \begin{align*}
    \cdimbi(\Psi,\Dset_\Psi,\Pset_\Psi,\Lset_{\Psi},T) \leq 4 \bidim(\Psi,\Dset_\Psi,\Pset_\Psi,\Lset_{\Psi}) \log\left( 1 + \frac{4 B_W^2 T }{d} \right).
  \end{align*}
\end{proof}

\arxiv{
\subsection{Further Connections to Bellman-Eluder Dimension}

\iclr{
\begin{definition}[\betextsq]
\label{def:be_dim_sq}
  We define the \betextsq $\bedimsq(\Fcal, \Pi, \varepsilon,h)$ for layer $h$ is the
  largest $d\in\bbN$ such that there exist sequences
  $\{d_h^\iter{1},d_h^\iter{2},\dotsc,d_h^\iter{d}\} \subseteq
  \Dset^\Pi_h$ and
  $\{\delta_h\ind{1},\ldots,\delta_h^\iter{d}\} \subseteq \Fcal_h - \Tcal_h
  \Fcal_{h+1}$ such that for all $t\in\brk{d}$,
  \begin{align}
    \label{eq:bedimsq}
    \abs{\E_{d_h^\iter{t}}[\delta_h^\iter{t}]} >
    \varepsilon^\iter{t},\mathand \sqrt{\sum_{i = 1}^{t - 1} \E_{d_h^\iter{i}}[(\delta_h^\iter{t})^2]} \leq \varepsilon^\iter{t},
  \end{align}
  for $\varepsilon\ind{1},\ldots,\veps\ind{d} \geq \varepsilon$. We define $\bedimsq(\cF,\Pi,\veps)=\max_{h\in\brk{H}}\bedimsq(\cF,\Pi,\veps,h)$.
\end{definition}
This definition is identical to \pref{def:be_dim}, except that the
constraint $\sqrt{\sum_{i = 1}^{t - 1}
  \prn{\E_{d_h^\iter{i}}[\delta_h^\iter{t}]}^2} \leq
\varepsilon^\iter{t}$ in \pref{eq:bedim} has been replaced by the
constraint $\sqrt{\sum_{i = 1}^{t - 1}
  \E_{d_h^\iter{i}}[(\delta_h^\iter{t})^2]} \leq
\varepsilon^\iter{t}$, which uses squared Bellman error instead of
average Bellman error. By adapting the analysis of
\citet{jin2021bellman} it is possible to show that this definition
yields
$\Reg\leq\bigoht\prn[\big]{H\sqrt{\inf_{\veps>0}\crl{\veps^2T+\bedimsq(\cF,\Pi,\veps)}\cdot{}T\log\abs{\cF}}}$. If
one could show that
$\dimbesq(\cF,\Pi,\veps)\approxleq{}\Ccov\cdot\polylog(\veps^{-1})$, 
this would recover \pref{thm:golf_guarantee_basic}. Unfortunately, it
turns out that in general, one can have
$\dimbesq(\cF,\Pi,\veps)=\bigom(\Ccov/\veps)$, which leads to
suboptimal $T^{2/3}$-type regret using the result above. The following
result shows that this guarantee cannot be improved without changing
the complexity measure under consideration.
\begin{proposition}
\label{prop:bedim_lower_bound}
Fix $T\in\bbN$, and let $\veps_T\ldef{}T^{-1/3}$. There exist MDP-policy class-value function class tuples
$(M_1,\Pi_1,\cF_1)$ and $(M_2,\Pi_2,\cF_2)$ with the following
properties.
\begin{itemize}
\item For MDP $M_1$, we have
  $\dimbesq(\cF_1,\Pi_1,\veps_T)\propto{}1/\veps_T$, and any algorithm
  must have $\En\brk*{\Reg}\geq{}\bigom(T^{2/3})$.
\item For MDP $M_2$ we also have
  $\dimbesq(\cF_2,\Pi_2,\veps_T)\propto{}1/\veps_T$, yet
  $\Ccov=\bigoh(1)$ and \golf attains $\En\brk*{\Reg}\leq\bigoht(\sqrt{T})$.
\end{itemize}

\end{proposition}
Informally, the reason why \betext fails capture the optimal rates for
the problem instances in \pref{prop:bedim_lower_bound} is that the
definition \pref{eq:bedimsq} only checks whether the average Bellman
error violates the threshold $\veps$, and does not consider how far
the error violates the threshold (e.g.,
$\abs{\E_{d_h^\iter{t}}[\delta_h^\iter{t}]}>\veps$ and
$\abs{\E_{d_h^\iter{t}}[\delta_h^\iter{t}]}>1$ are counted the same).

 }

In spite of this counterexample, it is possible to show that the
\betext with squared Bellman error is always bounded by the
\CompMeasure up to a $\poly(\veps^{-1})$ factor, and hence can always
be bounded by \coverability, albeit suboptimally.

\begin{proposition}
\label{prop:sqbedim_leq_snc}
$\min \{\bedimsq(\Fcal,\Pi,\veps), T\} \leq  \frac{\cdimrl(\Fcal,\Pi,T)}{\varepsilon^2}$. 
\end{proposition}

\begin{proof}[\cpfname{prop:sqbedim_leq_snc}]
Fix an arbitrary $h \in [H]$, suppose $\{d_h^\iter{1},d_h^\iter{2},\dotsc,d_h^\iter{n}\}$ and $\{\delta_h^\iter{1},\delta_h^\iter{2},\dotsc,\delta_h^\iter{n}\}$ is the sequence suggested by the definition of $\bedim(\Fcal,\Pi,\veps,h)$. Then
\begin{align*}
\bedimsq(\Fcal,\Pi,\veps,h) \leq &~ \sum_{t = 1}^{n} \frac{\E_{d_h^\iter{t}}\left[\delta_h^\iter{t}\right]^2}{\left( \varepsilon^\iter{t} \right)^2}
\\
\leq &~ \sum_{t = 1}^{n} \left( 1 + \left( \varepsilon^\iter{t} \right)^2 \right) \cdot \frac{\E_{d_h^\iter{t}}\left[\delta_h^\iter{t}\right]^2}{\left( \varepsilon^\iter{t} \right)^2 \left(1 + \sum_{i = 1}^{t - 1} \E_{d_h^\iter{i}}[\delta_h^\iter{t}]^2\right)}
\tag{by $\sum_{i = 1}^{t - 1} \E_{d_h^\iter{i}}[\delta_h^\iter{t}]^2 \leq ( \varepsilon^\iter{t} )^2$}
\\
\leq &~ \sum_{t = 1}^{n} \frac{1 + \left( \varepsilon^\iter{t} \right)^2}{\left( \varepsilon^\iter{t} \right)^2} \frac{\E_{d_h^\iter{t}}\left[\delta_h^\iter{t}\right]^2}{1 + \sum_{i = 1}^{t - 1} \E_{d_h^\iter{i}}[\delta_h^\iter{t}]^2}
\\
\leq &~ \sum_{t = 1}^{n} \frac{2}{\left( \varepsilon^\iter{t} \right)^2} \frac{\E_{d_h^\iter{t}}\left[\delta_h^\iter{t}\right]^2}{1 + \sum_{i = 1}^{t - 1} \E_{d_h^\iter{i}}[\delta_h^\iter{t}]^2}
\tag{by $\varepsilon^\iter{t} < | \E_{d_h^\iter{t}}[\delta_h^\iter{t}] | \leq 1$}
\\
\leq &~ \frac{1}{\varepsilon^2} \sum_{t = 1}^{n} \frac{\E_{d_h^\iter{t}}\left[\delta_h^\iter{t}\right]^2}{1 \vee \sum_{i = 1}^{t - 1} \E_{d_h^\iter{i}}[\delta_h^\iter{t}]^2}
\tag{by $\varepsilon^\iter{t} \geq \varepsilon$}
\\
\leq &~ \frac{\cdimrl(\Fcal,\Pi,n)}{\varepsilon^2}.
\end{align*}
This implies for any $T > 0$,
\begin{align*}
\min \{\bedimsq(\Fcal,\Pi,\veps), T\} \leq  \frac{\cdimrl(\Fcal,\Pi,T)}{\varepsilon^2}.
\end{align*}
\end{proof}
}

\section{Extension: Reward-Free Exploration}
\label{sec:reward_free}
\renewcommand{\off}{{\sf off}}

Reward-free exploration investigates is a problem where 1) the
learning agent interacts with an environment without rewards, aiming
to gather information so that 2) in a subsequent offline phase, the
information collected can be used to learn near-optimal policies for a
wide range of possible reward
functions~\citep{jin2020reward,zhang2020task,wang2020reward,zanette2020provably,chen2022statistical}. This
section provides a reward-free extension of our main results, and gives
sample complexity bounds based on coverability for a reward-free
extension of \golf.

\paragraph{Function approximation}
We assume access to a value function class $\cF$, which is used for
the offline optimization, and a function class, $\Gcal$, which is used for the reward-free exploration phase. Following the normalized reward assumption, we assume $g_h \in \Xcal \times \Acal \to [0,1]$, $\forall (g,h) \in \Gcal \times [H]$.  

We define $\Pcal_{h}$ as be the ``zero-reward'' Bellman operator for horizon $h \in [H]$. That is, for any $g_h \in \Gcal_h$ and any $h \in [H]$,
\begin{align*}
(\Pcal_{h} g_{h+1})(x_h,a_h) \coloneqq \sum_{x'} \PP_h(x_{h+1}|x_h,a_h) \max_{a_{h+1} \in \Acal} g_{h+1}(x_{h+1},a_{h+1}).
\end{align*}

We let $R$ denote the \emph{target reward function} used in the
offline phase, which is not known to the algorithm in the offline
exploration phase. We make the following assumption.
\begin{assumption}[Reward-free completeness]
\label{asm:rf_completeness}
Let $\Tcal_{1:H}$ be the Bellman operator with the target reward function $R$, and $\Fcal$ be the function class used to optimize the target reward function. Then for all $h \in [H]$
\begin{enumerate}[(a)]
    \item $\Pcal_h \Gcal_{h+1} \in \Gcal_h$ for all $g_{h+1} \in \Gcal_{h+1}$ 
    \item $\Fcal_h - \Tcal_h \Fcal_{h+1} \subseteq \Gcal_h - \Pcal_{h} \Gcal_{h+1}$.
\end{enumerate}
\end{assumption}
Analogous to \cref{asm:completeness}, \cref{asm:rf_completeness}(a) is
used to control the squared Bellman error with zero
reward. \cref{asm:rf_completeness}(b) guarantees that the class of test
functions of interest for the reward-free exploration phase ($\Gcal_h
- \Pcal_{h} \Gcal_{h+1}$ for layer $h \in [H]$, see
\cref{alg:reward_free_golf}) is sufficiently rich relative to the
relevant class of test functions for the offline phase ($\Fcal_h - \Tcal_{h} \Fcal_{h+1}$ for layer $h \in [H]$, see \cref{alg:rf_exploitation}).
Without loss of generality, we assume that $|\Gcal| =
\max\{|\Fcal|,|\Gcal|\}$.

\paragraph{Reward-free \CompMeasure}

The main guarantees for this section are stated in terms of a
reward-free variant of the \compmeasure, which we define as follows.
\begin{definition}[\CompMeasure for Reward-Free RL]
\label{def:snc_qtype_rf}
For each $h \in [H]$, let $\Dset_h^{\Pi_\Gcal} \coloneqq \{d^\pi_h: \pi \in \Pi_\Gcal\}$ and $\Gcal_h - \Pcal_h \Gcal_{h+1} \coloneqq \{g_h - \Pcal_h g_{h+1}: g \in \Gcal\}$. Then we define,
\begin{align*}
\cdimrlrf(\Gcal,\Pi_\Gcal,T) \coloneqq \max_{h \in [H]} \cdim(\Gcal_h - \Pcal_h \Gcal_{h+1},\Dset_h^{\Pi_{\Gcal}},T).
\end{align*}
\end{definition}
Using the same arguments (and same proofs) as \cref{sec:snc_main}, the
reward-free variant of \compmeasure can be shown to subsume
\coverability (as well as reward-free counterpart of the \betext,
which we omit).

\subsection{Algorithm and Theoretical Analysis}
\label{sec:rf_algo}

\begin{algorithm}[th]
\caption{Reward-Free Exploration with \golf}
\label{alg:reward_free_golf}
{\bfseries input:} Function class for reward-free exploration
$\Gcal$. \\
{\bfseries initialize:} $\Dcal_{h,\rf}^\iter{0} \leftarrow \emptyset$,
$\forall h \in [H]$. $\Gcal^\iter{0} \leftarrow \Gcal$.
\begin{algorithmic}[1]
\For{episode $t = 1,2,\dotsc,T$}
    \State Select policy $\pi^\iter{t} \leftarrow \pi_{g^\iter{t}}$, where $g^\iter{t} = \argmax_{g \in \Gcal^\iter{t-1}}g(x_1,\pi_{g,1})$.
    \State Execute $\pi^\iter{t}$ for one episode and obtain $\left\{x_1^\iter{t},a_1^\iter{t},x_2^\iter{t},\dotsc,x_H^\iter{t},a_H^\iter{t},x_{H+1}^\iter{t}\right\}$.
    \State Update historical data $\Dcal_{h,\rf}^\iter{t} \leftarrow \Dcal_{h,\rf}^\iter{t-1} \bigcup \left\{\left(x_h^\iter{t},a_h^\iter{t},x_{h+1}^\iter{t}\right)\right\}$, $\forall h \in [H]$.
    \State Compute confidence set:
    \begin{gather}
    \label{eq:def_vspace_rf}
    \Gcal^\iter{t} \leftarrow \left\{ g \in \Gcal: \Lcal_{h,\rf}^\iter{t}(g_h,g_{h+1}) - \min_{g'_h \in \Gcal_h} \Lcal_{h,\rf}^\iter{t}(g'_h,g_{h+1}) \leq \beta_\rf, ~\forall h \in [H] \right\},
    \\
    \nonumber
    \text{where \quad } \Lcal_{h,\rf}^\iter{t}(g,g') \coloneqq \sum_{(x,a,x') \in \Dcal_{h,\rf}^\iter{t}}\left[ \left( g(x,a) - \max_{a' \in \Acal} g'(x',a') \right)^2 \right],~\forall g,g' \in \Gcal.
    \end{gather}
\EndFor
\State Select $t_\star \leftarrow \argmin_{t \in [T]} g_1^\iter{t}(x_1,\pi_1^\iter{t})$.
\State Return data $\Dcal_{h,\rf}^\iter{t_\star - 1}$, $\forall h \in [H]$.
\end{algorithmic}
\end{algorithm}
\begin{algorithm}[th]
\caption{Offline \golf with Exploration Data and Target Reward}
\label{alg:rf_exploitation}
{\bfseries input:}
\begin{itemize}
\item Target reward function, $R$.
      \item Function class $\cF$ for offline RL.
  \item Exploration data from \pref{alg:reward_free_golf}, denoted by
    $\Dcal_{h,\rf}$, $\forall h \in [H]$.
\end{itemize}
\begin{algorithmic}[1]
\State Compute confidence set:
\begin{gather}
\label{eq:def_offline_vs}
\Fcal^\iter{\off} \leftarrow \left\{f \in \Fcal: \Lcal_{h}^\iter{\off}(f_h,f_{h+1}) - \min_{f'_h \in \Fcal_h} \Lcal_{h}^\iter{\off}(f'_h,f_{h+1}) \leq \beta_\off, ~ \forall h \in [H] \right\},
\\
\nonumber
\text{where \quad } \Lcal_{h}^\iter{\off}(f,f') \coloneqq \sum_{(x,a,x') \in \Dcal_{h,\rf}}\left[ \left( f(x,a) - R(x,a) - \max_{a' \in \Acal} f'(x',a') \right)^2 \right],~\forall f,f' \in \Fcal.
\end{gather}
\State Return $\pihat \leftarrow \pi_{\fhat}$, where $\fhat = \argmax_{f \in \Fcal^\iter{\off}}f(x_1,\pi_{f,1})$.
\end{algorithmic}
\end{algorithm}

Recall that the key ideas in \golf are: 1) using optimism to relate
regret to on-policy average Bellman error; 2) using squared Bellman
error to construct a confidence set, which ensures optimism. In the
reward-free setting, one can apply these ideas by running \golf
(\cref{alg:reward_free_golf}) with
rewards set to zero. Intuitively, this strategy ensures exploration
because the algorithm must explore to rule out test functions in
$\cG$. However, a-priori it is unclear whether running some standard
offline RL algorithms on the exploration data produced by this
strategy should lead to a near-optimal policy, especially given that the PAC guarantee of \golf relies on outputting a uniform mixture of all historical policies (see, e.g., \cref{cor:basic_c_batch}).

To address such issues, one can imagine that, if we know which is the best over all historical policies (say, $\pi^\iter{t_\star}$ for some $t_\star$), could running one-step \golf on the exploration data at $t_\star$ (\cref{alg:rf_exploitation}) guarantee to find a good policy? Note that, for the original \golf algorithm (in the known-reward case), running so directly reproduces $\pi^\iter{t_\star}$. 
Although knowing which is the best over all historical policies seems impossible in the known-reward case, thanks to the reward-free nature, we will show that the value of $g(x_1,\pi_{g,1})$ directly captures ``how bad is $g$'' (akin to the regret in the known-reward case), which allow us to find the best step over the reward-free exploration phase.

The following result provides a sample complexity guarantee for this strategy.

\begin{theorem}
\label{thm:rf_golf}
Under \cref{asm:completeness,asm:rf_completeness}, there exists an absolute constants $c_1$ and $c_2$ such that for any $\delta \in (0,1]$ and $T \in \NN_+$, if we choose $\beta_\off = c_1 \cdot \log(\nicefrac{TH |\Gcal|}{\delta})$ and $\beta_\rf = (c_1 + c_2) \cdot \log(\nicefrac{TH |\Gcal|}{\delta})$ in \cref{alg:reward_free_golf,alg:rf_exploitation}, then with probability at least $1 - \delta$, the policy $\pihat$ output by \cref{alg:rf_exploitation} has
\begin{align*}
J(\pi^\star) - J(\pihat) \leq O \left( H\sqrt{\frac{\cdimrlrf(\Gcal,\Pi_\Gcal,T) \log(\nicefrac{TH |\Gcal|}{\delta})}{T}}\right).
\end{align*}
\end{theorem}

We defer the proof to \cref{sec:wf_proofs}. We also introduce the following two lemmas, which are key to adapting the known-reward results to the reward-free case.

\begin{lemma}[Reward-free exploration overestimates regret]
\label{lem:rf_offline_keylem}
For any $f \in \Fcal$, let $g$ be defined as $g_{h} = f_{h} - \Tcal_h f_{h+1} + \Pcal_h g_{h+1}$, $\forall h \in [H]$. Then for any $(x,a,h) \in \Xcal \times \Acal \times [H]$, we have $g_h(x,a) \geq f_h(x,a) - Q_h^{\pi_f}(x,a)$.
\end{lemma}
Since the Q-function for all policies in the zero-reward case are zero, \cref{lem:rf_offline_keylem} guarantees that, regret in the reward-free exploration phase---($g_1(x_1,\pi_{g,1}) - 0$) always upper bounds its counterpart of the offline phase---($f_1(x_1,\pi_{f,1}) - Q_h^{\pi_f}(x,\pi_{f,1})$). Equipped with the optimism argument, we can show that if $g_1(x_1,\pi_{g,1})$ is small, its corresponding $\pi_f$ (the $f$ with $f_h - \Tcal_h f_{h+1} = g_h - \Pcal_h g_{h+1}$, $\forall h \in [H]$) also has small regret.

\begin{lemma}[Reward-free exploration has larger confidence set]
\label{lem:rf_offline_vspace}
Suppose \pref{asm:rf_completeness} holds and under the same conditions as \cref{thm:rf_golf}. For any $f \in \Fcal^\iter{\off}$ (defined in \cref{eq:def_offline_vs}), there must exist $g \in \Gcal^\iter{t_\star - 1}$ (defined in \cref{eq:def_vspace_rf}), such that $f_h - \Tcal_h f_{h+1} = g_h - \Pcal_h g_{h+1}$, $\forall h \in [H]$.
\end{lemma}
\cref{lem:rf_offline_vspace} ensures that the reward-free version space $\Gcal^\iter{t_\star - 1}$ subsumes the offline version space $\Fcal^\iter{\off}$. Thus, we can use the metrics during reward-free exploration to upper bound that of the offline phase.

\subsubsection{Related Work}

Our approach adapts techniques for reward-free exploration in
nonlinear RL introduced in \citet{chen2022statistical}. In what
follows, we discuss the connection to this work in greater detail. We
focus on the $Q$-type results
of~\citet{chen2022statistical}, but similar arguments are likely apply to
the $V$-type.

Briefly, \citet{chen2022statistical} extends the \olive
algorithm to the reward-free setting by using the idea of online
exploration with zero rewards. The most important difference here is
that, as discussed in
\pref{sec:gen_bedim}, since \olive only considers average Bellman residuals,
it cannot capture \coverability. Beyond this difference, let us
compare the completeness assumptions in \pref{asm:rf_completeness} to
those made in \citet{chen2022statistical}.  We will show that the
completeness assumption used by \citet[Assumption
2]{chen2022statistical} is a sufficient condition for ours
(\cref{asm:completeness,asm:rf_completeness}). In our notation, \citet{chen2022statistical},
use $\Fcal \coloneqq \Psi + R \coloneqq \{\psi_{1:H}(\cdot,\cdot) + R_{1:H}(\cdot,\cdot): \psi \in \Psi\}$ for some function class $\Psi$ during offline phase, and select $\Gcal \coloneqq \Psi - \Psi \coloneqq \{\psi_{1:H}(\cdot,\cdot) - \psi'_{1:H}(\cdot,\cdot): \psi,\psi' \in \Psi\}$ for the reward-free exploration phase. Thus for any $h \in [H]$, we have: For \cref{asm:completeness} and \cref{asm:rf_completeness}(a):
    \begin{align*}
        \Tcal_h \Fcal_{h+1} = &~ R_h + \Pcal_h (\Psi_{h+1} + R_{h+1})
        \\
        \subseteq &~ R_h + \Psi_{h} = \Fcal_h.
        \tag{by \citet[Assumption 2]{chen2022statistical}}
        \\
        \Pcal_h \Gcal_{h+1} = &~ \Pcal_h (\Psi_{h+1} - \Psi_{h+1})
        \\
        \subseteq &~ \Psi_{h+1} - \Psi_{h+1} = \Gcal_{h}.
        \tag{by \citet[Assumption 2]{chen2022statistical}}
    \end{align*}
    For \cref{asm:rf_completeness}(b):
    \begin{align*}
        \Fcal_h - \Tcal_h \Fcal_{h+1} = &~ \Psi_h + R_h - R_h - \Pcal_h (\Psi_{h+1} + R_{h+1})
        \\
        = &~ \Psi_h - \Pcal_h (\Psi_{h+1} + R_{h+1})
        \\
        \subseteq &~ \Psi_h - \Psi_h.
        \tag{by \citet[Assumption 2]{chen2022statistical}}
        \\
        \Gcal_h - \Pcal_h \Gcal_{h+1} = &~ \Psi_h - \Psi_h - \Pcal_h(\Psi_{h+1} - \Psi_{h+1})
        \\
        \supseteq &~ \Psi_h - \Psi_h.
        \tag{$0 \in \Psi_{h+1} - \Psi_{h+1}$}
        \\
        \Longrightarrow \Fcal_h - \Tcal_h \Fcal_{h+1} \subseteq &~ \Gcal_h - \Pcal_h \Gcal_{h+1}.
    \end{align*}

\subsection{Proofs}
\label{sec:wf_proofs}

We first present the following form of Freedman's inequality for martingales~\citep[e.g.,][]{agarwal2014taming}.
\begin{lemma}[Freedman's Inequality]
\label{lem:freedman}
Let $\{X^\iter{1}, X^\iter{2}, \dotsc, X^\iter{T}\}$ be a real-valued martingale difference sequence adapted to a filtration $\{\filtr^\iter{1}, \filtr^\iter{2}, \dotsc, \filtr^\iter{T}\}$ (i.e., $\E[ X^\iter{t} \mid \filtr^\iter{t-1} ] = 0$, $\forall t \in [T]$). If $|X^\iter{t}| \leq R$ almost surely for all $t \in [T]$, then for any $\eta \in (0, \nicefrac{1}{R})$, with probability at least $1 - \delta$,
\begin{align*}
    \sum_{t = 1}^{T} X^\iter{t} \leq \eta \sum_{t = 1}^{T} \E\left[ (X^\iter{t})^2 \midmid \filtr^\iter{t-1} \right] + \frac{\log(\nicefrac{1}{\delta})}{\eta}.
\end{align*}
\end{lemma}

We now provide proofs from \cref{sec:rf_algo}.

\begin{proof}[\cpfname{thm:rf_golf}]
Over this section, the test function class is selected as
\begin{align*}
\delta_{h,\rf}^\iter{t}(x_h,a_h) \coloneqq g_h^\iter{t}(x_h,a_h) - (\Pcal_{h} g_{h+1}^\iter{t})(x_h,a_h), ~ \forall (h,t) \in [H] \times [T].
\end{align*}
By \cref{thm:regret_online_c} (setting reward to be zero and replacing everything regarding $\Fcal$ to $\Gcal$), we have
\begin{align}
\label{eq:rf_regret}
\sum_{t=1}^{T}\sum_{h=1}^{H}\E_{d_h^\iter{t}}\left[\delta_{h,\rf}^\iter{t}\right] \leq H \sqrt{T \cdimrlrf(\Gcal,\Pi_\Gcal,T) \beta_\rf}.
\end{align}
For any $(h,t) \in [H] \times [T]$, we have,
\begin{align}
\nonumber
\E_{d_h^\iter{t}}\left[ (\Pcal_{h} g_{h+1}^\iter{t})(x_{h},a_{h}) \right] = &~ \E_{d_h^\iter{t}}\left[ \sum_{x'} \PP_h(x_{h+1}|x_h,a_h) \max_{a_{h+1} \in \Acal} g_{h+1}^\iter{t}(x_{h+1},a_{h+1}) \right]
\\
\nonumber
= &~ \E_{d_h^\iter{t}}\left[ \sum_{x'} \PP_h(x_{h+1}|x_h,a_h) g_{h+1}^\iter{t} (x_{h+1},\pi_h^\iter{t}) \right]
\tag{$\pi^\iter{t}$ is the greedy policy of $g^\iter{t}$}
\\
\label{eq:dhpgh}
= &~ \E_{d_{h+1}^\iter{t}}\left[ g_{h+1}^\iter{t} (x_{h+1},a_{h+1}) \right].
\end{align}
Therefore, we know
\begin{align}
\nonumber
\sum_{h=1}^{H}\E_{d_h^\iter{t}}\left[\delta_{h,\rf}^\iter{t}\right] = &~ \sum_{h=1}^{H}\E_{d_h^\iter{t}}\left[g_h^\iter{t} - \Pcal_{h} g_h^\iter{t}\right]
\\
\nonumber
= &~ \sum_{h=1}^{H} \left( \E_{d_h^\iter{t}}\left[g_h^\iter{t}\right] - \E_{d_{h+1}^\iter{t}}\left[g_{h+1}^\iter{t}\right] \right)
\tag{by \cref{eq:dhpgh}}
\\
\nonumber
= &~ \E_{d_1^\iter{t}}\left[g_1^\iter{t}\right]
\\
\label{eq:rf_be}
= &~ g_1^\iter{t}(x_1,\pi_1^\iter{t}).
\end{align}

Now, since
\begin{align*}
t_\star \coloneqq \argmin_{t \in [T]} g_1^\iter{t}(x_1,\pi_1^\iter{t}),
\end{align*}
then,
\begin{align}
\nonumber
g_1^\iter{t_\star}(x_1,\pi^\iter{t_\star}) = &~ \frac{1}{T} \sum_{t = 1}^{T} g_1^\iter{t}(x_1,\pi_1^\iter{t})
\\
\nonumber
= &~ \frac{1}{T} \sum_{t=1}^{T}\sum_{h=1}^{H}\E_{d_h^\iter{t}}\left[\delta_{h,\rf}^\iter{t}\right]
\tag{by \cref{eq:rf_be}}
\\
\label{eq:g1_tstar_bound}
\leq &~ H \sqrt{\frac{\cdimrlrf(\Gcal,\Pi_\Gcal,T) \beta_\rf}{T}},
\end{align}
where the last inequality follows from \cref{eq:rf_regret}.

By \pref{lem:rf_offline_vspace}, we know there exists a $\ghat \in \Gcal^\iter{t_\star - 1}$, such that $\fhat_h - \Tcal_h \fhat_{h+1} = \ghat_h - \Pcal_h \widehat g_{h+1}$, $\forall h \in [H]$. In addition, we can obtain
\begin{align*}
J(\pi^\star) - J(\pi_{\fhat}) \leq &~ \fhat(x_1,\pi_{\fhat,1}) - J(\pi_{\fhat})
\tag{by \cref{lem:golf_concentration}}
\\
\leq &~ \ghat(x_1,\pi_{\ghat,1}).
\tag{by \cref{lem:rf_offline_keylem}}
\end{align*}
Therefore, we have
\begin{align*}
J(\pi^\star) - J(\pi_{\fhat}) \leq &~ \ghat(x_1,\pi_{\ghat,1})
\\
\leq &~ g_1^\iter{t_\star}(x_1,\pi_1^\iter{t_\star})
\\
\leq &~ H \sqrt{\frac{\cdimrlrf(\Gcal,\Pi_\Gcal,T) \beta_\rf}{T}}.
\tag{by \cref{eq:g1_tstar_bound}}
\end{align*}
Plugging back the selection of $\beta_\rf$ completes the proof.
\end{proof}

\begin{proof}[\cpfname{lem:rf_offline_keylem}]
We establish the proof by induction. For $h = H$, the the inductive hypothesis holds because $g_H = f_H - R_H = f_H - Q_H^{\pi_f}$.

Suppose the inductive hypothesis holds at $h + 1$, we have for any $x \in \Xcal$,
\begin{align}
\nonumber
g_{h+1}(x,a) \geq &~ f_{h+1}(x,a) - Q_{h+1}^{\pi_f}(x,a), ~ \forall a \in \Acal.
\\
\nonumber
\Longrightarrow g_{h+1}(x,\pi_{f,h+1}) \geq &~ f_{h+1}(x,\pi_{f,h+1}) - Q_{h+1}^{\pi_f}(x,\pi_{f,h+1}).
\\
\nonumber
\Longrightarrow \max_{a \in \Acal} g_{h+1}(x,a) \geq &~ f_{h+1}(x,\pi_{f,h+1}) - Q_{h+1}^{\pi_f}(x,\pi_{f,h+1}).
\\
\label{eq:ind_hp1}
\Longrightarrow g_{h+1}(x,\pi_{g,h+1}) \geq &~ f_{h+1}(x,\pi_{f,h+1}) - V_{h+1}^{\pi_f}(x).
\end{align}
Then, as $g_{h} = f_{h} - \Tcal_h f_{h+1} + \Pcal_h g_{h+1}$, we have for any $(x,a) \in \Xcal \times \Acal$,
\begin{align*}
g_h(x,a) = &~ f_{h}(x,a) - R_h(x,a) - \E_{x'|x,a}\left[ \max_{a' \in \Acal} f_{h+1}(x',a') \right] + \E_{x'|x,a}\left[ \max_{a' \in \Acal} g_{h+1}(x',a') \right]
\\
= &~ f_{h}(x,a) - R_h(x,a) + \E_{x'|x,a}\left[ \max_{a' \in \Acal} g_{h+1}(x',a') - \max_{a' \in \Acal} f_{h+1}(x',a') \right]
\\
= &~ f_{h}(x,a) - R_h(x,a) + \E_{x'|x,a}\left[ g_{h+1}(x',\pi_{g,h+1}) - f_{h+1}(x',\pi_{f,h+1}) \right]
\\
\geq &~ f_{h}(x,a) - R_h(x,a) + \E_{x'|x,a}\left[ - V_{h+1}^{\pi_f}(x') \right]
\tag{by \cref{eq:ind_hp1}}
\\
= &~ f_{h}(x,a) - \left( R_h(x,a) + \E_{x'|x,a}\left[ V_{h+1}^{\pi_f}(x') \right] \right)
\\
= &~ f_{h}(x,a) -  Q_{h}^{\pi_f}(x,a).
\end{align*}
Therefore, we prove that the inductive hypothesis also holds at $h$ using the inductive hypothesis at $h+1$. This completes the proof.
\end{proof}

\begin{proof}[\cpfname{lem:rf_offline_vspace}]
Over this proof, we use $d_h^\iter{t}$ as the shorthand of $d_h^{\pi^\iter{t}}$.
The proof of this lemma consists of two parts. 
\begin{enumerate}[(i)]
\item There exists a radius $\beta_1$, such that for any $g \in \Gcal$, if such $g$ satisfies
\begin{align*}
\sum_{t=1}^{t_\star - 1}\E_{d_h^\iter{t}}\left[ \left( g_h - \Pcal_{h} g_{h+1} \right)^2 \right] \leq \beta_1, ~~ \forall h \in [H]
\end{align*}
then $g \in \Gcal^\iter{t_\star - 1}$.
\item There exists another radius $\beta_2$, where $\beta_2 \leq \beta_1$. For any $f \in \Fcal^\off$, we have
\begin{align*}
\sum_{t=1}^{t_\star - 1}\E_{d_h^\iter{t}}\left[ \left( f_h - \Tcal_{h} f_{h+1} \right)^2 \right] \leq \beta_2, ~~ \forall h \in [H].
\end{align*}
\end{enumerate}

\paragraph{Proof of part (i)}
For any $(t,h,g) \in [T] \times [H] \times \Gcal$, let $Y_h^\iter{t}(g)$ be defined as
\begin{align*}
Y_h^\iter{t}(g) \coloneqq \left( g_h(x_h^\iter{t},a_h^\iter{t}) - g_{h+1}(x_{h+1}^\iter{t},\pi_{g,1}) \right)^2 - \left( (\Pcal_h g_{h+1})(x_h^\iter{t},a_h^\iter{t}) - g_{h+1}(x_{h+1}^\iter{t},\pi_{g,1}) \right)^2.
\end{align*}
Also, let $\filtr_h^\iter{t}$ be the filtration induced by $\{x_1^\iter{i}, a_1^\iter{i}, x_2^\iter{i}, a_2^\iter{i}, \dotsc, x_H^\iter{i} \}_{i=1}^{t}$, and we then have
\begin{align}
\label{eq:Yht_exp}
\E\left[ Y_h^\iter{t}(g) \midmid \filtr_h^\iter{t-1} \right] = \E_{d_h^\iter{t}}\left[ \left( g_h - \Pcal_{h} g_{h+1} \right)^2 \right]
\end{align}
and
\begin{align*}
\V\left[ Y_h^\iter{t}(g) \midmid \filtr_h^\iter{t-1} \right] \leq \E\left[ \left( Y_h^\iter{t}(g) \right)^2 \midmid \filtr_h^\iter{t-1} \right] \leq 2 \E\left[ Y_h^\iter{t}(g) \midmid \filtr_h^\iter{t-1} \right] = 2 \E_{d_h^\iter{t}}\left[ \left( g_h - \Pcal_{h} g_{h+1} \right)^2 \right].
\end{align*}

Now, let $\Ybar_h^\iter{t}(g) \coloneqq Y_h^\iter{t}(g) - \E\left[ Y_h^\iter{t}(g) \midmid \filtr_h^\iter{t-1} \right]$, so that $\left \{\Ybar_h^\iter{t}(g) \right\}_{t = 1}^{T}$ is a martingale difference sequence adapts to the filtration $\left\{ \filtr_h^\iter{t} \right\}_{t = 1}^{T}$, and $|\Ybar_h^\iter{t}(g)| \leq 2$ almost surely.
Then, by applying \cref{lem:freedman} with a union bound, we have for any $(h,g) \in [H] \times \Gcal$ and any $\eta \in (0,\nicefrac{1}{2})$, with probability at least $1 - \delta$,
\begin{align}
\nonumber
\sum_{t=1}^{t_\star - 1} \Ybar_h^\iter{t}(g) \leq &~ \eta \sum_{t = 1}^{t_\star - 1} \E\left[ \left( \Ybar_h^\iter{t}(g) \right)^2 \midmid \filtr_h^\iter{t-1} \right] + \frac{\log(\nicefrac{H |\Gcal|}{\delta})}{\eta}
\\
\leq &~ \eta \sum_{t = 1}^{t_\star - 1} \E\left[ \left( Y_h^\iter{t}(g) \right)^2 \midmid \filtr_h^\iter{t-1} \right] + \frac{\log(\nicefrac{H |\Gcal|}{\delta})}{\eta}
\tag{variance is bounded by the second moment}
\\
\leq &~ \eta \sum_{t = 1}^{t_\star - 1} \E\left[ Y_h^\iter{t}(g) \midmid \filtr_h^\iter{t-1} \right] + \frac{\log(\nicefrac{H |\Gcal|}{\delta})}{\eta}.
\tag{$|Y_h^\iter{t}(g)| \leq 1$ by its definition}
\\
\Longrightarrow
\sum_{t=1}^{t_\star - 1} Y_h^\iter{t}(g) \leq &~ \eta \sum_{t = 1}^{t_\star - 1} \E\left[ Y_h^\iter{t}(g) \midmid \filtr_h^\iter{t-1} \right] + \frac{\log(\nicefrac{H |\Gcal|}{\delta})}{\eta} + \sum_{t=1}^{t_\star - 1} \E\left[ Y_h^\iter{t}(g) \midmid \filtr_h^\iter{t-1} \right]
\tag{by the definition of $\Ybar_h^\iter{t}(g)$}
\\
= &~ (1 + \eta) \sum_{t = 1}^{t_\star - 1} \E\left[ Y_h^\iter{t}(g) \midmid \filtr_h^\iter{t-1} \right] + \frac{\log(\nicefrac{H |\Gcal|}{\delta})}{\eta}.
\label{eq:yhtg_freedman}
\end{align}

If some $g \in \Gcal$ satisfies 
\begin{align*}
\underbrace{\sum_{t=1}^{t_\star - 1}\E_{d_h^\iter{t}}\left[ \left( g_h - \Pcal_{h} g_{h+1} \right)^2 \right]}_{= \sum_{t=1}^{t_\star - 1} \E\left[ Y_h^\iter{t}(g) \midmid \filtr_h^\iter{t-1} \right] \text{~by~\cref{eq:Yht_exp}}} \leq \beta_1, ~~ \forall h \in [H],
\end{align*}
then by \cref{eq:yhtg_freedman}, we have for any $h \in [H]$
\begin{align*}
\sum_{t=1}^{t_\star - 1} Y_h^\iter{t}(g) \leq &~ (1 + \eta) \beta_1 + \frac{\log(\nicefrac{H |\Gcal|}{\delta})}{\eta}
\\
\leq &~ 3 (\beta_1 + \log(\nicefrac{H |\Gcal|}{\delta})).
\tag{e.g., by picking $\eta = \nicefrac{1}{3}$}
\end{align*}

So we only need to guarantee
\begin{align}
\nonumber
3 \cdot\left( \beta_1 + \log(\nicefrac{H|\Gcal|}{\delta})\right) \leq &~ \beta_\rf
\\
\label{eq:beta_1_condition}
\Longrightarrow \beta_1 \leq &~ \frac{\beta_\rf}{3} - \log(\nicefrac{H|\Gcal|}{\delta}).
\end{align}

\paragraph{Proof of part (ii)}
Similar to (i), for any $(t,h,f) \in [T] \times [H] \times \Fcal$, let $X_h^\iter{t}(f)$ be defined as
\begin{align*}
X_h^\iter{t}(f) \coloneqq &~ \left( f_h(x_h^\iter{t},a_h^\iter{t}) - R(x_h^\iter{t},a_h^\iter{t}) - f_{h+1}(x_{h+1}^\iter{t},\Pi_{\Gcal}) \right)^2
\\
&~ - \left( (\Tcal_h f_{h+1})(x_h^\iter{t},a_h^\iter{t}) - R(x_h^\iter{t},a_h^\iter{t}) - f_{h+1}(x_{h+1}^\iter{t},\pi_{f,h+1}) \right)^2.
\end{align*}
Also let $\Xbar_h^\iter{t}(f) \coloneqq \E\left[ X_h^\iter{t}(f) \midmid \filtr_h^\iter{t-1} \right] - X_h^\iter{t}(f)$, so that $\left \{\Xbar_h^\iter{t}(f) \right\}_{t = 1}^{T}$ is a martingale difference sequence adapts to the filtration $\left\{ \filtr_h^\iter{t} \right\}_{t = 1}^{T}$, and $|\Xbar_h^\iter{t}(f)| \leq 2$ almost surely.

Thus, by same arguments as \cref{eq:Yht_exp,eq:yhtg_freedman} (as well as applying \cref{lem:freedman}), we have
\begin{align}
\label{eq:Xht_exp}
\E\left[ X_h^\iter{t}(f) \midmid \filtr_h^\iter{t} \right] = \E_{d_h^\iter{t}}\left[ \left( f_h - \Tcal_{h} f_{h+1} \right)^2 \right]
\end{align}
and for any $(h,f) \in [H] \times \Fcal$ and any $\eta \in (0,\nicefrac{1}{2})$, with probability at least $1 - \delta$, 
\begin{gather}
\nonumber
\sum_{t=1}^{t_\star - 1} \E\left[ X_h^\iter{t}(f) \midmid \filtr_h^\iter{t-1} \right] \leq \eta \sum_{t = 1}^{t_\star - 1} \E\left[ X_h^\iter{t}(f) \midmid \filtr_h^\iter{t-1} \right] + \frac{\log(\nicefrac{H |\Fcal|}{\delta})}{\eta} + \sum_{t=1}^{t_\star - 1} X_h^\iter{t}(f)
\\
\Longrightarrow (1 - \eta)\sum_{t=1}^{t_\star - 1} \E\left[ X_h^\iter{t}(f) \midmid \filtr_h^\iter{t-1} \right] \leq \frac{\log(\nicefrac{H |\Fcal|}{\delta})}{\eta} + \sum_{t=1}^{t_\star - 1} X_h^\iter{t}(f).
\label{eq:xht_freedman}
\end{gather}

Therefore, if $f \in \Fcal^\iter{\off}$, we have
\begin{align}
\nonumber
\sum_{t=1}^{t_\star - 1} X_h^\iter{t}(f)
= &~ \sum_{t=1}^{t_\star - 1} \left( f_h(x_h^\iter{t},a_h^\iter{t}) - f_{h+1}(x_{h+1}^\iter{t},\pi_{f,h+1}) \right)^2 - \sum_{t=1}^{t_\star - 1} \left( (\Tcal_h f_{h+1})(x_h^\iter{t},a_h^\iter{t}) - f_{h+1}(x_{h+1}^\iter{t},\pi_{f,h+1}) \right)^2
\\
\nonumber
\leq &~ \sum_{t=1}^{t_\star - 1} \left( f_h(x_h^\iter{t},a_h^\iter{t}) - f_{h+1}(x_{h+1}^\iter{t},\pi_{f,h+1}) \right)^2 - \min_{f_h' \in \Fcal_h}\sum_{t=1}^{t_\star - 1} \left( f_h'(x_h^\iter{t},a_h^\iter{t}) - f_{h+1}(x_{h+1}^\iter{t},\pi_{f,h+1}) \right)^2
\\
\nonumber
\leq &~ \Lcal_{h}^\iter{\off}(f_h,f_{h+1}^\iter{t_\star}) - \min_{f'_h \in \Fcal_h} \Lcal_{h}^\iter{\off}(f'_h,f_{h+1}^\iter{t_\star})
\\
\label{eq:beta_off_bound}
\leq &~ \beta_\off.
\end{align}

We then combine \cref{eq:Xht_exp,eq:xht_freedman,eq:beta_off_bound} and obtain
\begin{align*}
\sum_{t=1}^{t_\star - 1}\E_{d_h^\iter{t}}\left[ \left( f_h - \Tcal_{h} f_{h+1} \right)^2 \right] = &~ \sum_{t=1}^{t_\star - 1} \E\left[ X_h^\iter{t}(f) \midmid \filtr_h^\iter{t-1} \right]
\tag{by \cref{eq:Xht_exp}}
\\
\leq &~ \frac{\log(\nicefrac{H |\Fcal|}{\delta})}{(1 - \eta)\eta} + \frac{1}{1 - \eta}\sum_{t=1}^{t_\star - 1} X_h^\iter{t}(f)
\tag{by \cref{eq:xht_freedman}}
\\
\leq &~ \frac{\log(\nicefrac{H |\Fcal|}{\delta})}{(1 - \eta)\eta} + \frac{1}{1 - \eta} \beta_\off
\tag{by \cref{eq:beta_off_bound}}
\\
\leq &~ \underbrace{5 \log(\nicefrac{H |\Fcal|}{\delta}) + 2 \beta_\off}_{\eqqcolon \beta_2}.
\tag{by e.g., setting $\eta = \nicefrac{1}{3}$}
\end{align*}

So we only need to guarantee
\begin{align}
\label{eq:beta_2_condition}
5 \log(\nicefrac{H |\Fcal|}{\delta}) + 2 \beta_\off = \beta_2 \leq \beta_1.
\end{align}

\paragraph{Putting everything together}

By \cref{eq:beta_1_condition,eq:beta_2_condition}, we know we only
need the following inequality to hold:
\begin{align*}
5 \log(\nicefrac{H |\Fcal|}{\delta}) + 2 \beta_\off \leq &~ \frac{\beta_\rf}{3} - \log(\nicefrac{H|\Gcal|}{\delta}).
\\
\Longrightarrow \beta_\rf \geq &~ 6 \beta_\off + 18 \log(\nicefrac{H|\Gcal|}{\delta}).
\end{align*}
This is satisfied via the condition of \cref{thm:rf_golf}.

Combining (i) and (ii), we can simply obtain for any $h \in [H]$,
\begin{align}
\left\{ f_h - \Tcal_h f_{h+1}: f \in \Fcal^\off \right\} \subseteq &~ \left\{ f_h - \Tcal_h f_{h+1}: \sum_{t=1}^{t_\star - 1}\E_{d_h^\iter{t}}\left[ \left( f_h - \Tcal_{h} f_{h+1} \right)^2 \right] \leq \beta_2, \forall h \in [H], f \in \Fcal \right\}
\tag{by (ii)}
\\
\subseteq &~ \left\{ g_h - \Pcal_h f_{h+1}: \sum_{t=1}^{t_\star - 1}\E_{d_h^\iter{t}}\left[ \left( g_h - \Pcal_{h} f_{h+1} \right)^2 \right] \leq \beta_1, \forall h \in [H], g \in \Gcal \right\}
\tag{by Assumption~\ref{asm:rf_completeness} and $\beta_2 \leq \beta_1$}
\\
\subseteq &~ \left\{ g_h - \Pcal_h g_{h+1}: g \in \Gcal^\iter{t_\star - 1} \right\}.
\tag{by (i)}
\end{align}
This completes the proof.
\end{proof}

\end{document}